\documentclass[12pt,hyperref]{UOthesis}
\pdfoutput=1

\usepackage{etex}
\usepackage[usenames,dvipsnames,svgnames,table]{xcolor}
\usepackage{amsthm,amssymb,mathtools,graphicx,amsfonts,mathrsfs}
\usepackage{float,enumerate,array,bbm,dsfont,wrapfig,booktabs,bm}
\usepackage{algorithm,algorithmic}
\usepackage{todonotes}
\usepackage{chngcntr}

\allowdisplaybreaks

\NoLogo
\NoResume
\NoListofsymbols
\NoIntro

\setUOname{Stan Hatko}         
\setUOcpryear{2015}               
\setUOtitle{$k$-Nearest Neighbour Classification of Datasets with a Family of Distances}  

\msc               

\setUOabstract{
  The $k$-nearest neighbour ($k$-NN) classifier is one of the oldest and most important supervised learning algorithms for classifying datasets. Traditionally the Euclidean norm is used as the distance for the $k$-NN classifier. In this thesis we investigate the use of alternative distances for the $k$-NN classifier.
  
  We start by introducing some background notions in statistical machine learning. We define the $k$-NN classifier and discuss Stone's theorem and the proof that $k$-NN is universally consistent on the normed space $(\R^d, \norm{\cdot})$. We then prove that $k$-NN is universally consistent if we take a sequence of random norms (that are independent of the sample and the query) from a family of norms that satisfies a particular boundedness condition. We extend this result by replacing norms with distances based on uniformly locally Lipschitz functions that satisfy certain conditions. We discuss the limitations of Stone's lemma and Stone's theorem, particularly with respect to quasinorms and adaptively choosing a distance for $k$-NN based on the labelled sample. We show the universal consistency of a two stage $k$-NN type classifier where we select the distance adaptively based on a split labelled sample and the query. We conclude by giving some examples of improvements of the accuracy of classifying various datasets using the above techniques.
}

\setUOresume{
  Vous pouvez introduire le r\'esum\'e de votre th\`ese ici.
}

\setUOdedicationsText{To my parents.}

\setUOthanks{
  I would like to thank my supervisor Dr. Vladimir Pestov for his guidance and support during my studies and for introducing me to the field of machine learning.
  
  I would also like to thank the members of the University of Ottawa Data Science Group and people who worked with the group, including Ga\"{e}l Giordano, Hubert Duan, Yue Dong, Samuel Buteau, Julien Roy, Sabrina Sixta, \'{E}milie Idene, Luiz Gustavo Cordeiro, and Christian Despres.
  
  Finally, I would like to thank my family for their support.
}

\newcommand{\R}{\mathbb{R}} 
\newcommand{\Q}{\mathbb{Q}} 
\newcommand{\N}{\mathbb{N}} 

\newcommand{\E}{\mathbb{E}} 
\newcommand{\Prob}{\mathbb{P}} 
\newcommand{\Indicator}{\mathds{1}} 
\newcommand{\Cov}{\operatorname{Cov}} 
\newcommand{\Corr}{\operatorname{Corr}} 

\newcommand{\A}{\mathcal{A}} 
\newcommand{\B}{\mathcal{B}} 
\newcommand{\F}{\mathcal{F}} 

\newcommand{\vecv}{{\bm{v}}} 
\newcommand{\vecu}{{\bm{u}}} 
\newcommand{\vecw}{{\bm{w}}} 
\newcommand{\vecx}{{\bm{x}}} 
\newcommand{\vecy}{{\bm{y}}} 
\newcommand{\veczero}{{\bm{0}}} 

\newcommand{\norm}[1]{{\left\lVert #1 \right\lVert}} 

\DeclareMathOperator{\err}{err} 
\newcommand{\ELL}{\mathcal{L}} 
\newcommand{\NormsFamily}{\mathcal{N}} 
\newcommand{\MatrixFamily}{\mathcal{M}} 

\theoremstyle{plain}

\newtheorem{theorem}{Theorem}[section]
\newtheorem{corollary}[theorem]{Corollary}
\newtheorem{lemma}[theorem]{Lemma}

\newtheorem{remark}[theorem]{Remark}

\theoremstyle{definition}
\newtheorem{definition}[theorem]{Definition}

\theoremstyle{remark}
\newtheorem{example}[theorem]{Example}

\counterwithout{equation}{section}
\counterwithin{equation}{chapter}  

\numberwithin{figure}{chapter}
\numberwithin{table}{chapter} 

\begin{document}

\chapter{An Introduction to Statistical Machine Learning}

In this chapter we introduce the fundamental notions of statistical machine learning. No prior knowledge of statistical machine learning is assumed in this section. We start by giving an informal discussion with some examples and then discussing the theory on a more formal level.

\section{An Informal Introduction}

In the classification problem of statistical machine learning, we start with a dataset (where the points come from some sample space), together with a label (or class) for each point (where there are a finite number of possible labels). We suppose that the points in the dataset are independently and identically distributed. We have a new data point, called the query, from the same distribution as the data set, and which is also assumed to be independent of the points in the data set. However, we do not have the label for the query. We would like to predict the label for the query based on the dataset.

For instance, suppose we would like to predict if a person has a predisposition for heart disease based on their genome. We have a dataset of the genome of people with their genomic sequence and whether or not they have heart disease. We now have a new patient, for which we have the genome but do not know if they have heart disease. We would like to predict, based on their genomic sequence, if they have heart disease, with the only information available to us being the dataset and the person's genomic sequence.

Let $X$ be the dataset and $Y$ be a set of classes. A \emph{classifier} $f: X \to Y$ is a function that attempts to predict a class $y$ for a data point $x$. The \emph{accuracy} of the classifier $f$ is the probability that we will predict the correct label for the query, and the \emph{misclassification error} of $f$ is the probability that we will predict a wrong label. Given the query point, we would like to predict its label. We would like to find a classifier $f$ whose accuracy is as high as possible (or equivalently, whose error is as small as possible). The \emph{Bayes error} is the infimum of the errors of all possible classifiers for a distribution $\mu$ on $X \times Y$. We can show that the Bayes error is attained by the \emph{Bayes classifier}, however constructing the Bayes classifier requires knowledge of the underlying distribution $\mu$, which we normally do not have, we only have a set of labelled data points.

The process of constructing a classifier $f$ is called \emph{learning}. A \emph{learning rule} is a family of functions that takes a set of labelled data points and outputs a classifier, which we can then use to classify query points. A learning rule is said to be \emph{consistent} for a distribution $\mu$ if the expected value of the error converges to the Bayes error in probability for $\mu$ as the number of labelled data points goes to infinity. A learning rule is \emph{universally consistent} if it is consistent for every distribution $\mu$ on $X \times Y$. Common learning rules include those based on $k$-nearest neighbour, support vector machine (SVM), and random forest. When applying a learning rule and then using it to classify points, we often refer to the combination of the learning rule and classifier together as simply a classifier.

For any classifier, to test its accuracy we take the dataset and split it into two disjoint subsets, the \emph{training set} and the \emph{testing set}. The training set is used in constructing $f$, and from this we predict the labels for points in the testing set. We then compare the predicted labels to the correct labels in the testing set and compute the accuracy of our prediction.

\section{Theory of Statistical Machine Learning}

We now introduce the fundamental notions of the theory of statistical machine learning. Let $\Omega$ be a nonempty set called the \emph{domain}, $\{1, 2, \dots, q\}$ (with $q \geq 2$) be a finite set of \emph{labels} (or \emph{classes}), and $\mu$ be a probability measure on $\Omega \times \{1, 2, \dots, q\}$. We often assume without loss of generality that there are only two classes (the \emph{binary classification problem}), for this section we will consider the case of $q \geq 2$ classes, but afterwards we will focus on the $q = 2$ case.

A \emph{classifier} is a Borel measurable function $f: \Omega \to \{1, 2, \dots, q\}$, that maps points in the domain $\Omega$ to classes in $\{1, 2, \dots, q\}$. We define the \emph{misclassification error} of a classifier as the probability that the label predicted by our classifier is different than the true label,
\begin{equation}
\err_\mu(f) = \mu\left( \left\{ (x, y) \in \Omega \times \{1, 2, \dots, q\} \;:\; f(x) \neq y \right\} \right) .
\end{equation}

The \emph{Bayes error} is the infimum of the misclassification error over all possible classifiers for the probability measure $\mu$ on $\Omega \times \{1, 2, \dots, q\}$,
\begin{equation}
\ell^*(\mu) = \inf_f \err_\mu(f) .
\end{equation}

We see that since the misclassification error must be in $[0, 1]$ (since any probability must be in $[0, 1]$), and the set of classifiers is nonempty (since we can simply take the classifier that maps every point in $\Omega$ to zero), it follows that the infimum exists in $[0, 1]$ and so the Bayes error is always well defined.

Suppose we have a set of $n$ independent and identically distributed random ordered pairs $D_n = (X_1, Y_1), (X_2, Y_2), \dots, (X_n, Y_n)$, modelling the data. A \emph{learning rule} $\ELL = {\left(\ELL_n\right)}_{n=1}^\infty$ is a family of functions that maps each possible labelled sample to a classifier,
\begin{equation}
\label{eq:LearningRule}
\ELL_n: {\left( \Omega \times \{1, 2, \dots, q\} \right)}^n \to \{ f: \Omega \to \{1, 2, \dots, q\} \;|\; f \text{ is Borel} \} .
\end{equation}

Common learning rules include those based on $k$-nearest neighbour ($k$-NN), Support Vector Machine (SVM), and Random Forest. In applications when classifying datasets (so we classify points immediately when learning) it is common to simply refer to these as ``classifiers", for now we will continue to make the distinction between learning rules and classifiers. A learning rule can also be thought of as a sequence of classifiers constructed based on the labelled sample of points.

Let $D_n = (X_1, Y_1), (X_2, Y_2), \dots, (X_n, Y_n)$ be an iid labelled sample with distribution $\mu$, and $(X, Y)$ be the query and the label, which is independent from the sample and also has the same distribution $\mu$. We let $\ell^*(\mu)$ be the Bayes error for the distribution $\mu$. Given the labelled sample, the error probability is the conditional probability
\begin{equation}
\label{eq:ErrorProbability}
L_n = \Prob(\ELL_n(X, D_n) \neq Y | D_n) .
\end{equation}

A learning rule $\ELL$ is said to be \emph{consistent} (or \emph{weakly consistent}) for the distribution $\mu$ if the misclassification error of the learning rule ${(\ELL_n)}_{n=1}^\infty$ converges in probability to the Bayes error, that is, as $n \to \infty$,
\begin{equation}
\label{eq:ConsistentLearningRuleProbConvergence}
\lim_{n \to \infty}\Prob(\left|L_n - \ell^*(\mu)\right| > \epsilon) = 0
\end{equation}
or equivalently, that
\begin{equation}
\label{eq:ConsistentLearningRule}
\Prob\left(\ELL_n(X, D_n) \neq Y\right) \to \ell^*(\mu) .
\end{equation}

We say that $\ELL$ is \emph{strongly consistent} if with probability one we have a sequence of labelled samples $D_1, D_2, \dots$ such that the misclassification error approaches the Bayes error as the sample size $n$ approaches infinity (so the above convergence in probability is replaced by almost sure convergence). That is, we have
\begin{equation}
\label{eq:StrongConsistency}
\Prob\left(\lim_{n \to \infty} \Prob(\ELL_n(X, D_n) \neq Y | D_n) = \ell^*(\mu)\right) = \Prob\left(\lim_{n \to \infty} L_n = \ell^*(\mu)\right) = 1 .
\end{equation}

\begin{figure}
\centering
\includegraphics[scale=1]{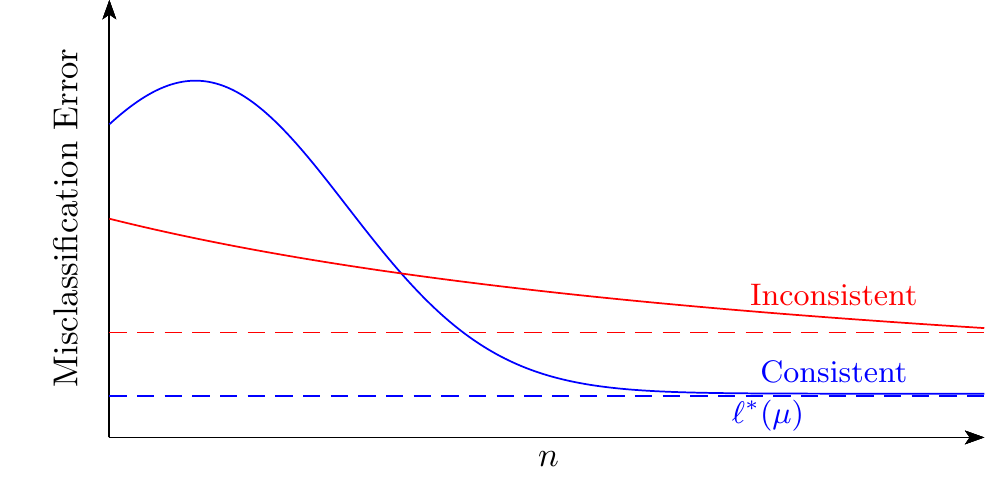}
\caption{Illustration of the misclassification error of a consistent and an inconsistent learning rule. We see that the misclassification error of the consistent learning rule approaches the Bayes error $\ell^*(\mu)$ as the samples size $n$ approaches infinity, while the inconsistent rule (which in this case starts off better for small $n$) performs more poorly as $n$ increases and does not converge to the Bayes error.}
\label{fig:ConsistencyIllustration}
\end{figure}

If the learning rule is consistent for every distribution on $\Omega \times \{1, 2, \dots, q\}$, we say that it is \emph{universally consistent}. We define \emph{strong universal consistency} in the same way, that a learning rule is strongly consistent for every distribution on $\Omega \times \{1, 2, \dots, q\}$.\footnote{Not every learning rule used in applications is universally consistent, for instance, random forests are not universally consistent (\cite{RandomForestConsistency}, Proposition 8) but have a very good classification accuracy on many datasets and are commonly used in applications.}

A learning rule whose misclassification error is monotone decreasing at each step $n$ is called a \emph{smart} learning rule. A simple example of a learning rule that is ``smart" by this definition is to select the classifier that selects a particular fixed label always (completely ignoring the labelled sample), then the misclassification error is constant regardless of the sample size (this is not a learning rule we would call ``smart" in the usual sense of the word). Such a learning rule is obviously not universally consistent. A universally consistent learning rule is not necessarily smart, the misclassification error can temporarily increase for some $n$ before decreasing again towards the Bayes error. There are no known examples of a universally consistent smart learning rule, it has been conjectured that no such learning rules exist.

\subsection{The Regression Function and the Bayes Classifier}

In this section, we assume we have a binary classification problem, that is, the set of classes is $\{0, 1\}$. We let $\mu$ be a probability measure on $\Omega \times \{0, 1\}$. We then define two new measures $\nu, \lambda$ on $\Omega$ by (for any Borel set $A \subseteq \Omega$):
\begin{align}
\nu(A) &= \mu(A \times \{1\}) \\
\lambda(A) &= \mu(A \times \{0, 1\}) = \mu(A \times \{0\}) + \mu(A \times \{1\})
\end{align}

We observe that for any Borel set $A \subseteq \Omega$, $\nu(A) \leq \lambda(A)$, and so $\nu$ is absolutely continuous with respect to $\lambda$. Hence by the Radon-Nikodym derivative theorem, the Radon-Nikodym derivative of $\nu$ with respect to $\lambda$ exists, which we call the regression function $\eta$ (that is, for any Borel set $A$, $\int_A \eta \text{d} \lambda = \nu(A)$).\cite{Billingsley,pbook} By the Radon-Nikodym derivative theorem, $\eta$ is integrable with respect to $\lambda$ and is Borel measurable. Equivalently, we can also write $\eta$ as the conditional probability $\eta(x) = \Prob(Y = 1 | X = x)$.

With the regression function, we are now able to define a classifier called the \emph{Bayes classifier}. The Bayes classifier $g^*$ is defined as:
\begin{equation}
g^*(x) = \begin{dcases}
1 & \text{if } \eta(x) \geq \frac12 \\
0 & \text{otherwise}
\end{dcases}
\end{equation}

We now show that the Bayes classifier is optimal, that is, it has the highest accuracy of any classifier on our dataset. This is a standard result, the proof below is based on the one found in \cite{pbook} (Theorem 2.1) and \cite{Duan} (Theorem 1.1.2).
\begin{theorem}[Bayes Optimality Theorem]
\label{theorem:BayesClassifierIsOptimal}
For any classifier $g: \Omega \to \{0, 1\}$ and any probability distribution $\mu$ on $\Omega \times \{0, 1\}$, we have the inequality
\begin{equation}
\mu(\{(x, y) \;:\; g^*(x) \neq y\}) \leq \mu(\{(x, y) \;:\; g(x) \neq y\}) .
\end{equation}

Equivalently, we can write this in terms of random variables,
\begin{equation}
\mu(g^*(X) \neq Y) \leq \mu(g(X) \neq Y) .
\end{equation}

From this, we see that the expected error of the Bayes classifier is the infimum of the misclassification errors of any classifier (for the distribution $\mu$ on $\Omega \times \{0, 1\}$). Hence the Bayes classifier achieves the Bayes error, and any classifier has a misclassification error which is at least that of the Bayes classifier.
\end{theorem}
\begin{proof}
It suffices for us to show that for all $x \in \Omega$,
\begin{equation}
\label{eq:BayesOptimalityTheoremProof1}
\mu(g^*(X) \neq Y | X = x) \leq \mu(g(X) \neq Y | X = x) .
\end{equation}

For any classifier $g: \Omega \to \{0, 1\}$, the following holds:
\begin{align*}
\mu(g(X) \neq Y | X = x) &= 1 - \left( \mu(Y = 1, g(x) = 1 | X = x) + \mu(Y = 0, g(x) = 0 | X = x) \right) \\
&= \begin{dcases}
1 - \mu(Y = 1 | X = x) & \text{if } g(x) = 1 \\
1 - \mu(Y = 0 | X = x) & \text{if } g(x) = 0
\end{dcases} \\
&= 1 - {\eta(x)}^{g(x)} {\left(1 - \eta(x)\right)}^{1 - g(x)}
\end{align*}

We see that the above equality holds with $g^*$ as well, so we have:
\begin{align*}
&\phantom{{}={}}\mu(g(X) \neq Y | X = x) - \mu(g^*(X) \neq Y | X = x) \\
&= 1 - {\eta(x)}^{g(x)} {\left(1 - \eta(x)\right)}^{1 - g(x)} - \left(1 - {\eta(x)}^{g^*(x)} {\left(1 - \eta(x)\right)}^{1 - g^*(x)}\right) \\
&= {\eta(x)}^{g^*(x)} {\left(1 - \eta(x)\right)}^{1 - g^*(x)} - {\eta(x)}^{g(x)} {\left(1 - \eta(x)\right)}^{1 - g(x)} \\
&= \begin{dcases}
0 & \text{if } g(x) = g^*(x) \\
2 \eta(x) - 1 & \text{if } g^*(x) = 1 \text{ and } g(x) = 0 \\
1 - 2 \eta(x) & \text{if } g^*(x) = 0 \text{ and } g(x) = 1
\end{dcases}
\end{align*}

Since $g^*(x) = 1$ if and only if $\eta(x) \geq 1/2$, we have:
\begin{itemize}
\item $2 \eta(x) - 1 > 0$ when $g^*(x) = 1$.
\item $1 - 2 \eta(x) \geq 0$ when $g^*(x) = 0$.
\end{itemize}

We therefore have
\begin{equation*}
\mu(g(X) \neq Y | X = x) - \mu(g^*(X) \neq Y | X = x) \geq 0.
\end{equation*}

Hence we have that equation \eqref{eq:BayesOptimalityTheoremProof1} holds, and so the theorem is proven.
\end{proof}

In order to construct the regression function $\eta$, we need to know the underlying distribution $\mu$, which we do not have access to. This means we cannot compute the regression function $\eta$ directly and simply use the Bayes classifier. The regression function is still a powerful theoretical notion which is very useful in proving various inequalities. We will often consider various estimates to the regression function, some of which can be constructed empirically from the data set.

\subsection{An Example}

We now illustrate a classical and simple example of a learning rule and classifier, and show it is consistent for a distribution but is inconsistent for another distribution. Suppose we have the distribution $\mu$ on $[0, 1] \times \{0, 1\}$ that takes $(0, 0)$ with probability $1/2$ (that is, a point mass at zero, with label zero), and otherwise (with probability $1/2$) is uniformly distributed on $(0, 1]$ with label $1$. That is, there is a point mass at 0 with label 0, and otherwise it is uniformly distributed on the rest of the interval with label 1. We see that the regression function $\eta: [0, 1] \to [0, 1]$ is
\begin{equation}
\eta(x) = \begin{dcases}
0 & \text{if } x = 0 \\
1 & \text{otherwise} .
\end{dcases}
\end{equation}
Hence, the Bayes classifier classifies the point 0 as having label 0, and any other point as label 1. The Bayes error is zero, since the label is a deterministic function of the point.

A simple learning rule is the \emph{nearest neighbour} learning rule ($1$-NN). In $1$-NN, for a query $X$ we assign the label of the nearest point in the dataset to $X$ (we use the usual metric $d(x, y) = |x - y|$ here). For our distribution $\mu$, a point is misclassified if it is at the point $x = 0$ and is assigned label 1 or is not at zero ($x \neq 0$) but is assigned label 0. Suppose we have an iid labelled sample of $n$ points $(X_1, Y_1), (X_2, Y_2), \dots, (X_n, Y_n)$. There is a $1/2$ probability of the query being at 0 and having label 0. In this case, the query will be misclassified if and only if none of the points in the sample are at zero (that is, all of the points are not at zero and have label 1). The probability of this occurring for a sample of $n$ points is $2^{-n}$, which goes to zero as $n \to \infty$. In the other case, with $1/2$ probability, the query is nonzero and has label 1. The only way that the query will be misclassified in this case is if either there are no points in the sample with label 1 or the nearest point with label 1 is further from the query than zero. We see that the probability of either of these occurring goes to zero as $n$ approaches infinity. Hence we find that the point is classified correctly with probability approaching one as $n \to \infty$ (equivalently, the error goes to zero as $n \to \infty$), and so $1$-NN is consistent for this distribution. This is an example of learning a \emph{deterministic concept} (that is, the Bayes error is zero), for which 1-NN is always consistent (this is an immediate consequence of Theorem 5.4 in \cite{pbook}).

Now, suppose we take the distribution $\nu$ on $[0, 1] \times \{0, 1\}$ that is uniform on $[0, 1]$ such that the label 0 occurs with probability $1/3$ and label 1 occurs with probability $2/3$, with the label being independent of the point in $[0, 1]$ (this is an example of a ``probabilistic" or ``fuzzy" concept, as opposed to a deterministic concept). We notice that the Bayes error is $1/3$ and is attained by predicting the label 1 always. We now find the expected error of the 1-NN classifier. Given a query $X$, there is a $1/3$ chance of the label being zero and a $2/3$ chance of the label being one. The nearest neighbour $X_{(1)}$ also has a $1/3$ probability of being label zero and a $2/3$ probability of being label one, independent of the label of $X$. In the 1-NN classifier, we assign the label of $X_{(1)}$ to the query $X$. Hence the probability that we will misclassify $X$ (for any sample size $n \geq 1$) is $(1/3) (2/3) + (2/3) (1/3) = 4 / 9$, which is greater than $1 / 3$. Hence the 1-NN learning rule is not consistent for the distribution $\nu$. This implies that the 1-NN learning rule is not universally consistent, even though it is consistent for the distribution $\mu$ above.

\chapter{The $k$-Nearest Neighbour Classifier}

In this chapter we discuss one of the most important learning rules for classifying points, the $k$-nearest neighbour classifier ($k$-NN). We first start by briefly discussing $k$-NN with an example, we then give a precise mathematical formulation of $k$-NN and we present the proof that it is universally consistent (provided that $k \to \infty$ and $k / n \to 0$ as $n \to \infty$).

\section{The $k$-Nearest Neighbour Classifier}

Suppose we have a set of points in a metric space $\Omega$, with each point assigned a label $0$ or $1$. Let $(X_1, Y_1), (X_2, Y_2), \dots, (X_n, Y_n)$ be a labelled sample and let $(X, Y)$ be the query. In the $k$-nearest neighbour classifier, we predict the label of the query based on which class is more common among the $k$ closest points to $X$ in the labelled sample. We illustrate an example of this in Figure \ref{fig:KnnImage}.

\begin{figure}
\centering
\includegraphics{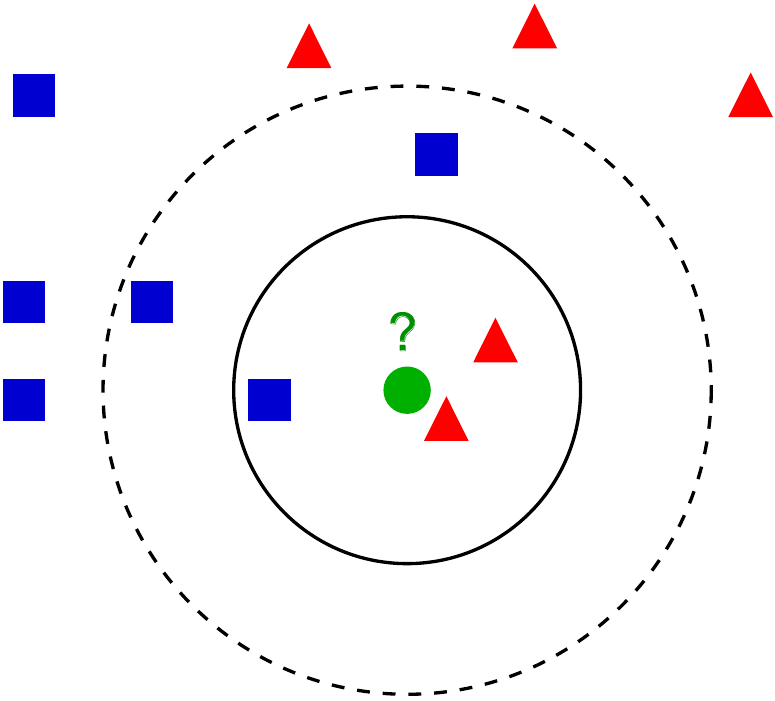}
\caption{An example of $k$-NN. We classify the query as a triangle for $k = 3$, and as a square for $k=5$. Image from \cite{KnnImage}.}
\label{fig:KnnImage}
\end{figure}

We have selected $k$ to be odd in our example to avoid the case of ties. There are two possible cases where ties can occur in our algorithm: it is possible to have multiple classes occurring equally frequently among the $k$-nearest neighbours of the query, and it is possible to have distances ties with multiple points at the same distance from the query. Many authors discuss consistency for distributions with a density to avoid the case of distance ties, however, we will prove universal consistency here and will not make such assumptions. Various methods for breaking ties are discussed in the literature. One common way to break ties is by random selection, so that if a voting tie occurs we pick randomly from the most common labels, and if a distance tie occurs we pick a random point at that distance. For our purposes for the binary classification problem, we will break voting ties (with the same number of points in each class, for a given $k$) by simply selecting the label 1. We break distance ties by generating independent random variables $U_1, U_2, \dots, U_n$ from the uniform distribution on $[0, 1]$, if there is a distance tie between two points $X_i$ and $X_j$, we select $X_i$ if $U_i > U_j$ and $X_j$ if $U_j > U_i$ (we can ignore ties between $U_i$ and $U_j$, since the probability that $U_i = U_j$ is zero). Example pseudocode of $k$-NN is shown in Algorithm \ref{algo:Knn}.

\begin{algorithm}
\caption{$k$-NN pseudocode}
\label{algo:Knn}
\begin{algorithmic}
\REQUIRE $k \in \N$, $X$ is the domain, $Y$ is the response (must be a finite set $\{1, 2, ..., p\}$), $a \in X$, $(x_1, y_1)$, ..., $(x_{n}, y_{n}) \in X \times Y$
\STATE \COMMENT{Calculate distances from input point to all the data points}
\FOR{$i=1$ \TO $n$}
\STATE $d_{i} \leftarrow d(a, x_{i})$
\ENDFOR
\STATE \COMMENT{Find response for the $k$ nearest neighbours of the input point}
\FOR{$i=1$ \TO $k$}
\STATE $m \leftarrow \underset{m}{\operatorname{arg\,min}} \{ d_{m} $ such that $1 \leq m \leq n$ not previously selected$ \}$
\STATE $a_{i} \leftarrow y_{m}$
\ENDFOR
\STATE \COMMENT{Find number of times each response occurs among the $k$ nearest neighbours}
\FOR{$i=1$ \TO $p$}
\STATE $v_{i} \leftarrow$ number of times $i$ occurs in $\{a_1, a_2, ..., a_{p}\}$
\ENDFOR
\STATE $r \leftarrow \{ y_{i} | 1 \leq i \leq p$ such that $v_{i}$ is maximal among $v_1, v_2, ..., v_{p} \}$ \COMMENT{Find the most common response among the $k$ nearest neighbours, if multiple responses are the most common, pick a fixed one}
\RETURN $r$ \COMMENT{Return most common response (or if a tie occurs, one of the most common responses)}
\end{algorithmic}
\end{algorithm}

We would like to establish that the $k$-NN classifier is universally consistent with the data points being independent and identically distributed. There are at least two known ways to do this, the first is the original proof by Stone which uses Stone's theorem, which we state and prove below, and another is the alternative proof that uses the Lebesgue-Besicovitch differentiation theorem, which was originally done in \cite{Devroye} and further discussed in \cite{KnnInf}.

\section{Stone's Theorem}

The original way in which $k$-NN was shown to be universally consistent was \emph{Stone's theorem}, named after Charles Stone.\cite{Stone} This was the first time any learning rule was shown to be universally consistent. We show that any classifier of a particular form that satisfies certain conditions is universally consistent, and then show that the $k$-NN classifier satisfies these conditions. We prove a slightly stronger version of the original Stone's theorem (the slight strengthening will be used later to assist in the proof of some results). Stone's theorem is the foundation for the results we will prove later on, that is why we discuss the proof (of Stone's theorem and the universal consistency of $k$-NN) in detail (following the approach in \cite{pbook} and \cite{Duan}).

Let $\Omega$ be the domain, with $\mu$ being a probability measure and $\eta$ be the regression function on $\Omega \times \{0, 1\}$. Let $(X_1, Y_1), \dots, (X_n, Y_n)$ be a labelled sample. We define real-valued \emph{weights} $W_{ni}(X, (X_1, Y_1), \dots, (X_n, Y_n), U_1, \dots, U_n, V)$ that are functions of the query $X$, the labelled sample, the tiebreakers $U_1, \dots, U_n$, and possibly a random variable $V$ that is independent of all the other random variables, such that they are nonnegative and sum to one,
\begin{equation}
\label{eq:Weights}
\sum_{i=1}^n W_{ni}(x) = 1 .
\end{equation}

We then define the estimate $\eta_n$ to the true regression function $\eta$ as the sum of the positive entries multiplied by their weights,
\begin{equation}
\label{eq:RegressionEstimate}
\eta_n(x) = \sum_{i=1}^n Y_i W_{ni}(x) .
\end{equation}

We now define a classifier $\ELL_n$ as follows:
\begin{equation}
\label{eq:WeightedClassifier}
\ELL_n(x) = \begin{dcases}
1 & \mbox{if } \eta_n(x) \geq 1/2 \\
0 & \mbox{otherwise} \\
\end{dcases}
\end{equation}

Given a query $X$, we define $X_{(1)}, \dots, X_{(n)}$ to be the points $X_1, \dots, X_n$ in order of increasing distance from $X$ (in the case of a tie between distances, we generate independent uniform random variables $U_1, U_2, \dots, U_n$ on $[0, 1]$ and we take the point $X_i$ such that the corresponding $U_i$ is larger). In the $k$ nearest neighbour ($k$-NN) learning rule, we take the weights $W_{ni}(X)$ to be $1/k$ if $X_i \in \{X_{(1)}, \dots, X_{(k)}\}$ and 0 otherwise.

We now prove a couple of inequalities, which will be useful for us.
\begin{lemma}
\label{lemma:StonesTheoremTechnicalInequalities}
For all $a, b, c \in \R$, we have the inequalities
\begin{enumerate}
\item ${(a + b)}^2 \leq 2 (a^2 + b^2)$.
\item ${(a + b + c)}^2 \leq 3(a^2 + b^2 + c^2)$.
\end{enumerate}
\end{lemma}

\begin{lemma}
\label{lemma:RegressionFunctionMean}
The expected value of the difference of the value of the regression function at $X_i$ and $Y_i$ is zero,
\begin{equation}
\E\left[\eta(X_i) - Y_i\right] = 0 .
\end{equation}
\end{lemma}
\begin{proof}
We see that (since $Y_i$ is nonzero if and only if $Y_i = 1$)
\begin{align*}
\E\left[\eta(X_i) - Y_i\right] &= \E\left[\eta(X_i)\right] - \E\left[Y_i\right] \\
&= \E\left[\Prob(Y_i = 1 | X_i)\right] - \E\left[Y_i\right] \\
&= \E\left[\E[Y_i | X_i]\right] - \E\left[Y_i\right] \\
&= \E\left[Y_i\right] - \E\left[Y_i\right] \\
&= 0 .
\end{align*}
\end{proof}

We now have a lemma that lets us bound the difference of the expected error and the Bayes error. This is a standard result, the proof below is based on \cite{pbook} (Theorem 6.5) and \cite{Duan} (Theorem 2.2.5), with more details explained.
\begin{lemma}
\label{lemma:WeightedClassifierErrorInequality}
If a classifier $\ELL_n$ is defined as in equation \eqref{eq:WeightedClassifier}, then the error probability satisfies the inequalities
\begin{equation}
\label{eq:WeightedClassifierErrorInequality1}
\err(\ELL_n) - \ell^* \leq 2 \E\left[\left|\eta(X) - \eta_n(X)\right|\right]
\end{equation}
and
\begin{equation}
\label{eq:WeightedClassifierErrorInequality2}
\err(\ELL_n) - \ell^* \leq 2 \sqrt{\E\left[{\left(\eta(X) - \eta_n(X)\right)}^2\right]} .
\end{equation}
\end{lemma}
\begin{proof}
From our proof of Theorem \ref{theorem:BayesClassifierIsOptimal}, we have (where $\ELL^*$ is the Bayes classifier):
\begin{align*}
&\phantom{{}={}}\Prob(\ELL_n(X) \neq Y | X = x) - \Prob(\ELL^*(X) \neq Y | X = x) \\
&= \begin{dcases}
0 & \text{if } \ELL_n(x) = \ELL^*(x) \\
2 \eta(x) - 1 & \text{if } \ELL^*(x) = 1 \text{ and } \ELL_n(x) = 0 \\
1 - 2 \eta(x) & \text{if } \ELL^*(x) = 0 \text{ and } \ELL_n(x) = 1
\end{dcases} \\
&= \left| 2 \eta(x) - 1 \right| \Indicator_{\{\ELL_n(x) \neq \ELL^*(x)\}} \;\parbox{9cm}{ since the above values are always nonnegative by the definition of $\ELL^*$}
\end{align*}

Hence we have
\begin{align*}
\Prob(\ELL_n(x) \neq Y) - \ell_\mu^* &= \Prob(\ELL_n(X) \neq Y) - \Prob(\ELL^*(X) \neq Y) \\
&= \E\left[\Prob(\ELL_n(X) \neq Y | X = x)\right] - \E\left[\Prob(\ELL^*(X) \neq Y | X = x)\right] \\
&= \E\left[\Prob(\ELL_n(X) \neq Y | X = x) - \Prob(\ELL^*(X) \neq Y | X = x)\right] \\
&= \int_{\Omega} \left| 2\eta(x) - 1 \right| \Indicator_{\{\ELL_n(x) \neq \ELL^*(x)\}} \text{d} \mu(x \times \{0, 1\}) \\
&= 2 \int_{\Omega} \left| \eta(x) - 1/2 \right| \Indicator_{\{\ELL_n(x) \neq \ELL^*(x)\}} \text{d} \mu(x \times \{0, 1\}) .
\end{align*}

We see that the function we are integrating can only be nonzero when $\ELL_n(\omega) \neq \ELL^*(\omega)$. If $\ELL_n(x) = 1$ and $\ELL^*(x) = 0$, then $\eta(x) < 1/2$ and $\eta_n(x) \geq 1/2$. Similarly, we find that if $\ELL_n(x) = 0$ and $\ELL^*(x) = 1$ then $\eta(x) \geq 1/2$ and $\eta_n(x) < 1/2$. In both cases we have the inequality
\begin{equation*}
\left| \eta(x) - 1/2 \right| \leq \left| \eta(x) - \eta_n(x) \right| .
\end{equation*}

Combining the above results, we find
\begin{align*}
\Prob(\ELL_n(x) \neq Y) - \ell_\mu^* &= 2 \int_{\Omega} \left| \eta(x) - 1/2 \right| \Indicator_{\{\ELL_n(x) \neq \ELL^*(x)\}} \text{d} \mu(x \times \{0, 1\}) \\
&\leq 2 \int_{\Omega} \left| \eta(x) - \eta_n(x) \right| \Indicator_{\{\ELL_n(x) \neq \ELL^*(x)\}} \text{d} \mu(x \times \{0, 1\}) \\
&\leq 2 \int_{\Omega} \left| \eta(x) - \eta_n(x) \right| \text{d} \mu(x \times \{0, 1\}) \\
&= 2 \E\left[\left|\eta(X) - \eta_n(X)\right|\right] .
\end{align*}

We have now proven the first inequality \eqref{eq:WeightedClassifierErrorInequality1}. The second inequality \eqref{eq:WeightedClassifierErrorInequality2} follows by applying Jensen's inequality, so we find
\begin{align*}
\Prob(\ELL_n(x) \neq Y) - \ell_\mu^* &\leq 2 \E\left[\left|\eta(X) - \eta_n(X)\right|\right] \\
&\leq 2 \sqrt{\E\left[{\left(\eta(X) - \eta_n(X)\right)}^2\right]} .
\end{align*}
\end{proof}

A core result is \emph{Stone's Theorem}, which gives sufficient conditions for $\ELL_n$ to be universally consistent. We state a slightly strengthened version of Stone's theorem below, the only difference from the original version is that the original does not include the $\epsilon_n$ sequence in the first condition and we only require bounded functions in the first condition. The proof of this result is based on \cite{pbook} (Theorem 6.3) and \cite{Duan} (Theorem 2.2.2), with more details added.

\begin{theorem}[Stone's Theorem]
\label{theorem:StonesTheorem}
Suppose a learning rule $(\ELL_n)_{n=1}^\infty$ is defined as in \eqref{eq:WeightedClassifier}, with the domain $\Omega$ being $\R^d$. Then if the following conditions hold (for any probability distribution of $(X, Y), (X_1, Y_1), \dots, (X_n, Y_n)$ on $\R^d \times \{0, 1\}$, with the points being iid), $(\ELL_n)_{n=1}^\infty$ is universally consistent.

\begin{enumerate}
\item There exists a constant $c \in \R$ and a sequence ${\left(\epsilon_n\right)}_{n=1}^\infty$ that goes to zero, $\epsilon_n \to 0$ as $n \to \infty$, such that for every measurable nonnegative function $f: \R^d \to \R$ bounded above by one, for all $n \geq 1$,
\begin{equation}
\label{eq:StonesTheoremCondition1}
\E \left[ \sum_{i=1}^n W_{ni} (X) f(X_i) \right] \leq c \E \left[ f(X) \right] + \epsilon_n .
\end{equation}

\item There exists a norm $\lVert \cdot \lVert$ such that for all $a > 0$, as $n \to \infty$,
\begin{equation}
\label{eq:StonesTheoremCondition2}
\E \left[ \sum_{i=1}^n W_{ni} (X) \mathds{1}_{\{\lVert X_i - X \lVert > a\}} \right] \to 0 .
\end{equation}

\item As $n \to \infty$, 
\begin{equation}
\label{eq:StonesTheoremCondition3}
\E \left[ \max_{1 \leq i \leq n} W_{ni} (X) \right] \to 0 .
\end{equation}
\end{enumerate}
\end{theorem}
\begin{proof}
For our proof, we show that
\begin{equation}
\E\left[{(\eta(X) - \eta_n(X))}^2\right] \to 0 \text{ as } n \to \infty .
\end{equation}

By Lemma \ref{lemma:WeightedClassifierErrorInequality}, this implies that $\err{\ELL_n} - \ell^* \to 0$ as $n \to \infty$, and hence that $(\ELL_n)_{n=1}^\infty$ is universally consistent.

\begin{itemize}
\item We define another approximation $\hat{\eta}_n$ of the regression function $\eta$ by
\begin{equation}
\label{eq:RegressionFunctionApproximationStones}
\hat{\eta}_n(x) = \sum_{i=1}^n \eta(X_i) W_{ni}(x) .
\end{equation}

We now see that (by Lemma \ref{lemma:StonesTheoremTechnicalInequalities}):
\begin{align}
\label{eq:RegressionFunctionApproximationStonesInequality}
\E\left[ {(\eta(X) - \eta_n(X)}^2 \right] &\leq \E\left[ {(\eta(X) - \hat{\eta}_n(X) + \hat{\eta}_n(X) -  \eta_n(X))}^2 \right] \nonumber \\
&\leq 2 \E\left[ {(\eta(X) - \hat{\eta}_n(X))}^2 \right] + 2 \E\left[ {(\hat{\eta}_n(X) - \eta_n(X))}^2 \right]
\end{align}
We now show that both terms go to zero as $n \to \infty$, by showing that both terms (in \eqref{eq:RegressionFunctionApproximationStonesInequality}) can be made arbitrarily small.

\item We show that for any $\epsilon > 0$, for sufficiently large $n$, $\E\left[ {(\eta(X) - \hat{\eta}_n(X))}^2 \right] < (3c + 12)\epsilon$, and hence $\E\left[ {(\eta(X) - \hat{\eta}_n(X))}^2 \right] \to 0$ as $n \to \infty$.

We first bound this expression by:
\begin{align*}
\E\left[ {(\eta(X) - \hat{\eta}_n(X))}^2 \right] &\leq \E\left[ {\left( \eta(X) - \sum_{i=1}^n \eta(X_i) W_{ni}(X) \right)}^2 \right]\\
&= \E\left[ {\left( \sum_{i=1}^n W_{ni}(X) \eta(X) - \sum_{i=1}^n \eta(X_i) W_{ni}(X) \right)}^2 \right] \\
&= \E\left[ {\left( \sum_{i=1}^n W_{ni}(X) \left( \eta(X) - \eta(X_i) \right) \right)}^2 \right] \\
&\leq \E\left[ \sum_{i=1}^n W_{ni}(X) {(\eta(X) - \eta(X_i))}^2\right] \text{ by Jensen's inequality}
\end{align*}

We observe that any bounded measurable function is square integrable and so is in $L^2(\mu)$, and that continuous functions with bounded support are dense in $L^2(\mu)$ and are uniformly continuous.\cite{Stein} Since $\eta$ is bounded between zero and one and is measurable, this means there exists a uniformly continuous function $\eta^*$ such that $\eta^*(x) \in [0, 1]$ for all $x \in \R^d$ and
\begin{equation}
\E\left[{\left(\eta(X) - \eta^*(X)\right)}^2\right] < \epsilon .
\end{equation}

We then find that (by applying Lemma \ref{lemma:StonesTheoremTechnicalInequalities}):
\begin{align}
&\phantom{{}={}}\E\left[\sum_{i=1}^n W_{ni}(X) {(\eta(X) - \eta(X_i))}^2\right] \nonumber\\
&= \E\left[\sum_{i=1}^n W_{ni}(X) {(\eta(X) - \eta^*(X) + \eta^*(X) - \eta^*(X_i) + \eta^*(X_i) - \eta(X_i))}^2\right] \nonumber\\
&\leq 3 \E\left[\sum_{i=1}^n W_{ni}(X) {(\eta(X) - \eta^*(X))}^2\right] + \nonumber\\&\phantom{{}={}} 3 \E\left[\sum_{i=1}^n W_{ni}(X) {(\eta^*(X) - \eta^*(X_i))}^2\right] + \nonumber\\&\phantom{{}={}} 3\E\left[\sum_{i=1}^n W_{ni}(X) {(\eta^*(X_i) - \eta(X_i))}^2\right] \nonumber
\end{align}

For the first term, we see that
\begin{align*}
\E\left[\sum_{i=1}^n W_{ni}(X) {(\eta(X) - \eta^*(X))}^2\right] &= \E\left[{(\eta(X) - \eta^*(X))}^2 \sum_{i=1}^n W_{ni}(X) \right] \\
&= \E\left[{(\eta(X) - \eta^*(X))}^2 \right] \\
&< \epsilon .
\end{align*}

For the second term $\E\left[\sum_{i=1}^n W_{ni}(X) {(\eta^*(X) - \eta^*(X_i))}^2\right]$, we notice that since $\eta^*$ is uniformly continuous, there exists an $a > 0$ such that if $\norm{X - X_i} \leq a$, then $\left| \eta^*(X) - \eta^*(X_i) \right| < \sqrt{\epsilon}$. First apply this fact (by splitting the expectation into two disjoint sets, the part with $\norm{X - X_i} \leq a$ and the part with $\norm{X - X_i} > a$) and the linearity of integration. We then apply the fact that ${(\eta^*(X) - \eta^*(X_i))}^2 \leq 1$ always, and then use the second condition of the theorem to create a bound. We find that
\begin{align*}
&\phantom{{}={}}\E\left[\sum_{i=1}^n W_{ni}(X) {(\eta^*(X) - \eta^*(X_i))}^2\right] \\
&= \E\left[\sum_{i=1}^n W_{ni}(X) {(\eta^*(X) - \eta^*(X_i))}^2 (\Indicator_{\{\norm{X_i - X} \leq a\}} + \Indicator_{\{\norm{X_i - X} > a\}})\right] \\
&= \E\left[\sum_{i=1}^n W_{ni}(X) {(\eta^*(X) - \eta^*(X_i))}^2 \Indicator_{\{\norm{X_i - X} \leq a\}}\right] \\
&\phantom{{}={}}+ \E\left[\sum_{i=1}^n W_{ni}(X) {(\eta^*(X) - \eta^*(X_i))}^2 \Indicator_{\{\norm{X_i - X} > a\}}\right] \\
&\leq \E\left[\sum_{i=1}^n W_{ni}(X) {\sqrt{\epsilon}}^2 \Indicator_{\{\norm{X_i - X} \leq a\}}\right] + \E\left[\sum_{i=1}^n W_{ni}(X) \Indicator_{\{\norm{X_i - X} > a\}}\right] \\
&< \epsilon \E\left[\sum_{i=1}^n W_{ni}(X) \Indicator_{\{\norm{X_i - X} \leq a\}}\right] + \epsilon \\
&\leq 2 \epsilon .
\end{align*}

For the third term, we see that ${(\eta^*(X_i) - \eta(X_i))}^2$ is bounded above by $1$ and is a measurable function of $X_i$ (since both $\eta$ and $\eta^*$ are bounded), and so by the first assumption, we have
\begin{equation*}
\E\left[\sum_{i=1}^n W_{ni}(X) {(\eta^*(X_i) - \eta(X_i))}^2\right] < c \epsilon + \epsilon_n .
\end{equation*}

We see that $\epsilon_n \to 0$ as $n \to \infty$, so we require $n$ to be sufficiently large such that $\epsilon_n < \epsilon$. We then find that
\begin{equation*}
\E\left[\sum_{i=1}^n W_{ni}(X) {(\eta^*(X_i) - \eta(X_i))}^2\right] < (c + 1) \epsilon .
\end{equation*}

Combining the results for these three terms, we find that for all sufficiently large $n$,
\begin{align*}
\E\left[\sum_{i=1}^n W_{ni}(X) {(\eta(X) - \eta(X_i))}^2\right] &< 3 \epsilon + 3 (2 \epsilon) + 3 (c + 1) \epsilon \\
&= (3c + 12) \epsilon .
\end{align*}

Since $(3c + 12) \epsilon$ can be made arbitrarily small by taking $\epsilon$ to be sufficiently small, it follows that $\E\left[ \sum_{i=1}^n W_{ni}(X) {(\eta(X) - \eta(X_i)}^2\right] \to 0$ as $n \to \infty$.

\item We now show that $\E\left[ {(\hat{\eta}_n(X) - \eta_n(X))}^2 \right] \to 0$ as $n \to \infty$. We directly substitute in the definition of $\eta_n$ and $\hat{\eta}_n$ into the expression and simplify. We obtain:
\begin{align*}
\E\left[ {\left( \hat{\eta}_n(X) - \eta_n(X) \right)}^2 \right] &= \E\left[ {\left( \sum_{i=1}^n \eta(X_i) W_{ni}(X) - \sum_{i=1}^n Y_i W_{ni}(X) \right)}^2 \right] \\
&= \E\left[ {\left( \sum_{i=1}^n W_{ni}(X) (\eta(X_i) - Y_i) \right)}^2 \right] \\
&= \E\left[ \sum_{i=1}^n \sum_{j=1}^n W_{ni}(X) W_{nj}(X) (\eta(X_i) - Y_i) (\eta(X_j) - Y_j) \right] \\
&= \sum_{i=1}^n \sum_{j=1}^n \E\left[ W_{ni}(X) W_{nj}(X) (\eta(X_i) - Y_i) (\eta(X_j) - Y_j) \right] \\
\end{align*}

If $i \neq j$, we first apply the law of total expectation (in which we condition on $X, X_1, X_2, \dots, X_n$ in the inner expectation), after which we notice that $W_{ni}(X)$, $W_{nj}(X)$, $(\eta(X_i) - Y_i)$, and $(\eta(X_j) - Y_j)$ are all conditionally independent with respect to $X, X_1, X_2, \dots, X_n$,\footnote{This holds since $W_{ni}(X)$ and $W_{nj}(X)$ are assumed to be functions of $X, X_1, X_2, \dots, X_n$ only. If this condition is violated, the theorem fails, see counterexample \ref{example:StonesTheoremLabelCounterexample}.} which means that we can split the inner expectation, so we find:
\begin{align*}
&\E\left[ W_{ni}(X) W_{nj}(X) (\eta(X_i) - Y_i) (\eta(X_j) - Y_j) \right] \\
&= \E\left[ \E\left[ W_{ni}(X) W_{nj}(X) (\eta(X_i) - Y_i) (\eta(X_j) - Y_j) \middle| X, X_1, X_2, \dots, X_n \right] \right] \\
&= \E\left[ \substack{ \E\left[ W_{ni}(X) \middle| X, X_1, X_2, \dots, X_n \right] \times \E\left[  W_{nj}(X)  \middle| X, X_1, X_2, \dots, X_n \right] \times \\ \E\left[ (\eta(X_i) - Y_i)  \middle| X, X_1, X_2, \dots, X_n \right] \times \E\left[ (\eta(X_j) - Y_j) \middle| X, X_1, X_2, \dots, X_n \right] } \right]
\end{align*}

We then notice that (since $Y_i$ takes on values zero and one only, so we can replace the expected value of $Y_i$ with the probability that $Y_i = 1$):
\begin{align*}
&\E\left[ \eta(X_i) - Y_i \middle| X, X_1, X_2, \dots, X_n \right] \\
&=\E\left[ \eta(X_i) \middle| X, X_1, X_2, \dots, X_n \right] - \E\left[ Y_i \middle| X, X_1, X_2, \dots, X_n \right] \\
&=\E\left[ \Prob(Y_i = 1 | X_i) \middle| X, X_1, X_2, \dots, X_n \right] - \Prob\left( Y_i = 1 \middle| X, X_1, X_2, \dots, X_n \right) \\
&=\Prob\left( Y_i = 1 \middle| X, X_1, X_2, \dots, X_n \right) - \Prob\left( Y_i = 1 \middle| X, X_1, X_2, \dots, X_n \right) \\
&= 0
\end{align*}

This implies that the expected value $\E\left[ W_{ni}(X) W_{nj}(X) (\eta(X_i) - Y_i) (\eta(X_j) - Y_j) \right]$ (with $i \neq j$) is zero, since one of the factors in the expectation is zero (namely $\E\left[ \eta(X_i) - Y_i \middle| X, X_1, X_2, \dots, X_n \right] = 0$) and all of the factors are finite. This means that the cross terms are all zero. Hence we have that the expectation is equal to the terms with $i = j$,
\begin{align*}
\E\left[ {\left( \hat{\eta}_n(X) - \eta_n(X) \right)}^2 \right] &= \sum_{i=1}^n \E\left[ {W_{ni}(X)}^2 {(\eta(X_i) - Y_i)}^2 \right] \\
&\leq \sum_{i=1}^n \E\left[ {W_{ni}(X)}^2 \right] \text{ since } {(\eta(X_i) - Y_i)}^2 \leq 1 \\
&\leq \sum_{i=1}^n \E\left[ W_{ni}(X) \max_{1 \leq i \leq n} W_{ni}(X) \right] \\
&= \E\left[ \max_{1 \leq i \leq n} W_{ni}(X) \sum_{i=1}^n W_{ni}(X) \right] \\
&= \E\left[ \max_{1 \leq i \leq n} W_{ni}(X) \right] \text{ since } \sum_{i=1}^n W_{ni}(X) = 1 \text{ always}\\
&\to 0 \text{ as } n \to \infty \text{ by the third condition.}
\end{align*}

\end{itemize}
\end{proof}

It can be shown that the $k$-NN learning rule on $\R^d$ (with any norm on $\R^d$) satisfies these conditions and so is universally consistent. This is what we will do in the next section.

\section{Universal Consistency of $k$-NN}

In this section, we prove that $k$-NN on the normed space $(\R^d, \norm{\cdot})$ is universally consistent. This is a known result, we explain the proof in detail as we will consider various extensions of this result later on. For the Euclidean norm, the result was first proven by Stone in \cite{Stone}. A nice version of the proof for the Euclidean norm was presented in the book \cite{pbook}, the result for arbitrary norms appears to have been known to the authors of the book but was not proven. The full proof for arbitrary norms is done in \cite{Duan}.

We observe that the weight function for $k$-NN is:
\begin{equation}
\label{eq:KnnWeights}
W_{ni}(X) = \begin{dcases}
\frac1k & \text{if } X_i \text{ is a } k \text{-nearest neighbour of } X \\
0 & \text{otherwise}
\end{dcases}
\end{equation}

We notice that the weights are all nonnegative and sum to one. We then classify points using this weight function with equations \eqref{eq:RegressionEstimate} and \eqref{eq:WeightedClassifier}.

An \emph{inframetric space with a $C$-inframetric inequality} $(\Omega, \rho)$ is a nonempty set $\Omega$ together with a function $\rho: \Omega \times \Omega \to \R^+$ that satisfies:
\begin{enumerate}
\item $\rho(x, y) \geq 0$ and $\rho(x, y) = 0$ if and only if $x = y$ for all $x, y \in \Omega$.
\item $\rho(x, y) = \rho(y, x)$ for all $x, y \in \Omega$.
\item $\rho(x, z) \leq C \cdot \max\{\rho(x, y), \rho(y, z)\}$.
\end{enumerate}

We easily see that any $C$-inframetric space satisfies a $2C$-weakened triangle inequality, that for all $x, y, z \in \Omega$, $\rho(x, z) \leq 2C (\rho(x, y) + \rho(y, z))$. As for metric spaces, we can define the notion of an open ball, open set, dense subset, separability, Borel $\sigma$-algebra, etc. for inframetric spaces. This is done for a more general family of symmetric kernels in \cite{Assouad}. First the open ball $B_r(y, \rho)$ is defined as the set $\{x \in \Omega \;|\; \rho(x, y) < r \}$ with the closed ball and sphere defined similarly (\cite{Assouad}, part 1.1). Open sets, the notion of separability, and the Borel $\sigma$-algebra are then defined. The theory of measures is developed on such spaces.

\begin{definition}
The \emph{support} of a measure $\mu$ is the set of points such that any open ball around any such point has nonzero measure, that is,
\begin{equation}
\operatorname{Support}(\mu) = \{ \mu \;:\; \forall r > 0,\; \mu(B_r(x)) > 0 \} .
\end{equation}
\end{definition}

It can be easily shown that the support of a measure is always closed. We now prove a standard result about the support of a measure on an inframetric spaces (a sketch of the proof for $(\R^d, \norm{\cdot})$ can be found in \cite{pbook} (Appendix 1, Lemma A.1), which works for any metric space). The result for spaces with a symmetric kernel satisfying a $C$-relaxed triangle inequality is proven in \cite{Assouad}, Proposition 2.6.2.

\begin{lemma}
\label{lemma:SeparableSpaceSupport}
The complement of the support has $\mu$-measure zero in any separable $C$-inframetric space.
\end{lemma}
\begin{proof}
We let $A$ be the support of $\mu$ and $T$ be a countable dense subset of $\Omega$. By definition,
\begin{equation}
A^\mathcal{C} = \{x \in \Omega \;:\; \exists r > 0, \; \mu(B_r(x)) = 0 \} .
\end{equation}

We let $x \in A^\mathcal{C}$, and $r > 0$ be a radius around $x$ such that $\mu(B_r(x)) = 0$. Without loss of generality we assume that $r$ is rational. We see that there exists $y \in T$ such that $\rho(x, y) < \frac{r}{4 C}$, and for any $z \in B_{r / (4 C)}(y)$,
\begin{align*}
\rho(x, z) &\leq 2C(\rho(x, y) + \rho(y, z)) \\
&< 2C \left( \frac{r}{4C} + \frac{r}{4C} \right) \\
&= r .
\end{align*}

This means that $z \in B_r(X)$, and hence $\mu\left( B_{r / (2 C)} (y) \right) = 0$.

By the above argument, we see that for all $x \in A^\mathcal{C}$, there exists $y_x \in T$ and rational $r_x > 0$ such that $x \in B_{r_x}(y_x)$ and $\mu(B_{r_x}(y_x)) = 0$. We define a family of such open balls
\begin{equation}
\B = \left\{ B_r(y) \;:\; r > 0,\; r \in \Q,\; y \in T,\; \mu(B_r(y)) = 0 \right\} .
\end{equation}

It is clear that $\B$ is countable (since the countable union of countable sets is countable) and that every $x \in \A^\mathcal{C}$ is in $\B$, since it is in at least one of the open balls, and hence $A^\mathcal{C} \subseteq \cup_{B \in \B} B$. We then find (using subadditivity)
\begin{align*}
\mu(A^\mathcal{C}) &\leq \mu\left( \bigcup_{B \in \B} B \right) \\
&\leq \sum_{B \in \B} \mu(B) \\
&= \sum_{B \in \B} 0 \\
&= 0 .
\end{align*}
\end{proof}

\begin{lemma}
\label{lemma:ConvergenceOfSupremumOfTail}
Let ${(A_n)}_{n=1}^\infty$ be a sequence of random variables that converges almost surely to zero as $n$ approaches infinity. We then have that the supremum of the tail also converges almost surely to zero, $\sup_{m \geq n} A_m \to 0$ with probability one as $n \to \infty$.
\end{lemma}
\begin{proof}
We first prove this for deterministic sequences. Let ${(a_n)}_{n=1}^\infty$ be a deterministic sequence that converges to zero. For any $\epsilon > 0$, there exists $N \geq 1$ such that for all $n \geq N$, $\left| a_n \right| < \epsilon / 2$. This means that $\sup_{m \geq N} a_m \leq \epsilon / 2 < \epsilon$. We observe the sequence ${\left(\sup_{m \geq N} a_m \right)}_{n=1}^\infty$ is monotone decreasing. Hence we have that for all $n \geq N$, $\sup_{m \geq n} a_m \leq \sup_{m \geq N} a_m < \epsilon$. It follows that for all $\epsilon > 0$, there exists $N \geq 1$ such that for all $n \geq N$, $\sup_{m \geq n} a_m < \epsilon$, and hence $\sup_{m \geq n} a_m \to 0$ as $n \to \infty$.

If $A_n$ converges almost surely to zero, we have that $A_n(\omega) \to 0$ as $n \to \infty$ for all $\omega \in \Omega_0$, for some subset $\Omega_0$ of the probability space with $\Prob(\Omega_0) = 1$. Since the result holds for deterministic sequences, we have that for all points $\omega \in \Omega_0$, $\sup_{m \geq n} A_m(\omega) \to 0$ as $n \to \infty$, and hence $\sup_{m \geq n} A_m$ converges to zero almost surely as $n \to \infty$.
\end{proof}

This following result is a generalization of a result originally proved by Cover and Hart in \cite{CoverHart}, see also \cite{pbook} (Lemma 5.1). Our new result extends the result to inframetric spaces.
\begin{lemma}
\label{lemma:ProbKPointIsFarGoesToZero}
Suppose we have a separable inframetric space $(\Upsilon, \rho)$ with probability measure $\Prob$. Given iid points $X, X_1, X_2, \dots, X_n$, let $X_{(1, \rho)}, X_{(2, \rho)}, \dots, X_{(n, \rho)}$ be the points in increasing distance from $X$ with respect to the metric $\rho$ and let $a > 0$ be a constant. Then as $n \to \infty$, for any sequence ${(k_n)}_{n=1}^\infty$ such that $\frac{k_n}{n} \to 0$,
\begin{equation}
\Prob( \rho(X_{\left(k_n\right)}, X) > a ) \to 0 .
\end{equation}
\end{lemma}
\begin{proof}
We first notice that for all $x \in \operatorname{Support}(\Prob)$ and $\epsilon > 0$,
\begin{align*}
\rho(X_{(k_n)}(x), x) \geq \epsilon &\Leftrightarrow \sum_{i=1}^n \Indicator_{\{X_i \in B_\epsilon(x)\}} < k_n \\
&\Leftrightarrow \frac1n \sum_{i=1}^n \Indicator_{\{X_i \in B_\epsilon(x)\}} < \frac{k_n}{n} .
\end{align*}

We then notice that $k_n / n \to 0$ as $n \to \infty$ by assumption and $\frac1n \sum_{i=1}^n \Indicator_{\{X_i \in B_\epsilon(x)\}} \to \Prob(B_\epsilon(x))$ almost surely as $n \to \infty$ by the strong law of large numbers, and by assumption $\mu(B_\epsilon(x)) > 0$ since $x$ is in the support of $\mu$. It follows that $\rho(X_{(x)} - x) \to 0$ as $n \to \infty$ with probability one.

We define $X_{(k, n)}(x)$ to be the $k^\text{th}$ order statistic of $x$ from the sample $X_1, X_2, \dots, X_n$ of size $n$ in the $\rho$ distance (this notation makes the sample size clear when we discuss the order statistics). By the above argument we have that for any $x$ in the support of $\Prob$, we have $\rho(X_{(k_n, n)}(x) - x) \to 0$ as $n \to \infty$ with probability one. From Lemma \ref{lemma:ConvergenceOfSupremumOfTail} we have that $\sup_{m \geq n} \rho(X_{(k_m, m)}(x) - x) \to 0$ as $n \to \infty$ with probability one as well.

We then notice that for the random variable $X$, by Lemma \ref{lemma:SeparableSpaceSupport},
\begin{align*}
&\phantom{{}={}} \Prob\left(\sup_{m \geq n} \rho(X_{(k_m, m)}(X), X) > \epsilon\right) \\
&= \Prob(X \in \operatorname{Support}(\Prob)) \Prob\left(\sup_{m \geq n} \rho(X_{(k_m, m)}(X), X) > \epsilon \middle| X \in \operatorname{Support}(\Prob) \right) +\\&\phantom{{}={}} \Prob(X \not\in \operatorname{Support}(\Prob)) \Prob\left(\sup_{m \geq n} \rho(X_{(k_m, m)}(X), X) > \epsilon \middle| X \not\in \operatorname{Support}(\Prob) \right) \\
&= \Prob\left(\sup_{m \geq n} \rho(X_{(k_m, m)}(X), X) > \epsilon \middle| X \in \operatorname{Support}(\Prob) \right) .
\end{align*}

Since the sequence $\sup_{m \geq n} \rho(X_{(k_m, m)}(x), x)$ is nonnegative, monotone decreasing, and converges to zero almost everywhere if $X\in \operatorname{Support}(\Prob)$, by the Monotone Convergence Theorem (applied to the expectations of the indicator functions of the events $\sup_{m \geq n} \rho(X_{(k_m, m)}(X), X) > \epsilon$) we find that the conditional probability $\Prob\left(\sup_{m \geq n} \rho(X_{(k_m, m)}(X), X) > \epsilon \middle| X \in \operatorname{Support}(\Prob) \right) \to 0$ as $n \to \infty$. Since $0 \leq \rho(X_{(k_n, n)}(X), X) \leq \sup_{m \geq n} \rho(X_{(k_m, m)}(X), X)$, we have that as $n \to \infty$,
\begin{equation*}
\Prob\left(\rho(X_{(k_n, n)}(X), X) > \epsilon\right) \leq \Prob\left(\sup_{m \geq n} \rho(X_{(k_m, m)}(X), X) > \epsilon \middle| X \in \operatorname{Support}(\Prob) \right) \to 0 .
\end{equation*}
\end{proof}

We recall that the \emph{weights} $W_{ni}(X)$ (with $1 \leq i \leq n$) are functions of $X,\allowbreak X_1,\allowbreak X_2,\allowbreak \dots,\allowbreak X_n$ that are nonnegative and sum to one, and from equation \eqref{eq:RegressionEstimate} if the sum of the weights $W_{ni}(X)$ multiplied by the corresponding $Y_i$ is at least $1/2$, we assign label one, otherwise we assign label zero. For $k$-NN, we recall that the weights are $1/k$ for the $k$-nearest points to the query, and are zero otherwise. The following result follows easily from Lemma \ref{lemma:ProbKPointIsFarGoesToZero} and is proven in \cite{pbook} (inside the proof of Theorem 6.4).

\begin{lemma}
\label{lemma:StonesTheoremSecondConditionNorm}
If we let $W_{ni}(X)$ be the weights in the $k$-NN learning rule for the normed space $(\R^d, \norm{\cdot})$, then
\begin{equation}
\E\left[ \sum_{i=1}^n W_{ni}(X) \Indicator_{\{\norm{X_i - X} > a \}} \right] \to 0 \text{ as } n \to \infty .
\end{equation}
\end{lemma}
\begin{proof}
We see that $\sum_{i=1}^n W_{ni}(X) \Indicator_{\{\norm{X_i - X} > a \}}$ is bounded above by one, is nonnegative, and is nonzero if and only if the $k^{\text{th}}$ nearest point to $X$ has a distance of at most $a$. By Lemma \ref{lemma:ProbKPointIsFarGoesToZero}, the probability of this goes to zero as $n$ goes to infinity, and hence the expected value $\E\left[ \sum_{i=1}^n W_{ni}(X) \Indicator_{\{\norm{X_i - X} > a \}} \right] \to 0$ as $n \to \infty$.
\end{proof}

\begin{lemma}
\label{lemma:UnitSphereHasFiniteCovering}
For any norm $\norm{\cdot}$ on $\R^d$ and radius $\delta > 0$, the unit sphere $S_1(\veczero, \norm{\cdot})$ can be covered by $c$ balls of radius $\delta$ each in the $\norm{\cdot}$ norm, that is,
\begin{equation}
S_1(\veczero, \norm{\cdot}) \subseteq \bigcup_{i=1}^c B_\delta(\vecx_i, \norm{\cdot}) .
\end{equation}
\end{lemma}
\begin{proof}
This holds since $S_1(\veczero, \norm{\cdot})$ is a bounded subset of $\R^d$.
\end{proof}

The result that $k$-NN satisfies the first condition in Stone's theorem is called \emph{Stone's Lemma}. We now present the proof of Stone's lemma for any norm on $\R^d$ (the \emph{Generalized Stone's Lemma}, as originally Stone's Lemma was only proved for the Euclidean norm in \cite{Stone}, and a modified version of the proof was presented in \cite{pbook}, with cones of angle $\pi / 6$). The result for arbitrary norms on $\R^d$ is given as an exercise in \cite{pbook} (Chapter 5, Problem 5.1). The first published proof of the general result (for any norm on $\R^d$) that I am aware of is in \cite{Duan} (Lemma 2.2.9).
\begin{lemma}
\label{lemma:UnitSphereGeneratesFiniteConeCovering}
Let $c$ be the number of subsets such that the unit sphere $S_1(\veczero, \norm{\cdot})$ can be covered by $c$ balls of radius $1/4$ each. Then there exist $c$ subsets $S_1, S_2, \dots, S_c$ (with each of them containing the zero vector) covering $\R^d$ such that in every subset $S_q$ (with $1 \leq q \leq c$), if $\vecx, \vecy \in S_q$ with $\norm{\vecx} \leq \norm{\vecy}$ and $\vecx \neq \veczero$, then $\norm{\vecy - \vecx} < \norm{\vecy}$.
\end{lemma}
\begin{proof}
We first see by Lemma \ref{lemma:UnitSphereHasFiniteCovering} that there exists a finite covering of $c$ of the unit sphere $S_1(\veczero, \norm{\cdot})$ by open balls of radius $1/4$. We let the points $\vecx_1, \vecx_2, \dots, \vecx_c$ be the centres of the balls of such a covering. For each open ball $B_{1/4}(\vecx_i, \norm{\cdot})$ (with $1 \leq i \leq c$), we define the set $A_i$ by having $\veczero \in A_i$ always and for all $\vecx \neq \veczero$,
\begin{equation}
\vecx \in A_i \Leftrightarrow \frac{\vecx}{\norm{\vecx}} \in B_{1/4}(\mathbf{x}_i, \norm{\cdot}) .
\end{equation}

We then see that the sets $A_1, A_2, \dots, A_c$ cover $\R^d$, since the zero vector is in all of the sets and for every nonzero vector $\vecx$, $\norm{\vecx / \norm{\vecx}}$ lies on the unit sphere which is covered by the above set of open balls and so $\vecx$ is in at least one $A_i$.

We then see that if $\vecx, \vecy \in A_i$, then
\begin{align*}
\norm{\frac{\vecx}{\norm{\vecx}} - \frac{\vecy}{\norm{\vecy}}} &= \norm{\frac{\vecx}{\norm{\vecx}} - \vecx_i + \vecx_i - \frac{\vecy}{\norm{\vecy}}} \\
&\leq \norm{\frac{\vecx}{\norm{\vecx}} - \vecx_i} + \norm{\vecx_i - \frac{\vecy}{\norm{\vecy}}} \\
&< \frac14 + \frac14 \\
&= \frac12 .
\end{align*}

It then follows that
\begin{align*}
\norm{\frac{\vecy \norm{\vecx}}{\norm{\vecy}} - \vecx} &= \norm{\norm{\vecx}\left(\frac{\vecy}{\norm{\vecy}} - \frac{\vecx}{\norm{\vecx}}\right)} \\
&= \norm{\vecx} \norm{\frac{\vecx}{\norm{\vecx}} - \frac{\vecy}{\norm{\vecy}}} \\
&< \frac{\norm{\vecx}}{2} .
\end{align*}

From this, we are able to find that
\begin{align*}
\norm{\vecy - \vecx} &= \norm{\vecy - \frac{\vecy \norm{\vecx}}{\norm{\vecy}} + \frac{\vecy \norm{\vecx}}{\norm{\vecy}} - \vecx} \\
&\leq \norm{\vecy - \frac{\vecy \norm{\vecx}}{\norm{\vecy}}} + \norm{\frac{\vecy \norm{\vecx}}{\norm{\vecy}} - \vecx} \\
&< \norm{\vecy - \frac{\vecy \norm{\vecx}}{\norm{\vecy}}} + \frac{\norm{\vecx}}{2} \\
&= \norm{\left(1 - \frac{\norm{\vecx}}{\norm{\vecy}}\right) \vecy} + \frac{\norm{\vecx}}{2} \\
&= \left(1 - \frac{\norm{\vecx}}{\norm{\vecy}}\right)\norm{\vecy} + \frac{\norm{\vecx}}{2} \\
&= \norm{\vecy} - \frac{\norm{\vecx}}{2} \\
&< \norm{\vecy} .
\end{align*}
\end{proof}

The following result has been proven in \cite{Duan} (Theorem 2.2.8).
\begin{lemma}
\label{lemma:StonesLemmaCore}
Suppose we have the finite dimensional normed vector space $(\R^d, \norm{\cdot})$, we let $f: \R^d \to \R$ be any nonnegative measurable function with finite expected value (in terms of the measure $\mu$), we let $1 \leq k \leq n$, and we define $c$ to be a constant such that $\R^d$ can be partitioned into a finite number of subsets $S_1, S_2, \dots, S_c$, such that if $1 \leq q \leq c$, if $\vecx, \vecy \in S_q$, $\vecx \neq \veczero$, and $\norm{\vecx} \leq \norm{\vecy}$, then $\norm{\vecy - \vecx} < \norm{\vecy}$. We let $X, X_1, X_2, \dots, X_n$ be iid random variables on $\R^d \times \{0, 1\}$ with probability distribution $\mu$. If we let $W_{ni}(X)$ be the weights in $k$-NN with the $\norm{\cdot}$ norm, we find that
\begin{equation}
\E \left[ \sum_{i=1}^n W_{ni}(X) f(X_i) \right] \leq c \E \left[ f(X) \right] .
\end{equation}
\end{lemma}
\begin{proof}
Given a query $X$, we define the subsets $S_1', S_2', \dots, S_c'$ as
\begin{equation}
X_i \in S_q' \Leftrightarrow X - X_i \in S_q .
\end{equation}

We see that since $S_1, S_2, \dots, S_c$ cover $\R^d$ and $X_i - X$ is a vector in $\R^d$, the new subsets  cover $\R^d$, that is, $\R^d$ is covered by $S_1', S_2', \dots, S_c'$. We let $S_q' \in \{S_1', S_2', \dots, S_c'\}$ be one the subsets. In the subset $S_q'$, we mark the $k$ points closest to $X$ in the $\norm{\cdot}$ norm among $\{X_1, X_2, \dots, X_n\} \cap S_q'$ (if there are fewer than $k$ such points, we take all of them, and we break distance ties by generating independent uniform random variables $U_1, U_2, \dots, U_n$ and taking the point such that $U_i$ is larger, as discussed previously). We see that the number of points that are marked in the subset $S_q'$ is at most $k$. If a point $X_i \in S_q'$ is not marked, then there must exist at least $k$ points in $S_q'$ that are either closer to $X$ than $X_i$ in the $\norm{\cdot}$ norm or have the same distance but the independent random variable $U_i$ is larger. Let $X_j \in S_q'$ be such a point. We need to show that $X_j$ is closer to $X_i$ than $X$ is (in the case of ties, the tiebreaking variables $U_1, U_2, \dots, U_n$ define which point is ``closer"). Since $X_i, X_j \in S_q'$, $X - X_i, X - X_j \in S_q$. There are now two possible cases:
\begin{enumerate}[(i)]
\item If $X_j \neq X$, $X - X_j \neq \veczero$, and by assumption $\norm{X - X_j} \leq \norm{X - X_i}$. This implies (by the definition of the subset $S_j$) that 
\begin{align*}
\norm{X_i - X_j} &= \norm{(X_i - X) + (X - X_j)} \\
&= \norm{(X_i - X) - (X_j - X)} \\
&< \norm{X_i - X}
\end{align*}
and so $X_j$ is closer to $X_i$ than $X$ is.
\item Otherwise, if $X_j = X$, then $U_j > U_i$ (these being the tie-breaking variables discussed earlier), so in our tie-breaking rule (for the nearest neighbour) we select $X_j$ before $X$ from the list $\{X_1, \dots, X_{i-1}, X, X_{i+1}, \dots, X_j \dots, X_n\} \cap S_q'$ (where $X$ takes the place of $X_i$ in the list of points, so the same $U_i$ variable is used for $X$ here).
\end{enumerate}

This holds for all the other $k$ points in $S_q'$ that are nearest $X$, and so $X$ is not a $k$ nearest neighbour of $X_i$.

Hence we see that if $X_i$ is not marked, $X$ is not a $k$-nearest neighbour of $X_i$. Equivalently, the set of $k$-nearest neighbours of $X$ is a subset of the set of points that are marked. The number of points that are marked is at most $c k$, since there are $c$ subsets and each subset contains at most $k$ marked points. We see that
\begin{align*}
&\phantom{{}={}}\sum_{i=1}^n \E \left[ W_{ni}(X) f(X_i) \right] \\
&= \E \left[ \sum_{i=1}^n \frac1k \Indicator_{\{ X_i \text{ is a } k \text{-nearest neighbour of } X \text{ in the } \norm{\cdot} \text{ norm among } X_1, X_2, \dots, X_n \}} f(X_i) \right] \\
&= \frac1k \E \left[ \sum_{i=1}^n \Indicator_{\{ X \text{ is a } k \text{-nearest neighbour of } X_i \text{ in the } \norm{\cdot} \text{ norm among } X_1, \dots, X_{i-1}, X, X_{i+1}, \dots, X_n \}} f(X) \right] \\
&= \frac1k \E \left[ f(X) \sum_{i=1}^n \Indicator_{\{ X \text{ is a } k \text{-nearest neighbour of } X_i \text{ in the } \norm{\cdot} \text{ norm among } X_1, \dots, X_{i-1}, X, X_{i+1}, \dots, X_n \}} \right] \\
&\leq \frac1k \E \left[ f(X) \sum_{i=1}^n \Indicator_{\{ X_i \text{ is marked} \}} \right] \\
&\leq \frac1k \E \left[ f(X) (k) c \right] \\
&= c \E\left[f(X)\right] .
\end{align*}
\end{proof}

The result that $k$-NN is universally consistent on the normed space $(\R^d, \norm{\cdot})$ appears to have been known to the authors of \cite{pbook}, where Stone's lemma (which is the first condition for Stone's theorem) for arbitrary norms on $\R^d$ is given as an exercise earlier (Chapter 5, Problem 5.1), universal consistency is proven for the Euclidean norm, and the proof of universal consistency for a fixed $\ell^p$ norm (with $1 \leq p \leq \infty$) is left as an exercise (Chapter 11, Problem 11.4, which recommends proving universal consistency by checking the conditions of Stone's theorem; the statement of the problem in the book includes the $\ell^p$ quasinorms with $0 < p < 1$ (by giving the problem for $0 < p \leq \infty$), which I think is a misprint as the geometric Stone's lemma does not hold for the $\ell^p$ quasinorms with $0 < p < 1$ as we show in Example \ref{example:FailureOfStonesLemmaForQuasinorms} and no alternative approach for quasinorms is stated in the book). However, no proof for arbitrary norms is provided in the book. The first published proof for arbitrary norms that I am aware of is in \cite{Duan} (Theorem 2.2.1).

\begin{theorem}
The $k$-NN classifier on the normed space $(\R^d, \norm{\cdot})$ with $k \to \infty$ as $n \to \infty$ and $\frac{k}{n} \to 0$ as $n \to \infty$ is universally consistent.
\end{theorem}
\begin{proof}
We show that $k$-NN satisfies the conditions of Stone's theorem:
\begin{enumerate}
\item The first condition holds because the unit sphere $S_1(\veczero, \norm{\cdot})$ can be covered by finitely many balls of radius $1/4$ each (by Lemma \ref{lemma:UnitSphereGeneratesFiniteConeCovering}) and hence Lemma \ref{lemma:StonesLemmaCore} applies.
\item The second condition holds by Lemma \ref{lemma:StonesTheoremSecondConditionNorm}.
\item The third condition holds since $k \to \infty$ as $n \to \infty$, so $\frac1k \to 0$ as $n \to \infty$.
\end{enumerate}
\end{proof}

\subsection*{Invariance of $k$-NN under Strictly Increasing Transformations}
We now make the observation that we can apply any strictly increasing transformation to our distance function in $k$-NN and $k$-NN will generate the same predictions for each point.
\begin{lemma}
\label{lemma:KnnStrictlyIncreasingTransformation}
Let $X$ be the query and $(X_1, Y_1), (X_2, Y_2), \dots, (X_n, Y_n)$ be the sample. We define $Y$ to be the label predicted for $X$ by $k$-NN with the function $d$ as the distance. If we let $h$ be a strictly increasing function and $Y'$ to be the label predicted for $X$ by $k$-NN with the function $h \circ d$ (the composition of the functions $d$ followed by $h$) as the distance, then $Y' = Y$.
\end{lemma}
\begin{proof}
Suppose a point $X_i$ is a $k$-nearest neighbour of $X$ with the distance function $d$, so there are fewer than $k$ sample points closer to $X$ than $X_i$ under $d$. If $X_j$ is a point such that $d(X, X_j) > d(X, X_i)$, then $h \circ d(X, X_j) > h \circ d(X, X_i)$, so $X_j$ remains further away than $X_i$. If $d(X, X_j) = d(X, X_i)$, then $h \circ d(X, X_j) = h \circ d(X, X_i)$, so the distance tie remains, which is broken by comparing the tiebreaking variables, which remain the same, so $X_i$ remains as a $k$-nearest neighbour. This means if a point $X_i$ has fewer than $k$ points closer than it to $X$ under $d$, this remains true under $h \circ d$, hence the $k$-nearest neighbours remain the same under $h \circ d$.
\end{proof}

This result allows us to apply a strictly increasing function to the distance kernel used for $k$-NN (that is, to find the distance between points $\vecx, \vecy$ for $k$-NN, we compute $f \circ d(x, y)$) and keep the same results with $k$-NN. This can be useful in reducing the computation time required for classification. For instance, when using the Euclidean distance for $k$-NN, we instead compute the Euclidean distance squared, which is $\sum_{i=1}^d x_i^2$, instead of the square root of this, and we obtain the same results (this eliminates the need for us to calculate the square root, which reduces our computation time required for $k$-NN). This also eliminates the need for us to include scaling coefficients or shift factors in our distance function in some cases.

\chapter{$k$-NN with a Sequence of Random Norms}

Suppose in the $k$-NN learning rule, we have a sequence of random norms from some family of norms, instead of a single norm. We show that under certain conditions, the resultant learning rule is universally consistent. A result of a form similar to ours (without the independence assumptions we make for the sequence of random norms) is found in \cite{pbook}, unfortunately, as we explain below the proof in the book is incomplete. We do not know if the result is correct after we remove the independence assumption.

\section{Families of Norms}

We now define a partial ordering $\preceq$ on the family of all norms $\F$ on a vector space $V$. For two norms ${\lVert \cdot \lVert}_A, {\lVert \cdot \lVert}_B \in \F$,
\begin{equation}
\label{eq:NormOrdering}
{\lVert \cdot \lVert}_B \preceq {\lVert \cdot \lVert}_A \mbox{ if and only if } \forall \vecv \in V \; {\lVert \vecv \lVert}_B \leq {\lVert \vecv \lVert}_A .
\end{equation}

\begin{lemma}
The relation $\preceq$ in \refeq{eq:NormOrdering} is a partial ordering on the family of norms $\F$ on a vector space $V$.
\end{lemma}
\begin{proof}
This is easily seen by verifying the conditions of a partial order.
\end{proof}

\begin{lemma}
\label{lemma:LpNormIsSandwiched}
If $p, q \in (0, \infty) \cup \{ \infty \}$ with $p \geq q$, then for all $\vecv \in \R^d$,
\begin{equation}
{\lVert \cdot \lVert}_q \succeq {\lVert \cdot \lVert}_p .
\end{equation}
\end{lemma}

\begin{lemma}
\label{lemma:BallsDominatedNorm}
If ${\lVert \cdot \lVert}_A$ and ${\lVert \cdot \lVert}_B$ are norms on $\R^d$ where ${\lVert \cdot \lVert}_B \preceq {\lVert \cdot \lVert}_A$, then for any point $\mathbf{x} \in \R^d$ and radius $r > 0$, the open balls and closed balls in norm ${\lVert \cdot \lVert}_A$ are smaller than those in norm ${\lVert \cdot \lVert}_B$:
\begin{equation}
B_r(\mathbf{x}, {\lVert \cdot \lVert}_A) \subseteq B_r(\mathbf{x}, {\lVert \cdot \lVert}_B)
\end{equation}
\begin{equation}
B_r^-(\mathbf{x}, {\lVert \cdot \lVert}_A) \subseteq B_r^-(\mathbf{x}, {\lVert \cdot \lVert}_B)
\end{equation}
\end{lemma}

\begin{lemma}
\label{lemma:SetBoundedEquivalenceOfNorms}
Let $\norm{\cdot}_A, \norm{\cdot}_B$ be two norms on $\R^n$ and $\epsilon > 0$. If a set $V \subseteq \R^d$ can be covered by finitely many open $\epsilon$-balls in the norm $\norm{\cdot}_A$, then it can be covered by finitely many open $\epsilon$-balls in the norm $\norm{\cdot}_B$.
\end{lemma}
\begin{proof}
By the equivalence of norms on $\R^d$, there exists a constant $C \geq 1$ such that
\begin{equation*}
\frac1C \norm{\vecv}_A \leq \norm{\vecv}_B \leq C \norm{\vecv}_A .
\end{equation*}
Any subset $V \subseteq \R^d$ that can be covered by finitely many $\epsilon$-balls in the $\norm{\cdot}_A$ norm is bounded in that norm, and since this is a subset of $\R^d$, is totally bounded (precompact). This means that there exists a finite set of points $S$ such that the balls of radius $\epsilon / C$ in the $\norm{\cdot}_A$ norm around these points cover $V$. If we let $\vecx \in V$ and $\vecy \in S$ such that $\norm{\vecx - \vecy} < \epsilon / C$. We then find that
\begin{equation*}
\norm{\vecx - \vecy}_B \leq C \norm{\vecx - \vecy}_A < C \frac{\epsilon}{C} = \epsilon .
\end{equation*}
\end{proof}

\section{Consistency of $k$-NN with a Family of Norms}

In this section, we define $\NormsFamily$ to be a family of norms on $\R^d$, we define the conditions $\NormsFamily$ must satisfy in each lemma and theorem. As usual, we assume that $(X_1, Y_1), (X_2, Y_2), \dots, (X_n, Y_n)$ is an iid labelled sample and let $(X, Y)$ be the query and its response, which is independent from the sample and follows the same distribution. We let $V$ be an arbitrary random variable independent from both the query and the sample. It is possible to assume that the tiebreaking variables $U_1, U_2, \dots, U_n$ are contained in $V$ (even if we require $V$ to be a classical real-values random variable, we can simply use a Borel isomorphism from $\R^n$ to $\R$ to combine the $n$ random variables into a single random variable that is independent of the labelled sample and the query).

\begin{lemma}
\label{lemma:UnitSphereCoveredByBallsInSandwichedNorm}
Let $\NormsFamily$ be a family of norms on $\R^d$ such that there exist two norms ${\lVert \cdot \lVert}_U$ and ${\lVert \cdot \lVert}_L$, such that for all $\rho \in \NormsFamily$, ${\lVert \cdot \lVert}_L \preceq \rho \preceq {\lVert \cdot \lVert}_U$. There exists a finite number $c$ such that  for any norm $\rho \in \NormsFamily$ the unit sphere $S_1(\veczero, \rho)$ can be covered by $c$ open balls of radius $1/4$.
\end{lemma}
\begin{proof}
Let $S_1(\mathbf{0}, \rho)$ be the unit sphere in norm $\rho$. It is clear that the unit sphere is a subset of the closed unit ball, $S_1(\mathbf{0}, \rho) \subseteq B_1^-(\mathbf{0}, \rho)$, and by Lemma \ref{lemma:BallsDominatedNorm} it follows that $B_1^-(\mathbf{0}, \rho) \subseteq B_1^{-}(\mathbf{0}, \norm{\cdot}_L)$, so that $S_1(\mathbf{0}, \rho) \subseteq B_1(\mathbf{0}, \norm{\cdot}_L)$.

We have that $B_1^{-}(\mathbf{0}, {\norm{\cdot}_L})$ is compact in the $\norm{\cdot}_L$ norm, so it is bounded, by Lemma \ref{lemma:SetBoundedEquivalenceOfNorms} it is also bounded in the $\norm{\cdot}_U$ norm, hence there is a finite subcover of $c$ open balls of $1/4$ radius in the ${\norm{\cdot}}_U$ norm, that is there exists a set of points $\mathbf{x_1}, \dots, \mathbf{x_c}$ such that
\begin{align*}
B_1^{-}(\mathbf{0}, {\norm{\cdot}}_L) \subseteq \bigcup_{i=1}^c B_{1/4}^{-} (\mathbf{x_i}, {\norm{\cdot}}_U) .
\end{align*}

By Lemma \ref{lemma:BallsDominatedNorm}, $B_{1/4} (\mathbf{x}_i, \norm{\cdot}_U) \subseteq B_{1/4} (\mathbf{x}_i, \rho)$, which means that
\begin{equation}
B_1^{-}(\mathbf{0}, \norm{\cdot}_L) \subseteq \bigcup_{i=1}^c B_{1/4} (\mathbf{x}_i, \norm{\cdot}_U) \subseteq \bigcup_{i=1}^c B_{1/4} (\mathbf{x}_i, \rho) .
\end{equation}

Since $S_1(\mathbf{0}, \rho) \subseteq B_1^{-}(\mathbf{0}, \norm{\cdot}_L)$, it follows that the set of open balls of radius $1/4$ around $\mathbf{x}_1, \dots, \mathbf{x}_c$ (in any of the norms in $\NormsFamily$) covers $S_1(\mathbf{0}, \rho)$.
\end{proof}

\begin{lemma}
\label{lemma:CoveringOfBallsImpliesFiniteExpectation}
Suppose $\NormsFamily$ is a family of norms that is bounded above by some norm ${\lVert \cdot \lVert}_U$ and below by another norm ${\lVert \cdot \lVert}_L$. We let $W_{ni}$ be the weight function for $k$-NN with the norm $\rho_n \in \NormsFamily$ for each $n$, with $\rho_n$ being independent of the sample and the query. For every nonnegative measurable function $f: \R^d \to \R$,
\begin{equation}
\E \left[ \sum_{i=1}^n W_{ni} (X) f(X_i) \right] \leq c \E \left[ f(X) \right] .
\end{equation}
\end{lemma}
\begin{proof}
By Lemma \ref{lemma:UnitSphereCoveredByBallsInSandwichedNorm}, we have that there exists a constant $c$ such that there are points $\vecx_1, \vecx_2, \dots, \vecx_c$ so that for any norm $\rho \in \NormsFamily$, the unit sphere $S_1(\veczero, \rho)$ is covered by $B_{1/4}(\vecx_1, \rho), B_{1/4}(\vecx_2, \rho), \dots, B_{1/4}(\vecx_c, \rho)$. By Lemma \ref{lemma:UnitSphereGeneratesFiniteConeCovering}, there exists a corresponding set of cones $S_1, S_2, \dots, S_c$ that cover $\R^d$, such that within any cone $S_q$, if $\vecx, \vecy \in S_q$ with $0 < \rho(\vecx) \leq \rho(\vecy)$, then $\rho(\vecy - \vecx) < \rho(\vecy)$ for any norm $\rho \in \NormsFamily$. A point is marked in a given norm if it is one of the $k$-nearest neighbours of $X$ among points in the sample that are in a given cone, that is, for each cone $S_q$, we mark the $k$ points among $X_1, \dots, X_n$ in the cone $S_i$ that are closest to $X$ (if there are fewer than $k$ points in a given cone, we mark all of them, and we break distance ties by comparing independent uniform random variables $U_1, U_2, \dots, U_n$ as discussed previously). It follows by and the argument in Lemma \ref{lemma:StonesLemmaCore} that for any fixed point and sample in $\R^d$ and norm $\rho \in \NormsFamily$, the set of $k$-nearest neighbours of a point in a sample is a subset of the set of points that are marked and at most $c k$ points are marked. We have that $\rho_n$ is always a norm in $\NormsFamily$ and is independent of the sample and the query. We can expand $W_{ni}(X)$ according to our definition and exchange $X$ and $X_i$ in the expectation (since they are iid and are independent of everything else) to find that
\begin{align*}
&\hphantom{{}={}}\E \left[ \sum_{i=1}^n W_{ni} (X) f(X_i) \right]\\
&= \E\left[\sum_{i=1}^n \frac1k \Indicator_{\left\{X_i \text{ is } \rho_n \text{ } k\text{-NN of } X \text{ among } X_1, \dots, X_{i-1}, X_i, X_{i+1}, \dots, X_n \right\}} f(X_i) \right] \\
&= \frac1k \sum_{i=1}^n \E\left[\Indicator_{\left\{X_i \text{ is } \rho_n \text{ } k\text{-NN of } X \text{ among } X_1, \dots, X_{i-1}, X_i, X_{i+1}, \dots, X_n \right\}} f(X_i) \right] \\
&= \frac1k \E\left[\sum_{i=1}^n \Indicator_{\left\{X \text{ is } \rho_n \text{ } k\text{-NN of } X_i \text{ among } X_1, \dots, X_{i-1}, X, X_{i+1}, \dots, X_n \right\}} f(X) \right] \\
&\leq \frac1k \E\left[\sum_{i=1}^n \Indicator_{\left\{X \text{ is marked in } \rho_n \text{ norm}\right\}} f(X) \right] \\
&\leq \frac1k c k \E\left[f(X)\right] \\
&= c \E\left[f(X)\right] .
\end{align*}
\end{proof}

\begin{lemma}
\label{lemma:KnnDistanceBoundedFamilyNormInequality}
Let $\NormsFamily$ be a family of norms in $\R^d$ such that for some norm $\norm{\cdot}$ and constant $C \geq 1$, $\forall \rho \in \NormsFamily$, $\frac1C \norm{\cdot} \preceq \rho \preceq C \norm{\cdot}$. Given a random point $X$, let $X_{(1, \norm{\cdot})}, \dots, X_{(n, \norm{\cdot})}$ be the points in increasing distance from $X$ with respect to the norm $\norm{\cdot}$. For any $\rho \in \NormsFamily$, if $W_{ni}$ are the weights in $k$-NN with $\rho$ as the norm, then for any $a > 0$,
\begin{equation}
\E \left[ \sum_{i=1}^n W_{ni} (X) \mathds{1}_{\{\norm{X_i - X} > a\}} \right] \leq \Prob\left( \norm{X_{\left(k, \norm{\cdot}\right)} - X} > \frac{a}{C^2}\right) .
\end{equation}
\end{lemma}
\begin{proof}
We first notice that the function $\sum_{i=1}^n W_{ni}(X)$ in the expectation is bounded above by one, hence the expectation is bounded above by the probability that the inner expression is nonzero,
\begin{equation}
\label{eq:KnnDistanceBoundedFamilyNormInequality1}
\E\left[\sum_{i=1}^n W_{ni}(X) \Indicator_{\{\norm{X_i - X} > a\}}\right] \leq \Prob\left( \sum_{i=1}^n W_{ni} (X) \mathds{1}_{\{\norm{X_i - X} > a\}} \neq 0 \right) .
\end{equation}

Suppose that $\norm{X_{\left(k, \norm{\cdot}\right)} - X} \leq a / C^2$. Since $\rho(\vecx) \leq \norm{\vecx}$ for all $\vecx \in \R^d$ and for any $i \leq k$, $\norm{X_{\left(i, \norm{\cdot}\right)} - X} \leq \norm{X_{\left(k, \norm{\cdot}\right)} - X}$,
\begin{align*}
\rho\left(X_{\left(i, \norm{\cdot}\right)} - X\right) &\leq C \norm{X_{\left(i, \norm{\cdot}\right)} - X} \\
&\leq C \norm{X_{\left(k, \norm{\cdot}\right)} - X} \\
&\leq C \frac{a}{C^2} \\
&= \frac{a}{C} .
\end{align*}

This means there exist at least $k$ points $X_{\left(1, \norm{\cdot}\right)}, X_{\left(2, \norm{\cdot}\right)}, \dots, X_{\left(k, \norm{\cdot}\right)}$ such that $\rho\left(X_{\left(i, \norm{\cdot}\right)} - X\right) \leq a / C$. For any point $X_j$ such that $\norm{X_j - X} > a$, we have that $a < \norm{X_j - X} \leq C \rho(X_j - X)$, and hence $\rho(X_j - X) > a / C$. This means that $X_j$ cannot be a $k$-nearest neighbour of $X$ in the $\rho$ distance, as there are at least $k$ points in the sample whose $\rho$ distance to $X$ is less than or equal to $a / C$. It follows that, if $\norm{X_{\left(k, \norm{\cdot}\right)} - X} \leq a / C^2$, then the interior of the expectation $\sum_{i=1}^n W_{ni} (X) \mathds{1}_{\{\norm{X_i - X} > a\}}$ (from equation  \eqref{eq:KnnDistanceBoundedFamilyNormInequality1}) is zero, as no point $X_i$ can be a $k$-nearest neighbour of $X$ in the $\rho$ distance and satisfy $\norm{X_i - X} > C^2 a$ simultaneously.

Hence we find that if the term inside the expectation in equation \eqref{eq:KnnDistanceBoundedFamilyNormInequality1} is nonzero, then $\norm{X_{\left(k, \norm{\cdot}\right)} - X} > a / C^2$ must hold. We conclude
\begin{align*}
\E\left[\sum_{i=1}^n W_{ni}(X) \Indicator_{\{\norm{X_i - X} > a\}}\right] &\leq \Prob\left( \sum_{i=1}^n W_{ni} (X) \mathds{1}_{\{\norm{X_i - X} > a\}} \neq 0 \right) \\
&\leq \Prob\left( \norm{X_{\left(k, \norm{\cdot}\right)} - X} > \frac{a}{C^2} \right) .
\end{align*}
\end{proof}

\begin{corollary}
\label{corollary:KnnNormsSecondCondition}
Let ${\left(\rho_n\right)}_{n=1}^\infty$ be a sequence of random norms independent of the sample $X_1, X_2, \dots, X_n$ and query $X$, and let $W_{ni}$ be the weights in the $k$-NN classifier with $\rho_n$ as the norm. For any $a > 0$, we have that
\begin{equation}
\E\left[ W_{ni}(X) \Indicator_{\{ \norm{X_i - X \} > a}\}} \right] \to 0 \text{ as } n \to \infty .
\end{equation}
\end{corollary}
\begin{proof}
We apply the law of total expectation (conditioning on the norm $\rho$), we then use the fact that the norm $\rho$ is independent of $X, X_1, X_2, \dots, X_n$ and apply Lemma \ref{lemma:KnnDistanceBoundedFamilyNormInequality} to find
\begin{align*}
&\E\left[ W_{ni}(X) \Indicator_{\{ \norm{X_i - X \} > a}\}} \right] \\
&= \E \left[ \E\left[ W_{ni}(X) \Indicator_{\{ \norm{X_i - X } > a\}} \middle| \rho \right] \right] \\
&\leq \E \left[ \Prob\left( \norm{X_i - X} > \frac{a}{C^2} \middle| \rho \right) \right] \\
&= \Prob\left( \norm{X_i - X} > \frac{a}{C^2}\right) .
\end{align*}

We then notice that the last term goes to zero by Lemma \ref{lemma:ProbKPointIsFarGoesToZero}.
\end{proof}

We are now able to prove our result that $k$-NN with a sequence of random norms (chosen independently of the sample and query) from a family of norms $\NormsFamily$ satisfying certain boundedness condition is universally consistent. An example of a family $\NormsFamily$ that satisfies our conditions is the family of all $\ell^p$ norms (with $1 \leq p \leq \infty$). For each $n \geq 1$, $\rho_n$ is a random norm from $\NormsFamily$, with $n$ being the sample size. This allows us to pick the random norm from $\NormsFamily$ differently as the sample size changes (as long as we keep independence from the sample and query).
\begin{theorem}
\label{theorem:KnnWithSandwichedNormIsUniversallyConsistent}
Let $\NormsFamily$ be a family of norms on $\R^d$ such that there exist norms $\norm{\cdot}_L, \norm{\cdot}_U$ where $\forall \rho \in \NormsFamily$ $\norm{\cdot}_L \preceq \rho \preceq \norm{\cdot}_U$. For any sequence of random norms ${\left( \rho_n \right)}_{n=1}^\infty$ in $\NormsFamily$ that are independent of the query and the sample, $k$-NN with this sequence of norms is universally consistent.
\end{theorem}
\begin{proof}
We verify that the $k$-NN learning rule with the norm chosen from $\NormsFamily$ by the function $\rho_n$ at each step satisfies the conditions for Stone's Theorem:
\begin{enumerate}
\item By Lemma \ref{lemma:CoveringOfBallsImpliesFiniteExpectation}, we have that for every nonnegative measurable function $f: \R^d \to \R$,
\begin{align*}
\E \left[ \sum_{i=1}^n W_{ni} (X) f(X_i) \right] \leq c \E \left[ f(X) \right] .
\end{align*}

\item The second condition is satisfied as shown by Corollary \ref{corollary:KnnNormsSecondCondition}.

\item This follows directly from the fact that in the $k$-NN classifier, as $n \to \infty$, $k_n \to \infty$, so that $1 / k_n \to 0$, and hence
\begin{align*}
\E \left[ \max_{1 \leq i \leq n} W_{ni} (X) \right] = \E \left[ 1/k_n \right] = 1/k_n \to 0 \text{ as } n \to \infty .
\end{align*}
\end{enumerate}

It follows from Stone's theorem (Theorem \ref{theorem:StonesTheorem}) that $k$-NN with a sequence of random norms in $\NormsFamily$ (with the random norms being independent of the sample and query) is universally consistent.
\end{proof}

With this result, we can take any sequence of random norms, chosen independently of the sample and the query, from a family of norms $\NormsFamily$ that satisfies certain conditions, and the $k$-NN learning rule with the resulting sequence of norms is universally consistent. We now provide some examples of universally consistent learning rules based on this theorem.

\begin{corollary}
Let $\NormsFamily$ be the family of $\ell^p$-norms on $\R^d$, with $1 \leq p \leq \infty$. Then the $k$-NN learning rule, with the norm $\rho_n \in \NormsFamily$ chosen at each step independently of the sample and the query, is universally consistent.
\end{corollary}
\begin{proof}
By Lemma \ref{lemma:LpNormIsSandwiched}, for any $\ell^p$ norm $\rho$, $\norm{\cdot}_\infty \preceq \rho \preceq \norm{\cdot}_1$, hence by Theorem \ref{theorem:KnnWithSandwichedNormIsUniversallyConsistent}, $k$-NN with any sequence of random $\ell^p$ norms is universally consistent (with $p$ being independent of the sample and the query).
\end{proof}

For an application of this result, suppose we have a labelled sample $(X_1, Y_1)$, $(X_2, Y_2), \dots, (X_{2n}, Y_{2n})$ of size $2n$. We can split this sample into two samples $(X_1, Y_1)$, $(X_2, Y_2), \dots, (X_n, Y_n)$ and $(X_{n+1}, Y_{n+1}), (X_{n+2}, Y_{n+2}), \dots, (X_{2n}, Y_{2n})$, which we see are independent of each other. We can then use one of these samples to find a norm (so we start with a family of norms satisfying the conditions of Theorem \ref{theorem:KnnWithSandwichedNormIsUniversallyConsistent} and optimize over norms in the family for this sample), and then use the norm we found as the norm for $k$-NN with the other sample to classify the query.

\section{Matrix-based Norms}

We now investigate the universal consistency of $k$-NN when we select a matrix from a family of matrices (based on the dataset), multiply by the matrix, and then apply an existing norm from a family of norms.

The \emph{general linear group} $GL(n)$ is the group of all invertible $n \times n$ matrices. We will now show that multiplication by matrices in $GL(n)$ can be used to create new norms.

\begin{lemma}
\label{lemma:NormInvertibleMatrix}
Let $A$ be an $n \times n$ invertible matrix (equivalently, $A \in GL(n)$). Then for any norm $\lVert \cdot \lVert$ on $\R^d$, $\rho(\vecv) = \lVert A \vecv \lVert$ is also a norm.
\end{lemma}
\begin{proof}
We show that $\rho(\vecv) = \lVert A \vecv \lVert$ for $\vecv \in \R^d$ satisfies the conditions for a norm in $\R^d$:
\begin{enumerate}[(i)]
\item The $\rho$-norm of the zero vector is zero,
\begin{align*}
\rho(\mathbf{0}) = \lVert A \veczero \lVert = \lVert \veczero \lVert = 0 .
\end{align*}
For any nonzero vector $\vecv$, since $A$ is invertible $A\vecv$ is nonzero, so the $\lVert A \vecv \lVert$ norm of $A\vecv$ is strictly positive, so the $\rho$-norm of $\vecv$ is strictly positive,
\begin{align*}
\rho(\mathbf{v}) = \lVert A \vecv \lVert > 0 \mbox{ since } A \vecv \neq \veczero .
\end{align*}
\item If we multiply a vector $\vecv \in \R^d$ by a constant $\lambda \in \R$, we see that
\begin{align*}
\rho(\lambda \vecv) &= \lVert A \lambda \vecv \lVert \\
&= \lVert \lambda A \vecv \lVert \\
&= \left| \lambda \right| \lVert A \vecv \lVert \\
&= \left| \lambda \right| \rho(\vecv) .
\end{align*}
\item The triangle inequality holds for the $l$-norm, for any $\vecu, \vecv \in \R^d$ we see that
\begin{align*}
\rho(\vecu + \vecv) &= \lVert A (\vecu + \vecv) \lVert \\
&= \lVert A \vecu + A \vecv \lVert \\
&\leq \lVert A \vecu \lVert + \lVert A \vecv \lVert \\
&= \rho(\vecu) + \rho(\vecv) .
\end{align*}
\end{enumerate}
\end{proof}

\begin{lemma}
\label{lemma:MatrixBasedNormsSingleNorm}
Let $\MatrixFamily$ be a family of invertible $d$ by $d$ matrices and $\NormsFamily$ be a family of norms on $\R^d$, such that there exists a constant $B \geq 1$, such that for all $\vecv \in \R^d$ and $\rho \in \NormsFamily$, $\frac1B \norm{\vecv} \leq \rho(\vecv) \leq B \norm{\vecv}$. If there exists a constant $C \geq 1$ such that $\frac1C \norm{\vecv} \leq \norm{A \vecv} \leq C \norm{\vecv}$ $\forall \vecv \in \R^d$ $\forall A \in M_{d,d}(\R)$, then the family of norms $\NormsFamily^* = \{\rho^* : \rho^*(\vecv) = \rho(A \vecv) \;\forall A \in M_{d,d}(\R) \}$ satisfies the property that there exists some constant $A \geq 1$ such that for all $\rho \in \NormsFamily^*$ and $\vecv \in \R^d$, $\frac1A \norm{\vecv} \leq \norm{\vecv} \leq A \norm{\vecv}$.
\end{lemma}
\begin{proof}
By assumption, there exists a constant $B \geq 1$ such that for all $\vecv \in \R^d$ and $\rho \in \NormsFamily$,
\begin{equation*}
\frac1B \norm{\vecv} \leq \rho(\vecv) \leq B \norm{\vecv} .
\end{equation*}

Suppose we have a norm $\rho^* \in \NormsFamily^*$, which corresponds to the matrix $A \in M_{d, d}(\R)$ and norm $\rho \in \NormsFamily$. We find that
\begin{equation*}
\rho^*(\vecv) = \rho(A \vecv) \leq B \norm{A \vecv} \leq B C \norm{\vecv}
\end{equation*}
and that
\begin{equation*}
\rho^*(\vecv) = \rho(A \vecv) \geq \frac1B \norm{A \vecv} \geq \frac{1}{B C} \norm{\vecv} .
\end{equation*}
Hence we find that
\begin{equation*}
\frac{1}{B C} \norm{\vecv} \leq \rho^*(\vecv) \leq B C \norm{\vecv} .
\end{equation*}
We therefore find that for $A = BC$, for all $\rho \in \NormsFamily^*$ and $\vecv \in \R^d$, $\frac1A \norm{\vecv} \leq \norm{\vecv} \leq A \norm{\vecv}$.
\end{proof}

\begin{corollary}
\label{corollary:MatrixBasedNormsBoundedConsistent}
Suppose we apply the $k$-NN classifier with a sequence of random norms (chosen independently of the sample and the query) from the family of norms $\NormsFamily^*$ of the form in Lemma \ref{lemma:MatrixBasedNormsSingleNorm}. Then the resulting classifier is universally consistent.
\end{corollary}
\begin{proof}
We have that $\NormsFamily^*$ is bounded both below and above by Lemma \ref{lemma:MatrixBasedNormsSingleNorm}, so it satisfies the conditions in Theorem \ref{theorem:KnnWithSandwichedNormIsUniversallyConsistent}, and so $k$-NN with $\NormsFamily^*$ is universally consistent.
\end{proof}

One important family of matrices that we will use is the group of \emph{orthogonal} $d$ by $d$ matrices $O(d)$ (and the subgroup of \emph{special orthogonal} matrices $SO(d)$).
\begin{definition}
\label{def:OrthogonalMatrix}
An $d$ by $d$ matrix $Q$ is an \emph{orthogonal} matrix if multiplication of the matrix by its transpose (in either order) results in the identity matrix:
\begin{equation}
\label{eq:OrthogonalMatrix}
Q^\intercal Q = Q Q^\intercal = I_m
\end{equation}
We say that $Q$ is a \emph{special orthogonal} matrix if it satisfies the additional criterion that its determinant is one:
\begin{equation}
\label{eq:SpecialOrthogonalMatrixDeterminantCondition}
\det(Q) = 1
\end{equation}
\end{definition}

An important result is that multiplication by orthogonal matrices preserves the Euclidean norm of a vector.
\begin{theorem}
\label{theorem:MatrixBasedNormsOrthogonalEuclideanEquality}
For any matrix $Q \in O(d)$ and vector $\vecx \in \R^d$, the Euclidean norm of $\vecx$ is equal to the Euclidean norm of $Q \vecx$, $\norm{\vecx}_2 = \norm{Q \vecx}_2$.
\end{theorem}
\begin{proof}
This is a standard result, it is proven in Theorem A.1.3 (in Appendix 1) in \cite{Pressley}.
\end{proof}

With this result, we are now able show that $k$-NN with a sequence of random orthogonal matrices $O(d)$ (independent of the sample and the query) is universally consistent.
\begin{corollary}
\label{corollary:MatrixBasedNormsOrthogonalKnnConsistent}
Suppose we have the family of norms consisting of multiplying the input by a random orthogonal matrix followed by applying an $\ell^p$-norm. Then $k$-NN with a sequence of random norms (independent of the sample and the query) from this family is universally consistent.
\end{corollary}
\begin{proof}
We see that by Lemma \ref{lemma:LpNormIsSandwiched} and Theorem \ref{theorem:MatrixBasedNormsOrthogonalEuclideanEquality} that the conditions of Lemma \ref{lemma:MatrixBasedNormsSingleNorm} are satisfied for this family of norms. Hence by Corollary \ref{corollary:MatrixBasedNormsBoundedConsistent}, $k$-NN with this family of norms is universally consistent.
\end{proof}

\begin{lemma}
\label{lemma:MatrixBoundedFamilyDiagonal}
Given $0 < \beta \leq \alpha < \infty$, let $\MatrixFamily_{\beta, \alpha}$ be the family of all diagonal matrices such that for each entry $a_{i,i}$ on the diagonal (with $1 \leq i \leq d$), $\beta \leq \left| a_{i, i} \right| \leq \alpha$. Then for any vector $\vecx \in \R^d$ and matrix $D \in \MatrixFamily_{\beta, \alpha}$, $\beta {\norm{\vecx}}_p \leq {\norm{D \vecx}}_p \leq \alpha {\norm{\vecx}}_p$, for any $\ell^p$ norm.
\end{lemma}
\begin{proof}
If $p$ is finite, we see that, for any diagonal matrix $D \in \MatrixFamily_{\beta, \alpha}$ with diagonal entries in $[\beta, \alpha]$ and vector $\vecx = (x_1, x_2, \dots, x_d) \in \R^d$,
\begin{align*}
\norm{A \vecx}_p &= \sqrt[p]{\sum_{i=1}^d {(a_i x_i)}^p} \\
&\leq \sqrt[p]{\sum_{i=1}^d {(\alpha x_i)}^p} \\
&= \alpha \sqrt[p]{\sum_{i=1}^d {x_i}^p} \\
&= \alpha \norm{\vecx}_p
\end{align*}
and that
\begin{align*}
\norm{A \vecx}_p &= \sqrt[p]{\sum_{i=1}^d {(a_i x_i)}^p} \\
&\geq \sqrt[p]{\sum_{i=1}^d {(\beta x_i)}^p} \\
&= \beta \sqrt[p]{\sum_{i=1}^d {x_i}^p} \\
&= \beta \norm{\vecx}_p .
\end{align*}

If $p$ is infinite, we have (with $D$, $\vecx$ as above)
\begin{align*}
\norm{A \vecx}_\infty &= \max_{1 \leq i \leq d} \{ \left| a_i x_i \right| \} \\
&\leq \max_{1 \leq i \leq d} \{ \left| \alpha x_i \right| \} \\
&= \alpha \max_{1 \leq i \leq d} \{ \left| x_i \right| \} \\
&= \alpha \norm{\vecx}_\infty
\end{align*}
and that
\begin{align*}
\norm{A \vecx}_\infty &= \max_{1 \leq i \leq d} \{ \left| a_i x_i \right| \} \\
&\geq \max_{1 \leq i \leq d} \{ \left| \beta x_i \right| \} \\
&= \beta \max_{1 \leq i \leq d} \{ \left| x_i \right| \} \\
&= \beta \norm{\vecx}_\infty .
\end{align*}
\end{proof}

\begin{lemma}
\label{lemma:MatrixProductBoundedFamilies}
Let $\MatrixFamily_1$ and $\MatrixFamily_2$ be two families of invertible $d$ by $d$ matrices that both satisfy the following boundedness condition: there exists $b > 0$ such that for all $\vecv \in \R^d$ and $B \in \MatrixFamily_1$, $\frac1b \norm{\vecv} \leq \norm{B \vecv} \leq b \norm{\vecv}$ (and a corresponding constant $c > 0$ exists for $\MatrixFamily_2$). We define $\MatrixFamily$ to be the product of all pairs of matrices in $\MatrixFamily_1$ and $\MatrixFamily_2$, that is $A \in \MatrixFamily$ if and only if $A = B C$ with $B \in \MatrixFamily_1$ and $C \in \MatrixFamily_2$. Then $\MatrixFamily$ is a bounded family of invertible matrices.
\end{lemma}
\begin{proof}
We first notice that the product of invertible matrices is invertible. For any matrix $A \in \MatrixFamily$, we let $B \in \MatrixFamily_1$ and $C \in \MatrixFamily_2$ be such that $A = B C$, and $b, c \geq 1$ be constants such that $\frac1b \norm{\vecv} \leq \norm{B \vecv} \leq b \norm{\vecv}$ for all $\vecv \in \R^d$ and for all $B \in \MatrixFamily_1$, and similarly for $C$. We notice that for every vector $\vecv \in \R^d$,
\begin{align*}
\norm{A \vecv} &= \norm{B C \vecv} \\
&\geq \frac1b \norm{C \vecv} \\
&\geq \frac{1}{b c} \norm{\vecv}
\end{align*}
and that
\begin{align*}
\norm{A \vecv} &= \norm{B C \vecv} \\
&\leq b \norm{C \vecv} \\
&\leq b c \norm{\vecv} .
\end{align*}
\end{proof}

From this result we see that we can, for instance, first multiply by a diagonal matrix from a bounded family and then multiply by an orthogonal matrix (or in the other order), and retain universal consistency.

We would now like to find a general criterion for checking if a family of matrices is bounded both below and above. We can do this from the \emph{singular value decomposition} of a matrix.
\begin{theorem}
\label{theorem:SingularValueDecomposition}
Let $A$ be a real valued $d$ by $d$ matrix. There exist orthogonal $d$ by $d$ matrices $U$ and $V$ and an $d$ by $d$ diagonal matrix $\Sigma$ such that the diagonal entries of $\Sigma$ are the square roots of the eigenvalues of $A^\intercal A$ and of $A A^\intercal$ (they are called the \emph{singular values} of $A$) and
\begin{equation}
A = U \Sigma V^\intercal .
\end{equation}
We call this the \emph{singular value decomposition} of the matrix $A$.\footnote{A version of this result also holds for complex matrices and non-square matrices, for our purposes the result for square real-valued matrices suffices.}
\end{theorem}
\begin{proof}
This is a standard result, a proof can be found in \cite{AppliedLinearAlgebra} (Chapter 8, Theorem 8.19).
\end{proof}

\begin{lemma}
\label{lemma:DiagonalMatrixNormInequality}
For any vector $\vecv \in \R^d$ and $d$ by $d$ diagonal matrix $D$, if we let $a_1, a_2, \dots, a_d$ be the entries on the diagonal of $D$ and $\vecv = (v_1, v_2, \dots, v_d)$, then
\begin{equation}
\min_{1 \leq i \leq d} \left| a_i \right| \norm{\vecv}_2 \leq \norm{D \vecv}_2 \leq \max_{1 \leq i \leq d} \left| a_i \right| \norm{\vecv}_2 .
\end{equation}
\end{lemma}
\begin{proof}
We notice that for the second inequality,
\begin{align*}
\norm{D \vecv}_2 &= \sqrt{{(a_1 v_1)}^2 + {(a_2 v_2)}^2 + \dots + {(a_d v_d)}^2} \\
&= \sqrt{ a_1^2 v_1^2 + a_2^2 v_2^2 + \dots + a_d^2 v_d^2} \\
&\leq \sqrt{\left( \max_{1 \leq i \leq d} a_i^2 \right) \left(v_1^2 + v_2^2 + \dots + v_d^2\right)} \\
&= \left( \max_{1 \leq i \leq d} \left| a_i \right| \right) \sqrt{ \left(v_1^2 + v_2^2 + \dots + v_d^2\right)} \\
&= \left( \max_{1 \leq i \leq d} \left| a_i \right| \right) \norm{\vecv}_2 .
\end{align*}

Similarly, for the first inequality, we have
\begin{align*}
\norm{D \vecv}_2 &= \sqrt{{(a_1 v_1)}^2 + {(a_2 v_2)}^2 + \dots + {(a_d v_d)}^2} \\
&= \sqrt{ a_1^2 v_1^2 + a_2^2 v_2^2 + \dots + a_d^2 v_d^2} \\
&\geq \sqrt{\left( \min_{1 \leq i \leq d} a_i^2 \right) \left(v_1^2 + v_2^2 + \dots + v_d^2\right)} \\
&= \left( \min_{1 \leq i \leq d} \left| a_i \right| \right) \sqrt{ \left(v_1^2 + v_2^2 + \dots + v_d^2\right)} \\
&= \left( \min_{1 \leq i \leq d} \left| a_i \right| \right) \norm{\vecv}_2 .
\end{align*}
\end{proof}

\begin{theorem}
\label{theorem:MatrixBoundSingularValues}
Let $A$ be a real-valued $d$ by $d$ matrix and $A = U \Sigma V^\intercal$ be its singular value decomposition, with $\sigma_1, \sigma_2, \dots, \sigma_m$ being the singular values of $A$. We have that
\begin{equation}
\left(\min_{1 \leq i \leq d} \left| \sigma_i \right|\right) \norm{\vecv}_2 \leq \norm{A \vecv}_2 \leq \left(\max_{1 \leq i \leq d} \left| \sigma_i \right|\right) \norm{\vecv}_2 .
\end{equation}
\end{theorem}
\begin{proof}
For the upper bound inequality, we see that
\begin{align*}
\norm{A \vecv} &= \norm{U \Sigma V^\intercal \vecv}_2 \\
&= \norm{\Sigma V^\intercal \vecv}_2 \text{ since } U \text{ is orthogonal}\\
&\leq \left(\max_{1 \leq i \leq d} \left| \sigma_i \right|\right) \norm{V^\intercal \vecv}_2 \text{ by Lemma \ref{lemma:DiagonalMatrixNormInequality}} \\
&= \left(\max_{1 \leq i \leq d} \left| \sigma_i \right|\right) \norm{\vecv}_2 \text{ since } V^\intercal \text{ is orthogonal.}
\end{align*}

Similarly for the lower bound inequality, we see that
\begin{align*}
\norm{A \vecv} &= \norm{U \Sigma V^\intercal \vecv}_2 \\
&= \norm{\Sigma V^\intercal \vecv}_2 \text{ since } U \text{ is orthogonal}\\
&\geq \left(\min_{1 \leq i \leq d} \left| \sigma_i \right|\right) \norm{V^\intercal \vecv}_2 \text{ by Lemma \ref{lemma:DiagonalMatrixNormInequality}} \\
&= \left(\min_{1 \leq i \leq d} \left| \sigma_i \right|\right) \norm{\vecv}_2 \text{ since } V^\intercal \text{ is orthogonal.}
\end{align*}
\end{proof}

From this, we see that we can multiply the data by a matrix from a family of matrices whose singular values are bounded (both above and from below away from zero) and maintain universal consistency.

\begin{corollary}
\label{corollary:KnnConsistentWithMatricesBoundedSingularValue}
If $\NormsFamily$ is a bounded family of norms and $\MatrixFamily$ is a family of matrices whose singular values are bounded below away from zero and are bounded above by some finite value, then $k$-NN with a sequence of random norms (independent of the sample and the query) from the family of norms consisting of first multiplying by a matrix in $\MatrixFamily$ and then applying a norm in $\NormsFamily$ is universally consistent.
\end{corollary}
\begin{proof}
This follows from Theorem \ref{theorem:MatrixBoundSingularValues} and Corollary \ref{corollary:MatrixBasedNormsBoundedConsistent}.
\end{proof}

\section{Sequences of Norms that Depend on the Sample}

In Theorem \ref{theorem:KnnWithSandwichedNormIsUniversallyConsistent}, we have assumed that the sequence of norms is independent of the sample and the query. This is a strong assumption we would like to eliminate. Unfortunately, removing this assumption appears to be quite difficult.

In \cite{pbook}, Theorem 26.3, it is claimed that a result of a form similar to Theorem \ref{theorem:KnnWithSandwichedNormIsUniversallyConsistent} holds. The book claims that if we multiply the data by a matrix $A_n$ that is a function of the sample points $X_1, X_2, \dots, X_n$ and then apply the Euclidean norm, then $k$-NN with the resulting distance is universally consistent. The proof is performed by checking the three conditions of Stone's theorem. Unfortunately, the proof that such a classifier satisfies the first condition of Stone's theorem is incorrect. The book makes the following argument: for the first condition we need that the number of data points that can be among the $k$ nearest neighbours of $X$ is at most $k c_d$, where $c_d$ is a constant that depends on the dimension only; and that this is a deterministic property that can be proven in exactly the same manner as for the usual $k$-NN learning rule.

The problem with this argument is that if the norm is a function of the sample, when we do the exchange of random variables in the proof of Stone's lemma (Lemma \ref{lemma:StonesLemmaCore}), we obtain a different norm for each point, as shown below:
\begin{align*}
&\phantom{{}={}}\sum_{i=1}^n \E \left[ W_{ni}(X) f(X_i) \right] \\
&= \E \left[ \sum_{i=1}^n \frac1k \Indicator_{\{ X_i \text{ is a } k \text{-nearest neighbour of } X \text{ in the } \rho_n(X_1, X_2, \dots, X_n) \text{ norm among } X_1, X_2, \dots, X_n \}} f(X_i) \right] \\
&= \frac1k \E \left[ \sum_{i=1}^n \Indicator_{\left\{ \substack{X \text{ is a } k \text{-nearest neighbour of } X_i \text{ in the } \rho_n(X_1, \dots, X_{i-1}, X, X_{i+1}, \dots, X_n) \\ \text{norm } \text{among } X_1, \dots, X_{i-1}, X, X_{i+1}, \dots, X_n } \right\}} f(X) \right]
\end{align*}

We see that we have a different norm for each point, namely $\rho_n(X, X_2, X_3, \dots, X_n)$ for the point $X_1$, $\rho_n(X_1, X, X_3, \dots, X_n)$ for the point $X_2$, and so on. This means we must consider the set of points such that $X$ can be the $k$-nearest neighbour of them in any of the norms from the family, and not just each norm by itself. As we show in the following example, the geometrical argument used in Stone's lemma does not work if we require a bound on the number of points that can be considered a nearest neighbour for \emph{any} norm in a family (even if the family of norms is bounded, as described in Theorem \ref{theorem:KnnWithSandwichedNormIsUniversallyConsistent}) as opposed to a single norm.

\begin{figure}
\centering
\includegraphics[scale=3]{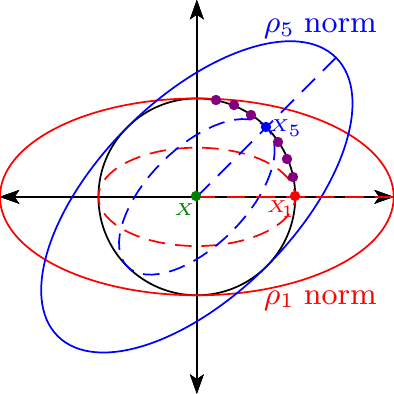}
\caption{We see that for each of the points $X_1, X_2, \dots, X_n$ (with $n = 8$ in this illustration), we have that $X_i$ is the nearest neighbour to the origin $X$ in the $\rho_i$ norm. The points $X_1, X_2, \dots, X_n$ are distributed along the circle as described by equation \eqref{eq:NormsCircleCounterexample} (with corresponding norms given by \eqref{eq:NormsCounterexampleFormula}).}
\label{fig:DependentNormsCounterexample}
\end{figure}

\begin{example}
\label{example:DependentNormsCounterexample}
Suppose the query $X$ is at the origin, and the data points $X_1, X_2, \dots, X_n$ are arranged on the upper-right part of the unit circle around $X$, as shown in Figure \ref{fig:DependentNormsCounterexample}. We can accomplish this by taking the coordinate of the point $X_i$ to be as follows:
\begin{equation}
\label{eq:NormsCircleCounterexample}
X_i = \left(\cos\left(\frac{(i - 1) \pi}{2 n}\right), \sin\left(\frac{(i - 1) \pi}{2 n}\right)\right),\; 1 \leq i \leq n
\end{equation}

We then define a norm $\rho_i$ for each $1 \leq i \leq n$ as follows: we first apply a rotation of angle $-(i - 1) \pi / (2 n)$ (effectively rotating the unit circle such that the point $X_i$ is now at $(1, 0)$), we then multiply by the matrix $\begin{bmatrix}
1 & 0 \\ 
0 & 2
\end{bmatrix} $ and apply the Euclidean norm. This norm is given by the formula
\begin{align}
\label{eq:NormsCounterexampleFormula}
\rho_i((x, y)) &= \norm{\begin{bmatrix}
1 & 0 \\ 
0 & 2
\end{bmatrix} \begin{bmatrix}
\cos\left(\frac{(i - 1) \pi}{2 n}\right) & \sin\left(\frac{(i - 1) \pi}{2 n}\right) \\
-\sin\left(\frac{(i - 1) \pi}{2 n}\right) & \cos\left(\frac{(i - 1) \pi}{2 n}\right)
\end{bmatrix} \begin{bmatrix}
x \\
y
\end{bmatrix}}_2 \\
&= \norm{\begin{bmatrix}
\cos\left(\frac{(i - 1) \pi}{2 n}\right) & \sin\left(\frac{(i - 1) \pi}{2 n}\right) \\
-2\sin\left(\frac{(i - 1) \pi}{2 n}\right) & 2\cos\left(\frac{(i - 1) \pi}{2 n}\right)
\end{bmatrix} \begin{bmatrix}
x \\
y
\end{bmatrix}}_2 \nonumber \\
&= \sqrt{\left( 3 \sin^2\left(\frac{(i - 1) \pi}{2 n}\right) + 1 \right) x + \left( 3 \cos^2\left(\frac{(i - 1) \pi}{2 n}\right) + 1 \right) y } . \nonumber
\end{align}

This norm $\rho_i$ is intuitively a norm that gives twice as much importance to the $y$-axis as to the $x$-axis, with the axes rotated by angle $(i - 1) \pi / (2 n)$ prior to applying the norm (without rotating the points). We easily see that $X_i$ is the nearest point to the origin $X$ among $X_1, X_2, \dots, X_n$ in the $\rho_i$ norm. This holds for each of the norms in $\rho_1, \rho_2, \dots, \rho_n$ and the corresponding data point, hence we see that there are $n$ points in the sample that can be the nearest neighbour of the query $X$ for some norm $\rho_1, \rho_2, \dots, \rho_n$ (with a different norm for each point). In Stone's lemma (Lemma \ref{lemma:StonesLemmaCore}), we need that the number of such points is at most $ck$, with $c$ being a fixed constant. Since there are $n$ such points, if we substitute this into the inequality (instead of $c k$) we obtain an upper bound of $\frac{n}{k} E\left[f(X)\right]$, which is not useful for us as $n / k \to \infty$ as $n \to \infty$ (since $k / n \to 0$ as $n \to \infty$). Hence we see that even though the argument in Stone's lemma works for each fixed norm among $\rho_1, \rho_2, \dots, \rho_n$, it does not work for the combined family of all such norms.

The family of norms of the form in equation \eqref{eq:NormsCounterexampleFormula} (containing all such norms for any $n \geq 1$ and $1 \leq i \leq n$) is bounded (in the sense of Theorem \ref{theorem:KnnWithSandwichedNormIsUniversallyConsistent}), since the rotation matrix is an orthogonal matrix and all entries of the fixed diagonal matrix are nonzero, so (by Theorem \ref{theorem:MatrixBasedNormsOrthogonalEuclideanEquality} and Lemmas \ref{lemma:MatrixBoundedFamilyDiagonal} and \ref{lemma:MatrixProductBoundedFamilies}) the family of norms consisting of first applying the rotation, then multiplying by the diagonal matrix and then applying the Euclidean norms is bounded both above and below (that is, for any norm $\rho$ in this family, there exists $C \geq 1$ such that $\frac1C \norm{\vecv}_2 \leq \rho(\vecv) < C \norm{\vecv}$). Therefore, this family of norms cannot be excluded by the boundedness conditions on the family of norms, since it satisfies these conditions of the theorem.
\end{example}

This example above shows that our proof for universal consistency will not work if we replace the independently chosen norm by a norm chosen as some function of the sample data points. To prove universal consistency with the norm as a function of the sample (and possibly the query), we will need to use additional or different techniques. A possible approach would be to show that the probability of a configuration of points such as the one described above occurs with a probability that decreases sufficiently fast as $n$ approaches infinity, for every probability measure on $\R^d$ (the deterministic geometric result would then be replaced by a probabilistic argument).

\section{Necessity of the Boundedness Conditions}

Our theorem requires that there exist norms $\norm{\cdot}_L$ and $\norm{\cdot}_U$ such that for any norm $\rho$ in our family of norms, $\norm{\cdot}_L \preceq \rho \preceq \norm{\cdot}_U$. We now see that both the conditions that the family of norms is bounded from above and from below are necessary.

Suppose we have the probability measure $\mu$ on $\R^2 \times \{0, 1\}$, such that for any $A \subseteq \R^2 \times \{0, 1\}$ (where $\lambda$ is the Lebesgue measure on $\R$):
\begin{align}
\label{eq:BadProbMeasureNormsSequence}
\mu(A) = \frac{1}{2} \lambda(\{ x : (x, 0) \times \{0\} \in A, x \in [0, 1] \}) + \frac{1}{2} \lambda(\{ x : (x, 1) \times \{1\} \in A, x \in [0, 1] \})
\end{align}

That is, there is a line segment $\{(x, 0) \;: x \in [0, 1]\}$ with uniform probability density $1/2$ at $y=0$ with label 0, and a line segment $\{(x, 1) \;: x \in [0, 1]\}$ with uniform probability density $1/2$ at $y=1$ with label 1. Let $X$ be a query and $(X_1, Y_1), \dots, (X_n, Y_n)$ be $n$ iid sample points. We define $X_{i, x}$ to be the $x$ coordinate of the point and $X_{i, y}$ to be the $y$ coordinate of the point $X_i$ (for the point $X$, we define $X_x$ to be the $x$ coordinate and $X_y$ to be the $y$ coordinate). We also define $X_{(i)}$ to be the $i^{\mbox{th}}$ point from $X_1, X_2, \dots, X_n$ in distance from $X$, in a given norm.

Suppose we have a sequence of norms ${(\rho_n)}_{n=1}^\infty$ of the form (with ${(a_n)}_{n=1}^\infty$, ${(b_n)}_{n=1}^\infty$ being two numeric sequences):
\begin{equation}
\label{eq:BadNormSequenceGeneralForm}
\rho_n(\vecv) = \norm{\begin{pmatrix}
a_n & 0 \\
0 & b_n
\end{pmatrix} \vecv}_\infty
\end{equation}

The distance between the query $X$ and a point $X_i$ is
\begin{equation}
\label{eq:BadNormSequenceDistance}
\rho_n(X, X_i) = \max\left\{a_n \left| X_{i, x} - X_x \right|, b_n \left| X_{i, y} - X_y \right| \right\} .
\end{equation}

Suppose that the distance between $X$ and $X_i$ is strictly greater than $b_n$. Since $b_n \left| X_{i, y} - X_y \right| \leq b_n$, it follows that $a_n \left| X_{i, x} - X_x \right|$ is the larger term, and so
\begin{equation}
\label{eq:BadNormSequenceDistanceFar}
\rho_n(X, X_i) = a_n \left| X_{i, x} - X_x \right| \text{ if } \rho_n(X, X_i) > b_n .
\end{equation}

\begin{figure}
\centering
\includegraphics[scale=1]{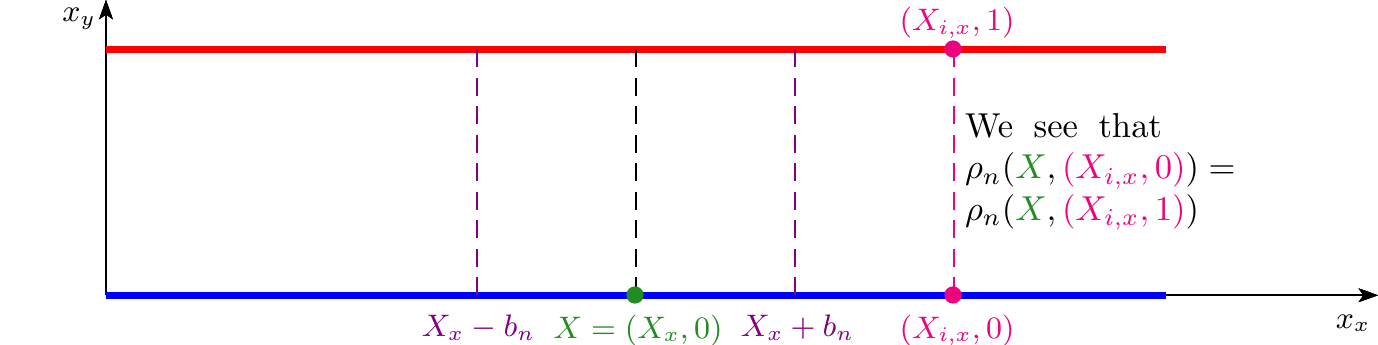}
\caption{If we have a point $X_i$ whose $x$-coordinate differs by more than $b_n$ from $X$, then the $\rho_n$-distance of $X_i$ from $X$ is $\rho_n(X, X_i) = a_n \left| X_{i, x} - X_x \right| $, and does not depend whether the point is on the upper or the lower line segment.}
\label{fig:BadNormsDistances}
\end{figure}

We notice that the strict inequality $\rho_n(X, X_i) > b_n$ holds if and only if $X_{i, x} \in [0, 1] \setminus \left[ X_x - b_n / a_n, X_x + b_n / a_n \right]$ (no matter what $Y_i = X_{i, y}$ is). This means we have equal length intervals of equal probability density on which the condition holds, and so we find that conditioning on this results in a probability of $1 / 2$ of the point being on either segment (and hence a $1 / 2$ probability of each label),
\begin{align}
\label{eq:BadNormsEqualProbSegments}
\Prob\left(Y_i = 0 \middle| \rho_n(X, X_i) > b_n \right) &= \Prob\left(Y_i = 1 \middle| \rho_n(X, X_i) > b_n \right) \\
 &= \frac{1}{2} . \nonumber
\end{align}

If all the points $X_1, X_2, \dots, X_n$ satisfy the property that $\rho_n(X, X_i) > b_n$ (or equivalently, the condition $\rho_n(X, X_{(1)}) > b_n$ holds), then by the formula \ref{eq:BadNormSequenceDistanceFar} we see that the distance of any of the points $X_1, X_2, \dots, X_n$ from $X$ does not depend on $Y_1, Y_2, \dots, Y_n$, and by equation \ref{eq:BadNormsEqualProbSegments} we have that the points are equally likely to be on either segment. We easily see that the order statistics $(X_{(1)}, Y_{(1)}), (X_{(1)}, Y_{(1)}), \dots, (X_{(n)}, Y_{(n)})$ therefore do not depend on which segments the points are on, as long as $\rho_n(X, X_i) > b_n$ holds for each point. From this we see that the order statistics of the points are conditionally independent of the $Y_i$, if $\rho_n(X, X_{(1)}) > b_n$ holds. Indeed, all the order statistics are equally likely to be on either line segment, as long as the $\rho_n(X, X_i) > b_n$ condition holds, that is for all $i \in \{1, 2, \dots, n\}$,
\begin{align}
\label{eq:BadNormsOrderStatisticsConditional}
&\phantom{{}={}}\Prob\left(Y_{(i)} = 0 \middle| \forall j \in \{1, 2, \dots, n\} \; \rho_n(X, X_j) > b_n \right) \\
&= \Prob\left(Y_{(i)} = 1 \middle| \forall j \in \{1, 2, \dots, n\} \; \rho_n(X, X_j) > b_n \right) \nonumber \\
&= \frac{1}{2} . \nonumber
\end{align}

In addition, since the $Y_i$ are iid and are conditionally independent of the distance of the $i^{\text{th}}$ order statistic from the query $X$ (as long as $\rho_n(X, X_{(1)}) > b_n$), we see that the $Y_{(i)}$ are conditionally independent of each other. Hence, if $\rho_n(X, X_{(1)}) > b_n$, then the $Y_{(i)}$ are iid Bernoulli random variables with probability $1/2$ of being 1 and $1/2$ of being 0.

We now need to find a lower bound for the probability that for all $i \in \{1, 2, \dots, n\}$, $\rho_n(X, X_i) > b_n$. We find that
\begin{align*}
&\phantom{{}={}}\Prob(\rho_n(X, X_1) > b_n \text{ and } \ldots \text{ and } \rho_n(X, X_n) > b_n) \\
&= \prod_{i=1}^n \Prob(\rho_n(X, X_i) > b_n) \\
&= {\left(\Prob(\rho_n(X, X_1) > b_n)\right)}^n \\
&= {\left( \Prob\left( X_{1, x} \in [0, 1] \setminus \left[ X_x - \frac{b_n}{a_n}, X_x + \frac{b_n}{a_n} \right] \right) \right)}^n \\
&= {\left( \E\left[ \Prob\left( X_{1, x} \in [0, 1] \setminus \left[ X_x - \frac{b_n}{a_n}, X_x + \frac{b_n}{a_n} \right] \middle| X_x \right) \right] \right)}^n \\
&= {\left( \E\left[ \mu\left( \left([0, 1] \setminus \left[ X_x - \frac{b_n}{a_n}, X_x + \frac{b_n}{a_n} \right]\right) \times \{0, 1\} \right) \right] \right)}^n \\
&\geq {\left( 1 - \frac{2 b_n}{a_n} \right)}^n .
\end{align*}

The resulting bound is
\begin{align}
\label{eq:BadNormsConditionLowerBound}
\Prob(\rho_n(X, X_1) > b_n \text{ and } \ldots \text{ and } \rho_n(X, X_n) > b_n) \geq {\left( 1 - \frac{2 b_n}{a_n} \right)}^n .
\end{align}

Equivalently, we see that probability that the closest order statistic is closer than $b_n$ to the query is bounded from below by ${\left( 1 - \frac{2 b_n}{a_n} \right)}^n$.

Let $k$ be any odd integer between 1 and $n$, $1 \leq k \leq n$. Suppose the query $X$ is on the lower axes, with $Y = 0$ (this occurs with probability $1 / 2$). We then apply $k$-NN with $\rho_n$ as the choice of norm. If $\rho_n(X, X_i) > b_n$ holds for all $1 \leq i \leq n$, then as we have seen above, the $Y_{(i)}$ are iid Bernoulli random variables with $1/2$ probability of being one. Hence the distribution of the fraction of the $k$ nearest points to $X$ that take value one is $C = \frac1k \text{Binomial}(1/2, k)$. We find that for all odd $k$, $\Prob(C \geq 1/2) = 1/2$ (we suppose that $k$ is odd to avoid the case of ties, which slightly complicates our discussion). By symmetry, the argument holds the same way if the query $X$ is on the upper axes, hence the conditional probability of a point being misclassified if $\rho_n(X, X_i) > b_n$ is $1/2$.

Now, suppose we have the sequence of norms that is unbounded above (of the form in equation \eqref{eq:BadNormSequenceGeneralForm} with $a_n = n^2$, $b_n = 1$):
\begin{align}
\label{eq:BadNormSequenceUnboundedAbove}
\rho_n(\vecv) = \norm{\begin{pmatrix}
n^2 & 0 \\
0 & 1
\end{pmatrix} \vecv}_\infty
\end{align}

We then have, by equation \eqref{eq:BadNormsConditionLowerBound}, that
\begin{align*}
\Prob(\rho_n(X, X_1) > 1 \text{ and } \ldots \text{ and } \rho_n(X, X_n) > 1) &\geq {\left( 1 - \frac{2}{n^2} \right)}^n \\
&= \frac{{(n^2 - 2)}^n}{n^{2n}} \\
&= \frac{n^{2n} + \mathcal{O}(n^{2n-1})}{n^{2n}} \to 1 \text{ as } n \to \infty .
\end{align*}

This means that the probability of all of the sample points being at least distance 1 from the query approaches 1, and hence we see that the probability of a query being misclassified approaches $1/2$ as $n \to \infty$. Since the Bayes error is zero here, this means that $k$-NN with this sequence of norms is not consistent with this distribution.

Now, suppose we have the sequence of norms that is not bounded from below by any norm (of the form in equation \eqref{eq:BadNormSequenceGeneralForm} with $a_n = 1$, $b_n = n^2$):
\begin{equation}
\label{eq:BadNormSequenceUnboundedBelow}
\rho_n(\vecv) = \norm{\begin{pmatrix}
1 & 0 \\
0 & 1/n^2
\end{pmatrix} \vecv}_\infty
\end{equation}

Similarly to the previous case, we have by equation \eqref{eq:BadNormsConditionLowerBound} that
\begin{align*}
\Prob(\rho_n(X, X_1) > 1/n^2 \text{ and } \ldots \text{ and } \rho_n(X, X_n) > 1) &\geq {\left( 1 - \frac{2}{n^2} \right)}^n \\
&= \frac{{(n^2 - 2)}^n}{n^{2n}} \\
&= \frac{n^{2n} + \mathcal{O}(n^{2n-1})}{n^{2n}} \to 1 \text{ as } n \to \infty .
\end{align*}

Similarly to the previous case, we see that the misclassification error approaches $1/2$ as $n \to \infty$, and hence $k$-NN with this sequence of norms is not consistent with this distribution $\mu$.

We now see that $k$-NN with the sequences of norms \eqref{eq:BadNormSequenceUnboundedAbove} and \eqref{eq:BadNormSequenceUnboundedBelow} is not universally consistent. Both sequences of norms do not satisfy the boundedness conditions required in Theorem \ref{theorem:KnnWithSandwichedNormIsUniversallyConsistent}. The first sequence \eqref{eq:BadNormSequenceUnboundedAbove} is unbounded above, since if we take the vector $\vecv = (1, 0)$, then $\rho_n(\vecv) = n^2 \to \infty$ as $n \to \infty$. The second sequence \eqref{eq:BadNormSequenceUnboundedBelow} is unbounded below by any norm, since if we take the nonzero vector $\vecw = (0, 1)$, $\rho_n(\vecw) = 1 / n^2 \to 0$ as $n \to \infty$, while the norm of any nonzero vector is nonzero. Hence we see that we need to bound the sequence of norms from above and from below, otherwise universal consistency may not hold.

\section{Failure of the Geometric Stone's Lemma for Quasinorms}

We now give an example that shows that the geometric Stone's Lemma does not hold for the $\ell^p$ quasinorms, with $0 < p < 1$. In particular, we show that it fails for the $\ell^{1/2}$ quasinorm on $\R^2$ (which we denote $\rho$), that no cone can contain both the vector $(1, 0)$ and vectors nearby with nonzero $y$-component and satisfy the property that if $\vecx, \vecy$ are vectors in the cone with $0 < \rho(\vecx) < \rho(\vecy)$, then $\rho(\vecy - \vecx) < \rho(\vecy)$. We further show that this problem cannot be avoided by considering axis vectors separately from vectors that are not on an axis.
\begin{example}
\label{example:FailureOfStonesLemmaForQuasinorms}
Suppose we have the point $\vecy = (1, 0)$ in $\R^2$ and a value $r \in (0, 1)$ such that $(1, r)$ is still in the cone. Any positive multiple of a vector inside the cone is in the cone, so it follows that $\vecx = r (1, r) = (r, r^2)$  is in the cone as well. If we take $\rho$ to be the $\ell^{1/2}$ quasinorm and $0 < r < 1/4$, we find that
\begin{align*}
\rho(\vecx) &= \rho\left(\left( r, r^2 \right)\right) \\
&= {\left( \sqrt{\left|r\right|} + \sqrt{\left|r^2\right|} \right)}^2 \\
&= r^2 + 2 r \sqrt{r} + r \\
&\leq {\left(\frac14\right)}^2 + 2 \left(\frac14\right) \sqrt{\frac14} + \frac14 \\
&= 0.5625 \\
&< 1 = \rho(\vecy).
\end{align*}

\begin{figure}
\centering
\includegraphics[scale=1.5]{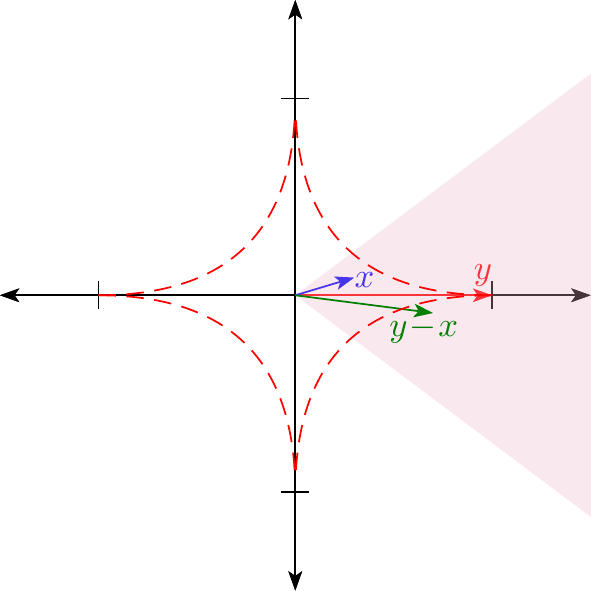}
\caption{In this illustration we have that $\rho(\vecy - \vecx) > \rho(\vecx)$ while $0 < \rho(\vecx) < \rho(\vecy)$ (with $\rho$ being the $\ell^{1/2}$ quasinorm). In this example $\rho(\vecy) = 1$. The dashed curve is the unit sphere (in the $\ell^{1/2}$ quasinorm). We see that the $\vecy - \vecx$ vector lies outside the unit ball. Such pairs of vectors $\vecx, \vecy$ are possible for arbitrarily thin cones around an axis in the $\ell^{1/2}$ quasinorm.}
\label{fig:QuasinormsCounterexample}
\end{figure}

We additionally find that
\begin{align*}
\rho(\vecy - \vecx) &= \rho\left(\left( 1, 0 \right) - \left( r, r^2 \right)\right) \\
&= \rho\left(\left( 1 - r, -r^2 \right) \right) \\
&= {\left( \sqrt{\left|1 - r\right|} + \sqrt{\left| - r^2 \right|} \right)}^2 \\
&= 1 - r + 2r \sqrt{1 - r} + r^2 .
\end{align*}

We now define the function $f(r)$ by subtracting $\rho(\vecy)$ from $\rho(\vecy - \vecx)$,
\begin{equation}
f(r) = r^2 + 2r \sqrt{1 - r} - r .
\end{equation}

We need to show that $f(r) > 0$ for all $r \in (0, 1/4)$. We find that $f(0) = 0$ and the derivative is
\begin{equation}
f'(r) = 2r + \frac{r}{\sqrt{1 - r}} + 2 \sqrt{1 - r} - 1 .
\end{equation}

We see that for all $0 < r < 1/4$, $2r + \frac{r}{\sqrt{1 - r}} \geq 0$, and that
\begin{equation*}
2 \sqrt{1 - r} - 1 \geq 2 \sqrt{1 - \frac14} - 1 \approx 0.732051 .
\end{equation*}

It follows that $f(r) > 0$ for all $r \in (0, 1/4)$, since $f$ is strictly increasing on this interval and $f(0) = 0$. This means that $\rho(\vecy - \vecx) - \rho(\vecy) > 0$, or equivalently $\rho(\vecy - \vecx) > \rho(\vecy)$. We have earlier found that $\rho(\vecx) < \rho(\vecy)$. Hence we see that no cone that satisfies the condition of the geometric Stone's lemma (that if $\vecx, \vecy$ are vectors in the cone with $0 < \rho(\vecx) < \rho(\vecy)$, then $\rho(\vecy - \vecx) < \rho(\vecy)$) can contain both an axis vector and any vector that does not lie on that axis.

We see that $\rho$ is continuous, which means that if $\rho(\vecy - \vecx) > \rho(\vecy)$ and $\rho(\vecx) < \rho(\vecy)$, there exists a neighbourhood of radius $\delta > 0$ around $\vecy$ such that for all $\vecy' \in B_r(\vecy, \norm{\cdot})$, $\rho(\vecy' - \vecx) > \rho(\vecy')$ and $\rho(\vecx) < \rho(\vecy')$. It follows that we can replace $\vecy$ with $\vecy' = (1, \delta / 2)$, which has a nonzero $y$-coordinate. From this we see that simply putting axis vectors in their own category does not fix the above problem, any cone that contains vectors arbitrarily close to an axis has this problem.
\end{example}

By this example, we see that we cannot prove the universal consistency of $k$-NN for quasinorms using the classical approach with Stone's lemma with cones. This does not necessarily mean that $k$-NN is inconsistent with $\ell^p$-quasinorms (indeed, this would be very surprising), it simply means that the geometric Stone's lemma with cones cannot be used to prove this result.

\chapter{$k$-NN with a Sequence of Random Uniformly Locally Lipschitz Functions}

We would like to generalize norms. Suppose we have a function $f: \R^d \to \R$, it can serve as a function to measure the ``distance" between two points as follows: given points $\vecx, \vecy \in \R^d$, we compute $f(\vecx - \vecy)$ and take this to be our distance between $\vecx$ and $\vecy$. We require $f$ to be nonnegative and that $f(\vecx) = 0$ only at $\vecx = 0$. We do not require that $f$ satisfies the triangle inequality or absolute homogeneity (which norms have to satisfy). In particular, we do not require the distance we have just defined (in terms of $f$) to be a metric.

We can now consider $k$-NN with the distance function $f$ (that is, given points $\vecx, \vecy \in \R^d$, we take $f(\vecx - \vecy)$ as the distance between $\vecx$ and $\vecy$ for $k$-NN). Intuitively, any such function that is increasing away from zero is a possible candidate for $k$-NN. We illustrate some examples of such functions in Figure \ref{fig:VariousFunctions}. In this section, we show that $k$-NN is universally consistent with such a function (or more generally, a sequence of such functions independent of the sample and the query) under certain conditions (in particular, that the family of functions we consider is uniformly locally Lipschitz near zero (with respect to the Euclidean norm, or equivalently any other norm on $\R^d$), in addition to a few other conditions). The question of universal consistency of the k-NN under quasinorms remains open, since quasinorms are not necessarily locally Lipschitz near the origin (in particular, this condition does not hold for $\ell^p$ quasinorms on $\R^d$ with $0 < p < 1$, even though the $\ell^p$ quasinorms are uniformly continuous on $\R^d$ with respect to the Euclidean norm).

\begin{figure}[h]
\centering
\includegraphics[scale=.5]{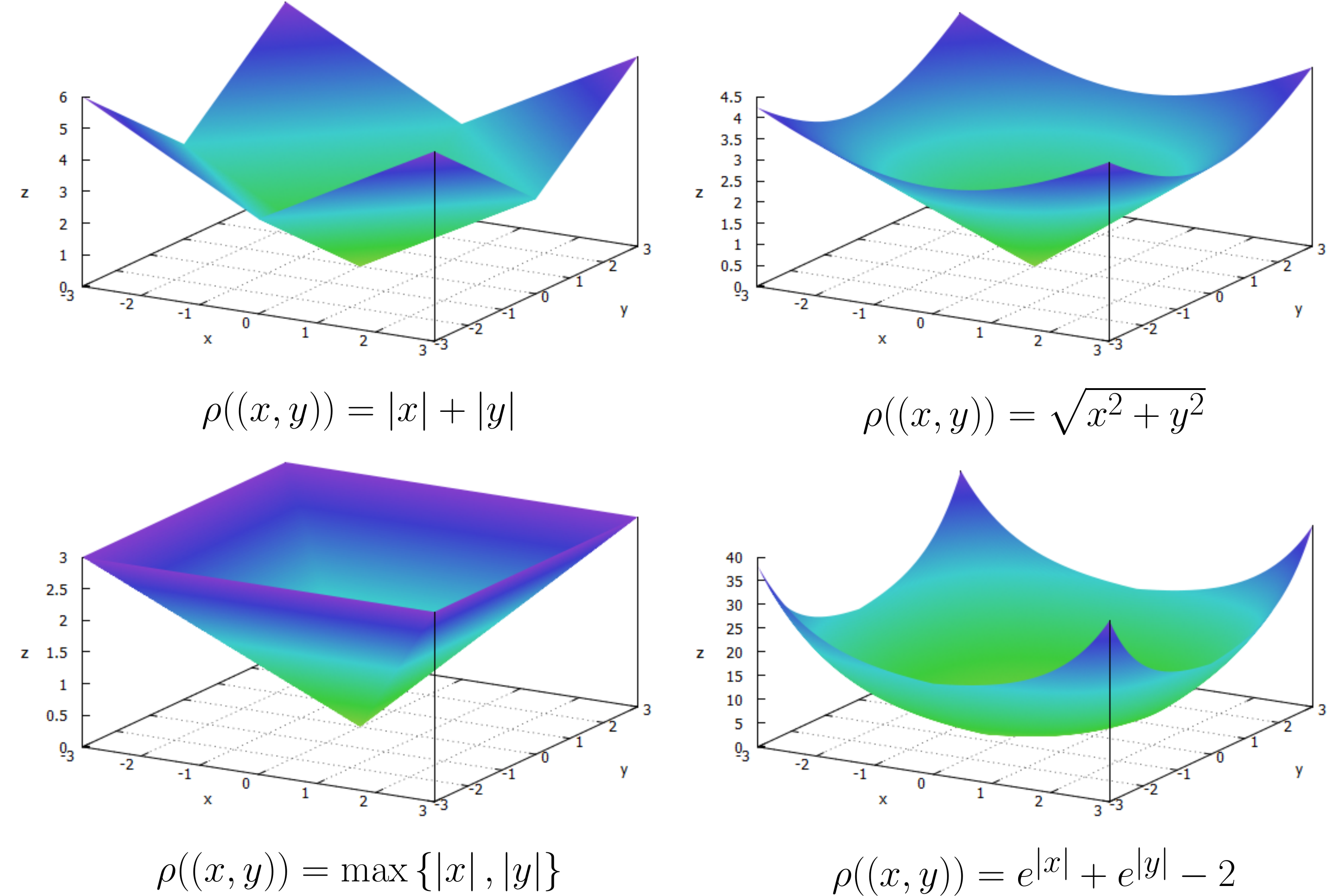} 
\caption{Illustration showing the graphs of the $\ell^1$ norm (top left), $\ell^2$ norm (top right), $\ell^\infty$ norm (bottom left), and $\rho((x, y)) = e^{|x|} + e^{|y|} - 2$ (bottom right) as functions on $\R^2$. We see that all of these functions are increasing as we move away from the origin, and we can use them as distances with $k$-NN (using the method we discussed). We would like to establish that $k$-NN with functions like the one on the bottom right is universally consistent.}
\label{fig:VariousFunctions}
\end{figure}

\section{General Theory}
Let $\F$ be a family of measurable functions on $(\R^d, \norm{\cdot})$ such that there exist constants $\alpha, \beta, \gamma > 0$ and a radius $r > 0$ where:
\begin{enumerate}
\item All $\rho \in F$ are Lipschitz with constant $\alpha$ on the domain $B_r(x)$, that is, for all $\vecx, \vecy \in B_r(x)$,
\begin{equation}
\label{eq:LipschitzFamilyFunctionDef}
\left| \rho(\vecx) - \rho(\vecy) \right| \leq \alpha \norm{\vecx - \vecy} .
\end{equation}
\item For any nonzero vector $\vecv \in B_r(\veczero, \norm{\cdot})$ (with $\vecv \neq \veczero$), if we define $f(\lambda) = \rho(\lambda \vecv)$, then $f'(\lambda) \geq \beta \norm{\vecv}$ for all $0 < \lambda < 1$.\footnote{It should be possible to replace this condition by a similar lower bound with the limit in the derivative replaced by liminf.}
\item Outside of the ball of radius $r$, $\rho$ is bounded from below by $\gamma$, that is, for all $\vecv \in \R^d$ with $\norm{\vecv} \geq r$, $\rho(\vecv) \geq \gamma$.
\item The function $\rho$ is symmetric, $\rho(\vecx) = \rho(-\vecx)$ for all $\vecx \in \R^d$.\footnote{It is almost certainly possible to remove or at least weaken this condition.}
\item The function $\rho$ takes value zero at the point zero, $\rho(\veczero) = 0$.
\end{enumerate}
We may assume without loss of generality that $\alpha \geq 1$ and that $\beta \leq 1$, this simplifies some of our proofs.

We show that $k$-NN is universally consistent with any sequence of functions from the family $\F$ that is independent of the sample and the query (like Theorem \ref{theorem:KnnWithSandwichedNormIsUniversallyConsistent}, but with family of norms replaced by a family of uniformly locally Lipschitz functions). We do this by showing that the conditions of Stone's Theorem (Theorem \ref{theorem:StonesTheorem}) are satisfied. For a query $X$ at step $n$, we define $X_{(1, \rho)}, X_{(2, \rho)}, \dots, X_{(n, \rho)}$ to be the order statistics in increasing distance from $X$, with $\rho$ as the distance. We define the weight function to be
\begin{equation}
\label{eq:LipschitzFamilyWeightFunction}
W_{ni}(X) = \begin{dcases}
\frac1k & \text{if } X_i \in \{X_{(1, \rho)}, X_{(2, \rho)}, \dots, X_{(k, \rho)}\} \\
0 & \text{otherwise} .
\end{dcases}
\end{equation}

\begin{lemma}
\label{lemma:LipschitzOneSidedTriangleInequality}
Let $f: \R^d \to \R$ be a Lipschitz continuous function on $B_r(\veczero, \norm{\cdot})$ with Lipschitz constant $\alpha$. Then for all $\vecx, \vecy$ such that $\norm{\vecx} \leq r$, $\norm{\vecy} \leq r$, and $\norm{\vecx + \vecy} \leq r$, we have the inequality
\begin{equation}
\left| f(\vecx + \vecy) \right| \leq \left| f(\vecx) \right| + \alpha \norm{\vecy} .
\end{equation}
\end{lemma}
\begin{proof}
We notice that since $f$ has Lipschitz constant $\alpha$ on $B_r(\veczero, \norm{\cdot})$, for all $\vecx', \vecy' \in B_r(\veczero, \norm{\cdot})$,
\begin{equation}
\left| f(\vecx') - f(\vecy') \right| \leq \alpha \norm{\vecx' - \vecy'} .
\end{equation}

By the reverse triangle inequality, we have that $\left|f(\vecx')\right| - \left|f(\vecy')\right| \leq \left|f(\vecx') - f(\vecy')\right|$, so it follows that
\begin{equation}
\left|f(\vecx')\right| - \left|f(\vecy')\right| \leq \alpha \norm{\vecx' - \vecy'} .
\end{equation}

Rearranging, we find
\begin{equation}
\left|f(\vecx')\right| \leq \left|f(\vecy')\right| + \alpha\norm{\vecx' - \vecy'} .
\end{equation}

We conclude the proof by taking $\vecx' = \vecx + \vecy$ and $\vecy' = \vecx$. By assumption, both $\vecx' = \vecx + \vecy \in B_r(\veczero, \norm{\cdot})$ and $\vecy' = \vecx \in B_r(\veczero, \norm{\cdot})$. Hence we find
\begin{equation}
\left|f(\vecx + \vecy)\right| \leq \left|f(\vecx)\right| + \alpha \norm{\vecy} .
\end{equation}
\end{proof}

We now define an inner radius $\delta$ by
\begin{equation}
\label{eq:LipschitzFamilyInnerRadius}
\delta = \frac{1}{4 \alpha} \min\{r, \gamma\} .
\end{equation}

We see that $0 < \delta \leq r / 4$ and $0 < \delta \leq \gamma / 4$.

\begin{lemma}
\label{lemma:LipschitzFamilyFunctionIsContinuous}
For any $\rho \in \F$ and $\vecx \in B_\delta(\veczero, \norm{\cdot})$, the function $f: [0, 1] \to \R$, $f(\lambda) = \rho(\lambda \vecx)$ is continuous on $[0, 1]$, with $f(0) = 0$.
\end{lemma}
\begin{proof}
We have that $f$ is the composition of continuous functions (multiplication by a constant, followed by the Lipschitz function $\rho$), hence $f$ is continuous. In addition, since $\rho(\veczero) = 0$, $f(0) = 0$.
\end{proof}

\begin{lemma}
\label{lemma:LipschitzFamilyInnerRadius}
We let $\rho \in \F$ be an arbitrary function in $\F$. For all $\vecx \in B_\delta(\veczero, \norm{\cdot})$, let $\beta \norm{\vecx} \leq \rho(\vecx) \leq \alpha \norm{\vecx}$ and $\rho(\vecx) < \gamma / 4$. Then for all $\vecy \in \R^d$ such that $\norm{\vecy} \geq \delta$, we have $\rho(\vecy) \geq \beta \delta$.
\end{lemma}
\begin{proof}
We let $\vecx \in B_\delta(\veczero, \norm{\cdot})$. If $\vecx = \veczero$, then both $\norm{\vecx} = 0$ and $\rho(\vecx) = 0$ and we are done. We now suppose that $\vecx \neq \veczero$. We define $f(\lambda) = \rho(\lambda \vecx)$, by assumption we have that $f'(\lambda) \geq \beta \norm{\vecx}$ for all $\lambda \in (0, 1)$. By Lemma \ref{lemma:LipschitzFamilyFunctionIsContinuous}, $f$ is continuous on $[0, 1]$, with $f(0) = \rho(\veczero) = 0$ and $f(1) = \rho(\vecx)$. By the mean value theorem, there exists a $c \in (0, 1)$ such that $f(1) - f(0) = f'(c) (1 - 0)$, which implies that $f'(c) = f(1)$. We then find that
\begin{align*}
\rho(\vecx) &= f(1) \\
&= f'(c) \\
&\geq \beta \norm{\vecx} .
\end{align*}

Since $\rho$ is Lipschitz continuous on $B_\delta(\veczero, \norm{\cdot})$ with Lipschitz constant $\alpha$, we also have that $\rho(\vecx) - \rho(\veczero) \leq \alpha (\norm{\vecx} - \norm{\veczero})$, which implies (with $\rho(\veczero) = 0$) that $\rho(\vecx) \leq \alpha \norm{\vecx}$. Hence we have that $\beta \norm{\vecx} \leq \rho(\vecx) \leq \alpha \norm{\vecx}$. We additionally notice that $\norm{\vecx} < \delta$ and so 
\begin{align*}
\rho(\vecx) &\leq \alpha \norm{\vecx} \\
&< \alpha \delta \\
&\leq \alpha \frac{\gamma}{4 \alpha} \\
&= \frac{\gamma}{4} .
\end{align*}

We let $\vecy \in \R^d$ be a point not in the open ball $B_\delta(\veczero, \norm{\cdot})$, so that $\norm{\vecy} \geq \delta$. If $\norm{\vecy} \geq r$, then $\rho(\vecy) \geq \gamma \geq \beta \gamma \geq \beta \delta$, and we are done. If $\delta \leq \norm{\vecy} < r$, then we let $f(\lambda) = \rho(\lambda \vecy)$, with $f$ being differentiable on $(0, 1)$ with $f'(\lambda) \geq \beta \norm{\vecy}$ for all $\lambda \in (0, 1)$. We also have (by Lemma \ref{lemma:LipschitzFamilyFunctionIsContinuous}) that $f$ is continuous on $[0, 1]$, $f(0) = 0$, and $f(1) = \rho(\vecy)$. By the mean value theorem there exists a $c \in (0, 1)$ such that $f(1) = f'(c)$. We then find that
\begin{align*}
\rho(\vecy) &= f(1) \\
&= f'(c) \\
&\geq \beta \norm{\vecy} \\
&\geq \beta \delta .
\end{align*}

We conclude that if $\norm{\vecy} \geq \delta$, then $\rho(\vecy) \geq \beta \delta$.
\end{proof}

\begin{lemma}
\label{lemma:LipschitzFamilyWeakTriangleInequality}
For all points within $B_\delta(\veczero, \norm{\cdot})$, a triangle inequality with a $\alpha / \beta$ multiplicative constant holds, that is if $\norm{\vecx} < \delta$ and $\norm{\vecy} < \delta$, then
\begin{equation}
\rho(\vecx + \vecy) \leq \frac{\alpha}{\beta} (\rho(\vecx) + \rho(\vecy)) .
\end{equation}
\end{lemma}
\begin{proof}
We see that for all $\vecx, \vecy \in B_\delta(\veczero, \norm{\cdot})$, by Lemma \ref{lemma:LipschitzFamilyInnerRadius},
\begin{align*}
\rho(\vecx + \vecy) &\leq \alpha \norm{\vecx + \vecy} \\
&\leq \alpha (\norm{\vecx} + \norm{\vecy}) \\
&\leq \alpha \left(\frac{\rho(\vecx)}{\beta} + \frac{\rho(\vecy)}{\beta}\right) \\
&= \frac{\alpha}{\beta} \left(\rho(\vecx) + \rho(\vecy)\right) .
\end{align*}
\end{proof}

\begin{lemma}
\label{lemma:LipschitzFamilyKnnDistanceInequality}
Suppose we have an iid query and sample points $X, X_1, X_2, \dots, X_n$. We let $X_{(1, \norm{\cdot})}, X_{(2, \norm{\cdot})}, \dots, X_{(n, \norm{\cdot})}$ be the points in increasing distance from $X$ with respect to the norm $\norm{\cdot}$ and $X_{(1, \rho)}, X_{(2, \rho)}, \dots, X_{(n, \rho)}$ be the points in increasing distance from $X$ with respect to the $\rho$ distance, for any function $\rho \in \F$. We define the weight function $W_{ni}(X)$ to be $1 / k$ if $X_i$ is one of the $k$-nearest neighbours of $X$ in the $\rho$ distance (that is, if $X \in \{X_{(1, \rho)}, X_{(2, \rho)}, \dots, X_{(k, \rho)}\}$, and to be zero otherwise. For any $a > 0$ and $\rho \in \F$,
\begin{equation}
\label{eq:LipschitzFamilyKnnDistanceInequality}
\E \left[ \sum_{i=1}^n W_{ni} (X) \mathds{1}_{\{\norm{X_i - X} > a\}} \right] \leq \Prob\left( \norm{X_{\left(k, \norm{\cdot}\right)} - X} > \frac{\beta}{2 \alpha} \min\{a, \delta\}\right) .
\end{equation}
\end{lemma}
\begin{proof}
We show that if the condition on the right hand side of \eqref{eq:LipschitzFamilyKnnDistanceInequality} does not hold, that is, if $\norm{X_{\left(k, \norm{\cdot}\right)} - X} \leq \frac{\beta}{2 \alpha} \min\{a, \delta\}$, then $\sum_{i=1}^n W_{ni} (X) \mathds{1}_{\{\norm{X_i - X} > a\}} = 0$. We assume that $\norm{X_{\left(k, \norm{\cdot}\right)} - X} \leq \frac{\beta}{2 \alpha} \min\{a, \delta\}$, since this is less than $\delta$ by assumption we can use our above results. We first define $a' = \min\{a, \delta\}$. We let $i \in \{1, 2, \dots, k\}$ and observe that
\begin{align*}
\rho(X - X_{(i, \rho)}) &\leq \alpha \norm{X - X_{(i, \rho)}} \\
&\leq \alpha \norm{X - X_{(k, \rho)}} \\
&\leq \alpha \frac{\beta}{2\alpha} a' \\
&= \frac{\beta}{2} a' .
\end{align*}

This means there are at least $k$ points such that the $\rho$-distance from $X$ to the point is at most $\frac{\beta}{2} a'$. Suppose we have a point $X_j$ such that $\norm{X - X_j} > a'$. There are two possible cases: either $\norm{X - X_j} \geq \delta$ or $a' < \norm{X - X_j} < \delta$ (only the first case is possible if $a' = \delta$). If $\norm{X - X_j} \geq \delta$, by Lemma \ref{lemma:LipschitzFamilyInnerRadius} we have that $\rho(X - X_j) \geq \beta \delta \geq \beta a'$. If $a' < \norm{X - X_j} < \delta$, then $\rho(X - X_j) \geq \beta \norm{X - X_j} > \beta a'$. In both cases $\rho(X - X_j) \geq \beta a' > \frac{\beta}{2} a' \geq \rho(X - X_{(i, \rho)})$, for each $1 \leq i \leq k$. From this we see that there are at least $k$ points closer to $X$ than $X_j$ in the $\rho$ distance.

From this we see that if $\norm{X_{\left(k, \norm{\cdot}\right)} - X} \leq \frac{\beta}{2 \alpha} \min\{a, \delta\}$ holds, then for all points $X_i$ such that $\norm{X_i - X} > a$, $X_i$ is not a $k$-nearest neighbour of $X$ in the $\rho$ distance, and since $W_{ni}(X)$ is only nonzero for the $k$-nearest neighbours of $X$ in the $\rho$ distance, it follows that
\begin{equation*}
\sum_{i=1}^n W_{ni} (X) \mathds{1}_{\{\norm{X_i - X} > a\}} = 0 .
\end{equation*}

We notice that $\sum_{i=1}^n W_{ni} (X) \mathds{1}_{\{\norm{X_i - X} > a\}} \leq 1$ always and is nonnegative (since $\sum_{i=1}^n W_{ni} (X) = 1$ and the indicator function is either zero or one). We then condition on $\norm{X_{\left(k, \norm{\cdot}\right)} - X} > \frac{\beta}{2 \alpha} a'$ to find that
\begin{align*}
&\phantom{{}={}}\E \left[ \sum_{i=1}^n W_{ni} (X) \mathds{1}_{\{\norm{X_i - X} > a\}} \right] \\
&\leq \E \left[ \sum_{i=1}^n W_{ni} (X) \mathds{1}_{\{\norm{X_i - X} > a\}} \middle| \norm{X_{\left(k, \norm{\cdot}\right)} - X} > \frac{\beta}{2 \alpha} a' \right] \Prob\left(\norm{X_{\left(k, \norm{\cdot}\right)} - X} > \frac{\beta}{2 \alpha} a'\right) +\\&\phantom{{}={}} \E \left[ \sum_{i=1}^n W_{ni} (X) \mathds{1}_{\{\norm{X_i - X} > a\}} \middle| \norm{X_{\left(k, \norm{\cdot}\right)} - X} \leq \frac{\beta}{2 \alpha} a' \right] \Prob\left(\norm{X_{\left(k, \norm{\cdot}\right)} - X} \leq \frac{\beta}{2 \alpha} a'\right) \\
&\leq (1) \left(\Prob\left(\norm{X_{\left(k, \norm{\cdot}\right)} - X} > \frac{\beta}{2 \alpha} a'\right)\right) + 0 \\
&= \Prob\left(\norm{X_{\left(k, \norm{\cdot}\right)} - X} > \frac{\beta}{2 \alpha} \min\{a, \delta\}\right) .
\end{align*}
\end{proof}

\begin{lemma}
\label{lemma:LipschitzFamilyAwayGoesToZero}
For any $a > 0$ and $\epsilon > 0$, there exists an $N \geq 1$ such that for all $n \geq N$, for all $\rho \in \F$, with $W_{ni}(X)$ as defined in Lemma \ref{lemma:LipschitzFamilyKnnDistanceInequality},
\begin{equation}
\E \left[ \sum_{i=1}^n W_{ni}(X) \Indicator_{\{\norm{X_i - X} > a\}} \right] < \epsilon .
\end{equation}
\end{lemma}
\begin{proof}
By Lemma \ref{lemma:LipschitzFamilyKnnDistanceInequality}, we have that for all $\rho \in \F$,
\begin{equation}
\label{eq:LipschitzFamilyAwayGoesToZeroProof1}
\E \left[ \sum_{i=1}^n W_{ni} (X) \mathds{1}_{\{\norm{X_i - X} > a\}} \right] \leq \Prob\left( \norm{X_{\left(k, \norm{\cdot}\right)} - X} > \frac{\beta}{2 \alpha} \min\{a, \delta\}\right) .
\end{equation}

We define $a' = \frac{\beta}{2 \alpha} \min\{a, \delta\}$ (we notice that $\delta$ is a fixed constant for the family $\F$, defined by \eqref{eq:LipschitzFamilyInnerRadius}). By Lemma \ref{lemma:ProbKPointIsFarGoesToZero}, $\Prob\left( \norm{X_{\left(k, \norm{\cdot}\right)} - X} > a'\right) \to 0$ as $n \to \infty$, so there exists an $N \geq 1$ such that for all $n \geq N$,
\begin{equation}
\label{eq:LipschitzFamilyAwayGoesToZeroProof2}
\Prob\left( \norm{X_{\left(k, \norm{\cdot}\right)} - X} > \frac{\beta}{2 \alpha} \min\{a, \delta\}\right) < \epsilon .
\end{equation}

Hence it follows from \eqref{eq:LipschitzFamilyAwayGoesToZeroProof1} and \eqref{eq:LipschitzFamilyAwayGoesToZeroProof2} that for all $\rho \in \F$ and $n \geq N$,
\begin{equation}
\label{eq:LipschitzFamilyAwayGoesToZeroProof3}
\E \left[ \sum_{i=1}^n W_{ni} (X) \mathds{1}_{\{\norm{X_i - X} > a\}} \right] < \epsilon .
\end{equation}
\end{proof}

\begin{lemma}
\label{lemma:LipschitzFamilyConesCovering}
There is a uniform upper bound $c > 0$ such that for any $\rho \in \F$, there are at most $c$ subsets $S_1, S_2, \dots, S_n$ covering $B_\delta(\veczero, \norm{\cdot})$ such that for each $S_i$ (where $1 \leq i \leq c$), if $\vecx, \vecy \in S_i$ with $\vecx \neq \veczero$ and $\rho(\vecx) \leq \rho(\vecy)$, then $\rho(\vecy - \vecx) < \rho(\vecy)$.
\end{lemma}
\begin{proof}
Since $B_{\delta}(\veczero, \norm{\cdot})$ is a bounded subset of $\R^d$ there exists a covering with $c$ open balls of radius $\beta^2 / (4 \alpha^2)$ each in the $\norm{\cdot}$ norm (this follows from the fact that any bounded subset of the normed space $(\R^d, \norm{\cdot})$ is precompact or totally bounded). We then find that for any $\vecv \in B_\delta(\veczero, \norm{\cdot})$ (using Lemma \ref{lemma:LipschitzFamilyInnerRadius}),
\begin{align*}
0 &\leq \norm{\frac{\vecv}{\rho(\vecv)}} \\
&= \frac{1}{\rho(\vecv)} \norm{\vecv} \\
&\leq \frac{1}{\beta \norm{\vecv}} \norm{\vecv} \\
&\leq \frac1\beta .
\end{align*}

This means that for any vector $\vecv \in B_\delta(\veczero, \norm{\cdot}) \setminus \{\veczero\}$, $\frac{\vecv}{\rho(\vecv)} \in B_{2/\beta}(\veczero, \norm{\cdot})$ (since $0 \leq \norm{\frac{\vecv}{\rho(\vecv)}} < 2 / \beta$). We then let $T_1, T_2, \dots, T_c$ be the open balls that cover $B_{2/\beta}(\veczero, \norm{\cdot})$ with radius $\beta^2 / (4 \alpha^2)$ each and $S_1, S_2, \dots, S_c$ be the subsets of $B_\delta(\veczero, \norm{\cdot})$ such that $\vecv \in S_i$ if and only if either $\vecv = \veczero$ or $\vecv \neq \veczero$ and $\frac{\vecv}{\rho(\vecv)} \in T_i$. We see that every vector $\vecv \in B_\delta(\veczero, \norm{\cdot})$ is in at least one $S_i$ since the zero vector is in every $S_i$ and every nonzero such vector has $\frac{\vecv}{\rho(\vecv)} \in T_i$ for some $1 \leq i \leq c$ so $\vecv \in S_i$. For each $T_i$ we let $\vecx_i \in T_i$ be an element of $T_i$.

Suppose that $\vecx, \vecy \in S_i$ with $\rho(\vecx) \leq \rho(\vecy)$ and $\rho(\vecx) > 0$. Since $\vecx, \vecy \in S_i$ we have that $\vecx / \norm{\vecx}, \vecy / \norm{\vecy} \in T_i$, and since $\vecx_i \in T_i$ we have that $\norm{\frac{\vecy}{\rho(\vecy)} - \vecx_i} < \frac{\beta^2}{4\alpha^2}$ and that $\norm{\vecx_i - \frac{\vecx}{\rho(\vecx)}} < \frac{\beta^2}{4\alpha^2}$. We then find that
\begin{align*}
\norm{\frac{\rho(\vecx)}{\rho(\vecy)} \vecy - \vecx} &= \rho(\vecx) \norm{\frac{\vecy}{\rho(\vecy)} - \frac{\vecx}{\rho(\vecx)}} \\
&= \rho(\vecx) \norm{\frac{\vecy}{\rho(\vecy)} - \vecx_i + \vecx_i - \frac{\vecx}{\rho(\vecx)}} \\
&\leq \rho(\vecx) \left(\norm{\frac{\vecy}{\rho(\vecy)} - \vecx_i} + \norm{\vecx_i - \frac{\vecx}{\rho(\vecx)}}\right) \\
&< \rho(\vecx) \left(\frac{\beta^2}{4\alpha^2} + \frac{\beta^2}{4\alpha^2}\right) \\
&= \frac{\beta^2}{2\alpha^2} \rho(\vecx) .
\end{align*}

We notice that $\vecy - \vecx = \vecy - \frac{\rho(\vecx)}{\rho(\vecy)} \vecy + \frac{\rho(\vecx)}{\rho(\vecy)} \vecy - \vecx$. We see that $\norm{\vecx} < \delta$, $\norm{\vecy} < \delta$, and that
\begin{align*}
\norm{\frac{\rho(\vecx)}{\rho(\vecy)} \vecy} &= \frac{\rho(\vecx)}{\rho(\vecy)} \norm{\vecy} \\
&\leq \norm{\vecy} \\
&< \delta .
\end{align*}

From this we see that the norm of the first part is less than $r$,
\begin{align*}
\norm{\vecy - \frac{\rho(\vecx)}{\rho(\vecy)} \vecy} &\leq \norm{\vecy} + \norm{\frac{\rho(\vecx)}{\rho(\vecy)}} \\
&< \delta + \delta \\
&\leq r/2 .
\end{align*}

The second part has a norm less than $r$,
\begin{align*}
\norm{\frac{\rho(\vecx)}{\rho(\vecy)} \vecy - \vecx} &\leq \norm{\vecx} + \norm{\frac{\rho(\vecx)}{\rho(\vecy)} \vecy} \\
&< \delta + \delta \\
&\leq r / 2 .
\end{align*}

The norm of sum of both parts is less than $r$,
\begin{align*}
\norm{\vecy - \frac{\rho(\vecx)}{\rho(\vecy)} \vecy + \frac{\rho(\vecx)}{\rho(\vecy)} \vecy - \vecx} &\leq \norm{\vecx} + \norm{\vecy} + 2\norm{\frac{\rho(\vecx)}{\rho(\vecy)} \vecy} \\
&< 4\delta \\
&\leq r .
\end{align*}

We are now able to apply Lemma \ref{lemma:LipschitzOneSidedTriangleInequality} and our above observations to find that (along with the fact that since $\alpha \geq 1$ and $\beta > 0$, $\frac{\beta^2}{2 \alpha} - \beta^2 < 0$)
\begin{align*}
\rho(\vecy - \vecx) &= \rho\left(\vecy - \frac{\rho(\vecx)}{\rho(\vecy)} \vecy + \frac{\rho(\vecx)}{\rho(\vecy)} \vecy - \vecx\right) \\
&\leq \rho\left(\vecy - \frac{\rho(\vecx)}{\rho(\vecy)} \vecy\right) + \alpha \norm{\frac{\rho(\vecx)}{\rho(\vecy)} \vecy - \vecx} \\
&< \rho\left(\vecy - \frac{\rho(\vecx)}{\rho(\vecy)} \vecy\right) + \alpha \frac{\beta^2}{2 \alpha^2} \rho(\vecx) \\
&= \rho\left(\vecy - \frac{\rho(\vecx)}{\rho(\vecy)} \vecy\right) + \frac{\beta^2}{2 \alpha} \rho(\vecx) \\
&\leq \rho(\vecy) - \frac{\rho(\vecx)}{\rho(\vecy)} \beta^2 \rho(\vecy) + \frac{\beta^2}{2\alpha} \rho(\vecx) \\
&= \rho(\vecy) + \left(\frac{\beta^2}{2\alpha} - \beta^2\right) \rho(\vecx) \\
&\leq \rho(\vecy) .
\end{align*}

This means that for all nonzero $\vecx, \vecy \in S_i$ with $\rho(\vecx) \leq \rho(\vecy)$, $\rho(\vecy - \vecx) < \rho(\vecy)$.
\end{proof}

\begin{lemma}
\label{lemma:LipschitzFamilyStonesTheoremCondition1}
Let $X$ be a query and $X_1, X_2, \dots, X_n$ be the sample points, all of which are iid. Let ${(\rho_n)}_{n=1}^\infty$ be a sequence of functions in $\F$ (that is possibly random, but is independent of the labelled sample and the query), and $W_{ni}$ be the corresponding weights. There exists a constant $c > 0$ (with the $c$ defined in Lemma \ref{lemma:LipschitzFamilyConesCovering}) and a sequence $\epsilon_n \to 0$ as $n \to \infty$ such that for any nonnegative measurable function $f$ bounded above by one,
\begin{equation}
\E\left[\sum_{i=1}^n W_{ni}(X) f(X_i)\right] \leq c \E\left[ f(X) \right] + \epsilon_n .
\end{equation}
\end{lemma}
\begin{proof}
We notice that
\begin{align*}
&\phantom{{}={}}\E\left[\sum_{i=1}^n W_{ni}(X) f(X_i)\right] \\
&= \E\left[\sum_{i=1}^n W_{ni}(X) f(X_i) (\Indicator_{\{\norm{X_i - X} < \delta\}} + \Indicator_{\{\norm{X_i - X} \geq \delta\}})\right] \\
&= \E\left[\sum_{i=1}^n W_{ni}(X) f(X_i) \Indicator_{\{\norm{X_i - X} < \delta\}}\right] + \E\left[\sum_{i=1}^n W_{ni}(X) f(X_i) \Indicator_{\{\norm{X_i - X} \geq \delta\}}\right] .
\end{align*}

For the first term, we have that
\begin{align*}
&\phantom{{}={}}\E\left[\sum_{i=1}^n W_{ni}(X) f(X_i) \Indicator_{\{\norm{X_i - X} < \delta'\}}\right] \\
&= \E\left[\sum_{i=1}^n \frac1k \Indicator_{\{X_i \text{ is a } \rho_n \text{ } k \text{-NN of } X \text{ among } X_1, \dots, X_i, \dots, X_n \}} f(X_i) \Indicator_{\{\norm{X_i - X} < \delta\}}\right] \\
&= \frac1k \E\left[\sum_{i=1}^n f(X_i) \Indicator_{\{X_i \text{ is a } \rho_n \text{ } k \text{-NN of } X \text{ among } X_1, \dots, X_i, \dots, X_n \text{ and } \norm{X_i - X} < \delta\}}\right] \\
&= \frac1k \E\left[\sum_{i=1}^n f(X) \Indicator_{\{X \text{ is a } \rho_n \text{ } k \text{-NN of } X_i \text{ among } X_1, \dots, X, \dots, X_n \text{ and } \norm{X - X_i} < \delta\}}\right] \\
&= \frac1k \E\left[ f(X) \sum_{i=1}^n \Indicator_{\{X \text{ is a } \rho_n \text{ } k \text{-NN of } X_i \text{ among } X_1, \dots, X, \dots, X_n \text{ and } \norm{X - X_i} < \delta\}}\right] .
\end{align*}

We define subsets $S_1, S_2, \dots, S_c$ as in Lemma \ref{lemma:LipschitzFamilyConesCovering}. In each of the subsets $S_1, S_2, \dots, S_c$, we mark the $k$ points closest to $X$ in the $\rho_n$ distance (we recall that these subsets cover $B_a(\veczero, \norm{\cdot})$). Suppose the point $X_i$ in the subset $S_q$ is not marked. Then there are at least $k$ points $X_{j_1}, X_{j_2}, \dots, X_{j_k}$ in $S_q$ such that $\rho_n(X_{j_l} - X) \leq \rho_n(X_i - X)$ that are marked. For these points, if $X_{j_l} \neq X$, then $\rho_n(X_i - X_{j_l}) < \rho_n(X_i - X)$ by Lemma \ref{lemma:LipschitzFamilyConesCovering}, and if $X_{j_l} = X$, then $\rho_n(X_i - X_{j_l}) = \rho_n(X_i - X)$ and $U_i < U_{j_l}$ must hold (the $U_i$ being the tiebreaking variables), in both cases $X_{j_l}$ is selected as being closer to $X_i$ in the $\rho_n$ distance. Hence there are at least $k$ points closer in the $\rho_n$ distance to $X_i$ than $X$. This means that if $X_i$ is not marked and $\norm{X_i - X} < a$, then $X$ is not a $k$-nearest neighbour of $X_i$ among $X_1, \dots, X_{i-1}, X, X_{i+1}, \dots, X_n$. Furthermore, the number of points that are marked is at most $c k$. It follows that
\begin{align*}
\E\left[\sum_{i=1}^n W_{ni}(X) f(X_i) \Indicator_{\{\norm{X_i - X} < \delta\}}\right] &\leq \frac1k \E\left[ f(X) \sum_{i=1}^n \Indicator_{\{X_i \text{ is marked}\}}\right] \\
&\leq \frac1k \E\left[ f(X) c k\right] \\
&\leq c \E\left[ f(X) \right] .
\end{align*}

For the second term, we see that since $f$ is nonnegative and bounded above by $1$ and by Lemma \ref{lemma:LipschitzFamilyAwayGoesToZero},
\begin{align*}
\E\left[\sum_{i=1}^n W_{ni}(X) f(X_i) \Indicator_{\{\norm{X_i - X} \geq \delta\}}\right] &\leq \E\left[\sum_{i=1}^n W_{ni}(X) \Indicator_{\left\{\norm{X_i - X} \geq \delta\right\}}\right] \\
&= \E\left[ \sum_{i=1}^k \frac1k \Indicator_{\left\{\norm{X_{(i, \rho_n)} - X} \geq \delta\right\}}\right] \\
&\leq \frac{1}{k} \sum_{i=1}^k \E\left[ \Indicator_{\left\{\norm{X_{(k, \rho_n)} - X} \geq \delta\right\}}\right] \\
&= \E\left[ \Indicator_{\left\{\norm{X_{(k, \rho_n)} - X} \geq \delta\right\}}\right] \\
&= \epsilon_n .
\end{align*}

Combining these results, we find that
\begin{equation*}
\E\left[\sum_{i=1}^n W_{ni}(X) f(X_i)\right] \leq c \E\left[ f(X) \right] + \epsilon_n .
\end{equation*}
\end{proof}

We are now able to prove our main theorem, that $k$-NN with a sequence of functions in $\F$ chosen independently of the sample and the query is universally consistent. We restate the conditions on $\F$ to make the statement of the theorem self-contained.
\begin{theorem}
\label{theorem:LipschitzFamilyKnnIsUniversallyConsistent}
Let $\F$ be a family of measurable functions from $\R^d$ to $\R$ such that there exist constants $\alpha \geq 0$, $\beta > 0$, $\gamma > 0$, and $r > 0$ so that for each function $\rho \in \F$,
\begin{enumerate}
\item The function $\rho$ is $\alpha$-Lipschitz on $B_r(\veczero, \norm{\cdot})$.
\item For any $\vecx \in B_r(\veczero, \norm{\cdot}) \setminus \{\veczero\}$, the derivative of $f(\lambda) = \rho(\lambda \vecx)$ is bounded below by $\beta \norm{\vecx}$ for all $\lambda \in (0, 1)$ (that is, $f'(\lambda) \geq \lambda \norm{\vecx}$).
\item Outside $B_r(\veczero, \norm{\cdot})$, the function is bounded below by $\gamma$, so if $\norm{\vecx} \geq r$, $\rho(\vecx) \geq \gamma$.
\item The function $\rho$ is symmetric, so $\rho(\vecx) = \rho(-\vecx)$ for all $\vecx \in \R^d$.
\item The function $\rho$ takes value zero at the zero vector, $\rho(\veczero) = 0$.
\end{enumerate}
Let ${\left(\rho_n\right)}_{n=1}^\infty$ be any sequence of random functions in $\F$, independent of the sample and query (with $n$ being the sample size). We have that $k$-NN with the sequence of random functions ${\left(\rho_n\right)}_{n=1}^\infty$ is universally consistent (where given points $\vecx, \vecy \in \R^d$, we take $\rho_n(\vecx, \vecy)$ to be the distance between these points for $k$-NN). Furthermore, if we have a family of functions $\mathcal{G}$ such that every function in $\mathcal{G}$ is the composition of a function in $\F$ with a strictly increasing function (that is, for any $g \in \mathcal{G}$, $g = h \circ f$, with $f \in \F$ and $h$ being strictly increasing), then $k$-NN with an sequence of random functions in $\mathcal{G}$ independent of the sample and the query is universally consistent.
\end{theorem}
\begin{proof}
We first notice that if we increase $\alpha$ to any larger value and decrease $\beta$ to any smaller value greater than 0 in the definition of $\F$, any function in $\F$ remains inside. Hence we may assume without loss of generality that $\alpha \geq 1$ and $\beta \in (0, 1]$.

Let ${(\rho_n)}_{n=1}^\infty$ be a sequence of functions in $\F$. We see by Lemma \ref{lemma:LipschitzFamilyStonesTheoremCondition1} that for any nonnegative function $f$ bounded above by one,
\begin{align*}
\E\left[ \sum_{i=1}^n W_{ni}(X) f(X_i) \right] &= \frac1k \E \left[ \sum_{i=1}^k f(X_{(i, \rho_n)}) \right] \\
&\leq \frac1k \left( k c \E\left[f(X) \right] + k \epsilon_n \right) \\
&= c \E\left[f(X) \right] + \epsilon_n
\end{align*}
and so the first condition of Stone's theorem is satisfied. By Lemma \ref{lemma:LipschitzFamilyAwayGoesToZero}, the second condition of Stone's theorem is satisfied. Since $k \to \infty$ as $n \to \infty$, $\frac1k \to 0$ as $n \to \infty$ and so the third condition of Stone's theorem holds. Hence, $k$-NN with any sequence of random functions in $\F$ (such that the sequence is independent of the sample and the query) is universally consistent, and so $k$-NN with a sequence of random functions in $\F$ (independent of the sample and query) is universally consistent. The choice of random function from the family can depend on the sample size $n$.

By Lemma \ref{lemma:KnnStrictlyIncreasingTransformation}, we have that since the transformations we apply after the distance are strictly increasing, the result of $k$-NN remains the same, so $k$-NN is universally consistent with a sequence of random functions (independent of the sample and the query) in $\mathcal{G}$ as well as $\F$.
\end{proof}

\begin{remark}
We notice that functions $\rho \in \mathcal{F}$ may take on the special value $\infty$ outside the ball $B_r(\veczero)$, which is larger than any finite value. The universal consistency proof holds in the same manner in this case.
\end{remark}

\section{Families of Lipschitz Distances}

In this section, we build families of Lipschitz distances that satisfy the conditions of Theorem \ref{theorem:LipschitzFamilyKnnIsUniversallyConsistent}. 

\begin{lemma}
\label{lemma:LipschitzFunctionDifferentiabilityCondition}
Let $f: \R^+ \to \R$ be a function and $\alpha > 0$, $r > 0$ be constants such that $f(0) = 0$, $f$ is continuous on $[0, r]$ and differentiable on $(0, r)$, with $f'(x) \leq \alpha$ for all $x \in (0, r)$. We then have that $f$ is $\alpha$-Lipschitz on $[0, r]$, and the function $g(x) = f(\left| x \right|)$ is $\alpha$-Lipschitz on $[-r, r]$.
\end{lemma}
\begin{proof}
This result follows from the intermediate value theorem.
\end{proof}

\begin{lemma}
\label{lemma:LipschitzFamilyTest}
Let $\F$ be a family of functions from $\R^d$ to $\R$ and $\alpha > 0$, $\beta > 0$, $\gamma > 0$, and $r > 0$ be constants such that each $\rho \in \F$, we have
\begin{equation}
\rho((x_1, x_2, \dots, x_d)) = \sum_{i=1}^d f_i(\left|x_i\right|)
\end{equation}
where each function $f_i: \R^+ \to \R^+$ is such that $f_i(x) = 0$, $\beta < f_i'(x) < \alpha$ for all $x \in (0, r)$, and $f_i(x) \geq \gamma$ for all $x \geq r$. Then $k$-NN with q sequence of random functions in $\F$ (independent of the sample and the query) is universally consistent.
\end{lemma}
\begin{proof}
We show that the family of distances $\F$ satisfies the conditions of Theorem \ref{theorem:LipschitzFamilyKnnIsUniversallyConsistent}. For the first condition, we let $\vecx = (x_1, x_2, \dots, x_d)$ and $\vecy = (y_1, y_2, \dots, y_d)$ be two points in $B_r(\veczero, \norm{\cdot})$, by expanding $\left| \rho(\vecx) - \rho(\vecy) \right|$ and applying Lemma \ref{lemma:LipschitzFunctionDifferentiabilityCondition} we find
\begin{align*}
\left| \rho(\vecx) - \rho(\vecy) \right| &= \left| \sum_{i=1}^d \left( f_i(\left|x_i\right|) - f_i(\left|y_i\right|)\right) \right| \\
&\leq \sum_{i=1}^d \left| \left( f_i(\left|x_i\right|) - f_i(\left|y_i\right|)\right) \right| \\
&\leq \sum_{i=1}^d \alpha \left| x_i - y_i \right| \\
&= \alpha \norm{\vecx - \vecy}_1 .
\end{align*}
Hence we have that every function in $\F$ is $\alpha$-Lipschitz in the $\norm{\vecx - \vecy}_1$ on $B_r(\veczero, \norm{\cdot})$.
For the second condition, we see that for all $\lambda \in (0, 1)$ and $\vecx \neq \veczero$ with $\norm{\vecx} \leq r$,
\begin{align*}
\frac{\partial}{\partial \lambda} \rho(\lambda \vecx) &= \frac{\partial}{\partial \lambda} \sum_{i=1}^d f_i(\left|\lambda x_i\right|) \\
&= \sum_{i=1}^d \left|x_i\right| {f_i'}(\left|\lambda x_i\right|) \\
&\geq \beta \sum_{i=1}^d \left|x_i\right| \\
&= \beta \norm{\vecx}_1 .
\end{align*}
For the third condition, we notice that for all $\vecx \in \R^d$ with $\norm{\vecx}_\infty \geq r$, the minimum $\min \{ |x_1|, |x_2|, \dots, |x_d| \} \geq r$, which implies $f_i(\left|\lambda x_i\right|) \geq \gamma$ for some $i \in \{1, 2, \dots, d\}$, and hence $\rho(\vecx) \geq \gamma$. The fourth condition follows directly from the fact we take the absolute value of each of the $x_i$, and the fifth condition holds since $\rho(\veczero) = \sum_{i=1}^d f_i(0) = 0$.
\end{proof}

\begin{corollary}
\label{corollary:LipschitzFamilyList1}
For each of the following functions, $k$-NN with any of the distances $\rho(\vecx) = \sum_{i=1}^{d} f(\left|x_i\right|)$ is universally consistent:
\begin{enumerate}
\item The exponential function $f_1(x) = e^x$.
\item The function $f_2(x) = \begin{dcases}
\sin(x) & \text{if } x \leq 1 \\
x & \text{if } x \geq 1
\end{dcases}$.
\item The function $f_3(x) = \begin{dcases}
\tan(x) & \text{if } x < \pi/2 \\
\infty & \text{if } x \geq \pi/2
\end{dcases}$.
\item The arctangent function $f_4(x) = \arctan(x)$.
\item The hyperbolic sine function $f_5(x) = \sinh(x)$.
\item The hyperbolic tangent function $f_6(x) = \tanh(x)$.
\end{enumerate}
\end{corollary}
\begin{proof}
We see that Lemma \ref{lemma:LipschitzFamilyTest} hold for the functions $f_2, f_3, f_4, f_5, f_6$ directly, and $e^x - 1$ satisfies the conditions of the lemma and is a strictly increasing transformation of $f_1$.
\end{proof}

We now show that Theorem \ref{theorem:KnnWithSandwichedNormIsUniversallyConsistent} follows from Theorem \ref{theorem:LipschitzFamilyKnnIsUniversallyConsistent}.
\begin{theorem}
\label{theorem:LipschitzFamilyNormsKnn}
Let $\F$ be a family of norms on $\R^d$ such that there exists a norm $\norm{\cdot}$ on $\R^d$ and a constant $C \geq 1$ such that for all $\rho \in \F$ and $\vecx \in \R^d$, $\frac1C \norm{\vecx} \leq \rho(\vecx) \leq C \norm{\vecx}$. We then have that $k$-NN with any sequence of random norms (independent of the sample and the query) in $\F$ is universally consistent.
\end{theorem}
\begin{proof}
We let $\vecx, \vecy \in \R^d$ be two points in $\R^d$. We see that
\begin{align*}
\left| \rho(\vecx) - \rho(\vecy) \right| &\leq \rho(\vecx - \vecy) \\
&\leq C \norm{\vecx - \vecy} .
\end{align*}

Hence we have that $\rho$ is Lipschitz with constant $C$ on $(\R^d, \norm{\cdot})$. Furthermore, we see that if we let $f(\lambda) = \rho(\lambda \vecx)$ for some $\vecx \neq 0$, for all $\lambda \in (0, 1)$,
\begin{align*}
f'(\lambda) &= \frac{\partial}{\partial \lambda} \rho(\lambda \vecx) \\
&\leq \frac{\partial}{\partial \lambda} \frac1C \norm{\lambda \vecx} \\
&= \frac{\partial}{\partial \lambda} \frac{\left| \lambda \right|}{C} \norm{\vecx} \\
&= \frac1C \norm{\vecx} .
\end{align*}

Additionally, we have that for all $\vecx$ such that $\norm{\vecx} \geq 1$, $\rho(\vecx) \geq \frac1C \norm{\vecx} \geq \frac1C$. We also see that $\rho(\vecx) = \rho(-\vecx)$ and $\rho(\veczero) = 0$ since $\rho$ is a norm. Hence we find that $\F$ satisfies the conditions of Theorem \ref{theorem:LipschitzFamilyKnnIsUniversallyConsistent} with $\alpha = C$, $\beta = \frac1C$, $\gamma = \frac1C$, and $r = 1$.
\end{proof}

Unfortunately, this result does not extend to quasinorms on $\R^d$, as they are not necessarily Lipschitz (or even locally Lipschitz near zero). In particular, the $\ell^p$ quasinorms with $p \in (0, 1)$ are not Lipschitz for any open ball around zero, with their partial derivatives in the $i^{\text{th}}$ coordinate being unbounded near the $i^{\text{th}}$ axis. To prove universal consistency for $k$-NN with quasinorms (which remains an open question), one must employ a different approach.

We can also consider families of polynomials.
\begin{theorem}
\label{theorem:LipschitzFamilyLinearPolynomialsKnn}
Let $\alpha, \beta > 0$ be fixed constants and $p \geq 1$ be an integer. Let $\F$ be the family of polynomials of the form
\begin{equation}
\F = \left\{ a_1 x + a_2 x^2 + \dots + a_p x^p \;\middle|\; a_1 \geq \beta;\; \forall m \geq 1, 0 \leq a_m \leq \alpha \right\} .
\end{equation}
That is, for any polynomial in $\F$ the first coefficient must be at least $\beta$ and all coefficients must be nonnegative and at most $\alpha$. Suppose we let $\mathcal{G}$ be the family of functions defined by applying $f$ to the modulus of the difference of each coordinate. Then $k$-NN with any sequence of random functions in $\mathcal{G}$ (independent of the sample and the query) is universally consistent.
\end{theorem}
\begin{proof}
We verify the conditions of Lemma \ref{lemma:LipschitzFamilyTest} for any function $f \in \F$. First, we find an upper bound for the derivative of $f$ on the interval $[0, 1]$,
\begin{align*}
f'(x) &= \sum_{m=1}^p a_m m x^{m-1} \\
&\leq \alpha \sum_{m=1}^p m \\
&= \alpha \frac{m (m+1)}{2} .
\end{align*}

We also find that $f'$ is bounded below by $\beta$ on $(0, \infty)$, since $f'(x) = \sum_{m=1}^p a_m m x^{m-1}$ with the first term being $\beta$ and all the other terms being nonnegative. This implies that $f$ is monotone increasing on $(0, \infty)$, so $f(x) \geq f(1)$ for all $x \geq 1$. Hence we have that $f$ satisfies the conditions of Lemma \ref{lemma:LipschitzFamilyTest} and so $k$-NN with a sequence from  the corresponding family of distances $\mathcal{G}$ is universally consistent.
\end{proof}

As with norms, we can take linear combinations of these functions and add them to our family.
\begin{lemma}
\label{lemma:LipschitzFamilyLinearCombinations}
Let $\F$ be a family of distances satisfying the conditions of Theorem \ref{theorem:LipschitzFamilyKnnIsUniversallyConsistent}. If we let $\mathcal{G}$ be the family of all functions $\rho: \R^d \to \R$ of the form (with $p \geq 1$ being a fixed constant)
\begin{equation}
\rho((x_1, x_2, \dots, x_d)) = \sum_{i=1}^p \frac{a_i}{A} \rho_i(\vecx)
\end{equation}
where $0 \leq a_i \leq 1$, at least one $a_i$ is strictly positive ($a_i > 0$), $A = \sum_{i=1}^p a_i$, and $\rho_i \in \F$. Then $k$-NN any sequence of functions in $\mathcal{G}$ (independent of the sample and the query) is universally consistent (in $\mathcal{G}$, we can also include strictly increasing functions of such linear combinations).
\end{lemma}
\begin{proof}
For any function $\rho \in \mathcal{G}$ of the above form, we see that $\rho$ is $\alpha p$-Lipschitz, since
\begin{align*}
\left|\rho(\vecx) - \rho(\vecy)\right| &= \left|\sum_{i=1}^p \frac{a_i}{A} (\rho_i(\vecx) - \rho_i(\vecy))\right| \\
&\leq \sum_{i=1}^p \frac{a_i}{A} \left|\rho_i(\vecx) - \rho_i(\vecy)\right| \\
&\leq \sum_{i=1}^p \left|\rho_i(\vecx) - \rho_i(\vecy)\right| \\
&= \alpha p \norm{\vecx - \vecy} .
\end{align*}

Furthermore, we see that the derivative is bounded below by $\beta p \norm{\vecx}$,
\begin{align*}
\frac{\partial}{\partial \lambda} \rho(\lambda \vecx) &= \frac{\partial}{\partial \lambda} \sum_{i=1}^p \frac{a_i}{A} (\rho_i(\lambda \vecx)) \\
&\geq \sum_{i=1}^p \frac{a_i}{A} \beta \norm{\vecx} \\
&= \beta p \norm{\vecx} .
\end{align*}

Since each $\rho_i$ is bounded below by $\gamma$ outside $B_r(\veczero)$ for some fixed $r > 0$, we see that for all $\vecx \in \R^d$ such that $\norm{\vecx} \geq r$,
\begin{align*}
\rho(\vecx) &= \sum_{i=1}^p \frac{a_i}{A} \rho_i(\vecx) \\
&\geq \sum_{i=1}^p \frac{a_i}{A} \gamma \\
&= \gamma .
\end{align*}

For the fourth and fifth conditions, we see that since each $\rho_i$ is symmetric and takes value zero at the zero vector, the same holds for $\rho$.

By Lemma \ref{lemma:KnnStrictlyIncreasingTransformation}, we can include strictly increasing transformations of such functions, and the same result will be generated, so universal consistency is preserved.
\end{proof}

\chapter{Adaptive $k$-NN}

\begin{figure}[h]
\centering
\includegraphics[scale=0.5]{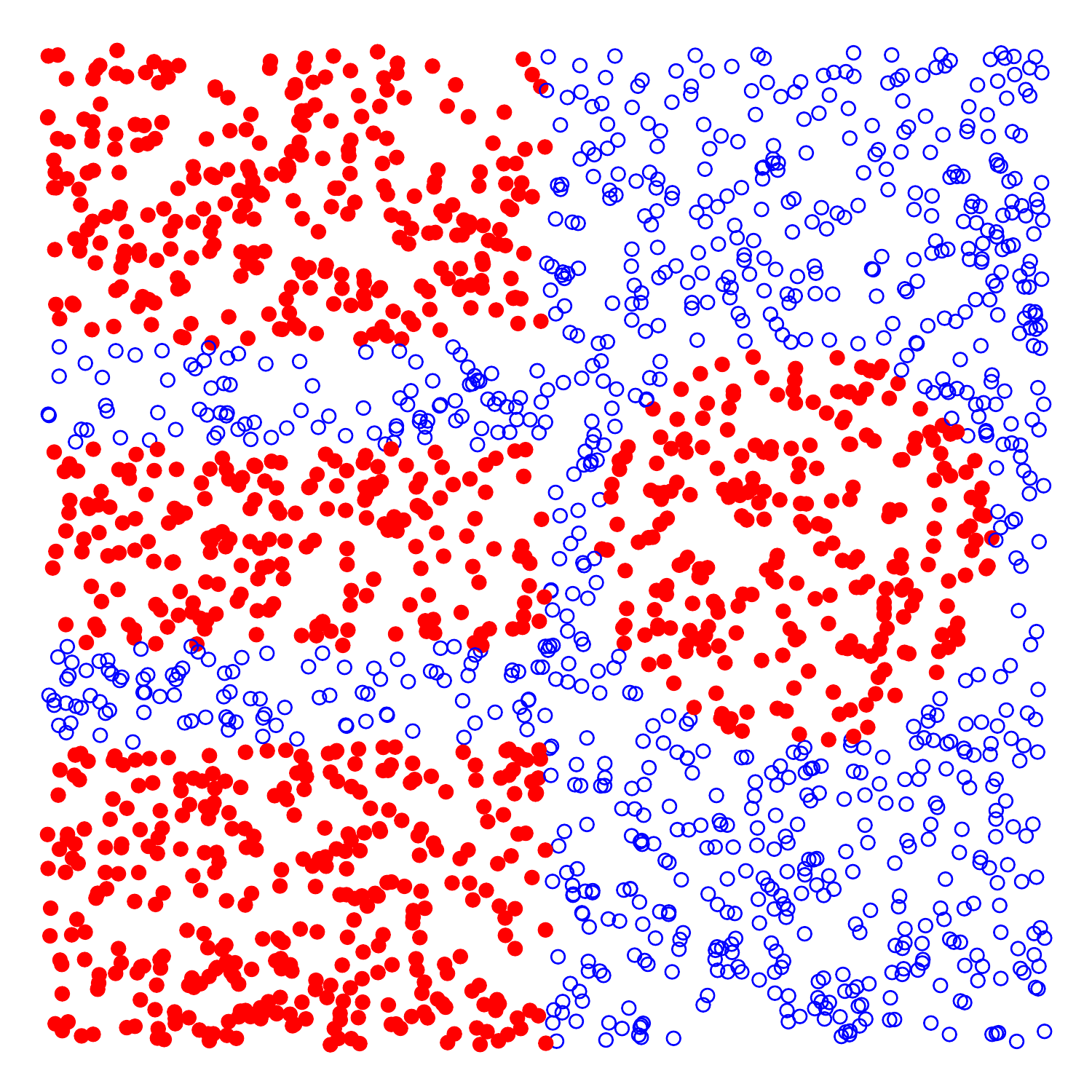}
\caption{We see we would like to use a different norm for $k$-NN on the left side and the right side, and possibly within the right side, for this dataset.}
\label{fig:AdaptiveKnn}
\end{figure}

In this chapter, we investigate under what conditions we can select the distance for $k$-NN based on the query and the sample and retain the universal consistency proof in Stone's theorem. We first look at a limitation of Stone's theorem, which prevents us from using the sample labels for a classifier we would like to prove is universal consistent using Stone's theorem. We then define a modified $k$-NN in which we can select the distance based on the query and on the sample points (but not the labels). We illustrate why such an adaptive procedure is useful in Figure \ref{fig:AdaptiveKnn}.

\section{Limitations of Stone's Theorem in Adaptive $k$-NN}

In Stone's theorem, we assumed that the weights $W_{ni}(X)$ are a functions of the query $X$, the sample points $X_1, X_2, \dots, X_n$, and an independent random variable $V$ only, and not on the sample labels $Y_1, Y_2, \dots, Y_n$. In the following example we show that this assumption is necessary.

\begin{example}
\label{example:StonesTheoremLabelCounterexample}
Suppose we have a joint distribution for $(X, Y) \in \R^d \times \{0, 1\}$ with $X$ having a multivariate normal distribution and $Y$ being an independent Bernoulli random variable with a probability of $1/3$ of being zero and $2/3$ of being one. We have an iid labelled sample $(X_1, Y_1), (X_2, Y_2), \dots, (X_n, Y_n)$ and a query $(X, Y)$ from this distribution. It is clear that the Bayes error for this distribution is $1/3$, and is attained by selecting label one always.

Suppose we let $k \to \infty$ and $k / n \to 0$ as $n \to \infty$. We define weights $W_{ni}(X)$ that are $1/k$ if $X_i$ is both a $5 k$ nearest neighbour of $X$ in the Euclidean norm and $X_i$ is one of the $k$ nearest points to $X$ among those $5 k$ points with label zero (if there are not enough such points, then we take points among the nearest $5k$ with label one until we have sufficiently many points), and zero otherwise. We see that the weights $W_{ni}(X)$ depend on the labels $Y_1, Y_2, \dots, Y_n$ and that they can only be nonzero for points that among the $5 k$ nearest neighbours of $X$ in the Euclidean norm. By the strong law of large numbers the fraction of the nearest $5 k$ points with label zero approaches $1 / 3$ as $n \to \infty$. Hence with probability approaching one there will be at least $k$ points among the nearest $5 k$ with label zero, thus the weights will be nonzero only for points with label zero, and so the point will be classified as zero by the weight based classifier (of the form in equation \eqref{eq:WeightedClassifier}). It follows that the expected misclassification error is $2/3$, which is much higher than the Bayes error (indeed, it is worse than randomly guessing).

However, we see that $W_{ni}(X)$ satisfies the three conditions of Stone's theorem (the only assumption violated is the initial one that $W_{ni}(X)$ does not depend on the sample labels):
\begin{enumerate}
\item For the first condition, from Lemma \ref{lemma:StonesLemmaCore} we find that
\begin{align*}
&\phantom{{}={}}\E \left[ \sum_{i=1}^{n} W_{ni}(X) f(X_i) \right] \\
&\leq \E \left[ \sum_{i=1}^{n} \frac1k \Indicator_{\{ X_i \text{ is a } \norm{\cdot} \text{ } 5 k \text{-NN of } X \text{ among } X_1, \dots, X_{i-1}, X_i, X_{i+1}, \dots, X_n \}} f(X_i) \right] \\
&= \frac1k \E \left[ \sum_{i=1}^{n} \Indicator_{\{ X \text{ is a } \norm{\cdot} \text{ } 5 k \text{-NN of } X_i \text{ among } X_1, \dots, X_{i-1}, X, X_{i+1}, \dots, X_n \}} f(X) \right] \\
&\leq 5 c \E[f(X)] .
\end{align*}

\item For the second condition, we see that $\sum_{i=1}^n W_{ni}(X) \Indicator_{\{\norm{X_i - X} > a \}}$ is bounded above by one, is nonnegative, and is nonzero if and only if the $5 k^{\text{th}}$ nearest point to $X$ has a distance of at most $a$. By Lemma \ref{lemma:ProbKPointIsFarGoesToZero}, the probability of this goes to zero as $n$ goes to infinity (since $5k / n \to 0$ as $n \to \infty$), and hence the expected value $\E\left[ \sum_{i=1}^n W_{ni}(X) \Indicator_{\{\norm{X_i - X} > a \}} \right] \to 0$ as $n \to \infty$.

\item The third condition follows from the fact that $k \to \infty$ as $n \to \infty$, so that $\max_{1 \leq i \leq n} W_{ni}(X) = 1 / k \to 0$ as $n \to \infty$.
\end{enumerate}

This example satisfies the three conditions of Stone's theorem, but is not universally consistent, the only assumption violated is that the weights $W_{ni}(X)$ depends on the sample labels $Y_1, Y_2, \dots, Y_n$. Hence this requirement in Stone's theorem is essential and cannot be removed. If a learning rule depends on the sample labels $Y_1, Y_2, \dots, Y_n$ for the weights $W_{ni}(X)$, we cannot use Stone's theorem (at least without modifications for that specific learning rule) to prove it is universally consistent.
\end{example}

\section{Consistency of $k$-NN with an Adaptively Chosen Sequence of Distances}

In this section we investigate the conditions under which we can select the distance for $k$-NN, depending on the query $X$ and the sample points $X_1, X_2, \dots, X_n$.
\begin{theorem}
\label{theorem:AdaptiveKnn}
Let $\norm{\cdot}$ be a norm on $\R^d$ and let $X_{(1, \norm{\cdot})}, X_{(2, \norm{\cdot})}, \dots, X_{(n, \norm{\cdot})}$ be the points in the sample in order of distance from $X$. If $m \geq 1$ is a constant and the weight function $W_{ni}(X)$ is a function of the query $X$, the sample points $X_1, X_2, \dots, X_n$, and an independent random variable $V$, has support that is a subset of $X_{(1, \norm{\cdot})}, X_{(2, \norm{\cdot})}, \dots, X_{(mk, \norm{\cdot})}$ (that is, it is nonzero only on the nearest $mk$ points to $X$ in the $\norm{\cdot}$ norm) and $\E\left[\max_{1 \leq i \leq n} W_{ni}(X)\right] \to 0$, then the weight based classifier (of the form in equation \eqref{eq:WeightedClassifier}) is universally consistent.
\end{theorem}
\begin{proof}
We see that the weights $W_{ni}$ depend only on the query $X$, the sample points $X_1, X_2, \dots, X_n$, and an independent random variable $V$, so we can apply Stone's theorem. We check the conditions of Stone's theorem:
\begin{enumerate}
\item For the first condition, we define $c$ cones and mark the $m k$ nearest neighbours of the query $X$ in each cone, as in Lemma \ref{lemma:StonesLemmaCore}. By the argument in Lemma \ref{lemma:StonesLemmaCore} (replacing the $k$ nearest neighbour by $mk$ nearest neighbour) we find that
\begin{align*}
&\phantom{{}={}}\sum_{i=1}^n \E \left[ W_{ni}(X) f(X_i) \right] \\
&= \E \left[ \sum_{i=1}^n \frac1k \Indicator_{\{ X_i \text{ is a } m k \text{-nearest neighbour of } X \text{ in the } \norm{\cdot} \text{ norm among } X_1, X_2, \dots, X_n \}} f(X_i) \right] \\
&= \frac1k \E \left[ \sum_{i=1}^n \Indicator_{\{ X \text{ is a } m k \text{-nearest neighbour of } X_i \text{ in the } \norm{\cdot} \text{ norm among } X_1, \dots, X_{i-1}, X, X_{i+1}, \dots, X_n \}} f(X) \right] \\
&= \frac1k \E \left[ f(X) \sum_{i=1}^n \Indicator_{\{ X \text{ is a } m k \text{-nearest neighbour of } X_i \text{ in the } \norm{\cdot} \text{ norm among } X_1, \dots, X_{i-1}, X, X_{i+1}, \dots, X_n \}} \right] \\
&\leq \frac{1}{k} \E \left[ f(X) \sum_{i=1}^n \Indicator_{\{ X_i \text{ is marked} \}} \right] \\
&\leq \frac{m}{k} \E \left[ f(X) (m k) c \right] \\
&= c m \E\left[f(X)\right] .
\end{align*}
\item We see that for any fixed $a > 0$, $\Prob(\norm{X - X_{(m k, \norm{\cdot})}} > a) \to 0$ as $n \to \infty$ by Lemma \ref{lemma:ProbKPointIsFarGoesToZero}, and hence the second condition follows (as in Lemma \ref{lemma:StonesTheoremSecondConditionNorm}, with $k$ replaced by $m k$, since $k / n \to 0$ as $n \to \infty$, $m k / n \to 0$ as $n \to \infty$).
\item The third condition holds by assumption.
\end{enumerate}
\end{proof}

Theorem \ref{theorem:AdaptiveKnn} remains true if we replace the fixed norm $\norm{\cdot}$ with a sequence ${\left(\rho_n\right)}_{n=1}^\infty$ of either random norms on $\R^d$ from a bounded family (from Theorem \ref{theorem:KnnWithSandwichedNormIsUniversallyConsistent}) or of random uniformly locally Lipschitz distances (from Theorem \ref{theorem:LipschitzFamilyKnnIsUniversallyConsistent}) such that the sequence is chosen independently of the query and of the sample. We now prove this for the most general case we have seen (a sequence of random uniformly locally Lipschitz distances, independent of the sample and the query).
\begin{theorem}
\label{theorem:AdaptiveKnn2}
Let ${\left(\theta_n\right)}_{n=1}^\infty$ be a sequence of random functions from a family of functions $\R^d$ to $\R$ satisfying the conditions of Theorem \ref{theorem:LipschitzFamilyKnnIsUniversallyConsistent} (independent of the sample and the query) and let $X_{(1, \norm{\cdot})}, X_{(2, \norm{\cdot})}, \dots, X_{(n, \norm{\cdot})}$ be defined as usual. If $m \geq 1$ is a constant and the weight function $W_{ni}(X)$ is a function of the query $X$, the sample points $X_1, X_2, \dots, X_n$, and an independent random variable $V$, has support that is a subset of $X_{(1, \theta_n)}, X_{(2, \theta_n)}, \dots, X_{(mk, \theta_{mk})}$ (that is, it is nonzero only on the nearest $mk$ points to $X$ in the $\theta_n$ distance) and $\E\left[\max_{1 \leq i \leq n} W_{ni}(X)\right] \to 0$, then the weight based classifier (of the form in equation \eqref{eq:WeightedClassifier}) is universally consistent.
\end{theorem}
\begin{proof}
We see that the weights $W_{ni}$ depend only on the query $X$, the sample points $X_1, X_2, \dots, X_n$, and an independent random variable $V$, so we can apply Stone's theorem. We check the conditions of Stone's theorem:
\begin{enumerate}
\item For the first condition, we proceed as in \ref{lemma:LipschitzFamilyStonesTheoremCondition1}, we define $c$ subsets in the same way as in the proof of the lemma and mark the $m k$ nearest neighbours of the query $X$ in each subsets. In the same manner as the proof of the lemma (replacing $k$ with $mk$, which does not change anything else because $m k / n \to 0$ as $n \to \infty$ since $k / n \to 0$ as $n \to \infty$ and $m$ is a fixed constant), we have
\begin{equation*}
\E\left[\sum_{i=1}^n W_{ni}(X) f(X_i)\right] \leq c \E\left[ f(X) \right] + \epsilon_n .
\end{equation*}
\item For the second condition, we have
\begin{equation*}
\E \left[ \sum_{i=1}^n W_{ni} (X) \mathds{1}_{\{\norm{X_i - X} > a\}} \right] \leq \Prob\left( \norm{X_{\left(k, \norm{\cdot}\right)} - X} > \frac{\beta}{2 \alpha} \min\{a, \delta\}\right) .
\end{equation*}
by Lemma \ref{lemma:LipschitzFamilyKnnDistanceInequality}, and since $m k / n \to 0$ as $n \to \infty$, by Lemma \ref{lemma:ProbKPointIsFarGoesToZero} the probability on the right hand side goes to zero as n approaches infinity, so the expectation on the left hand side (which by definition is nonnegative) goes to zero.
\item The third condition holds by assumption.
\end{enumerate}
\end{proof}

Suppose we take either a fixed norm or an independent sequence of Lipschitz distance for $k$-NN, and we take the $mk$ nearest points to the query at each step, with $m \geq 1$ a fixed constant. The above results allow us to adaptively pick a distance for $k$-NN based on the sample points (but not the sample labels), the query, and an independent random variable, with those $mk$ nearest points to the query (that is, we only consider those $mk$ points in $k$-NN and ignore the rest). By Theorem \ref{theorem:AdaptiveKnn2} (or Theorem \ref{theorem:AdaptiveKnn} for a fixed norm) we have that this results in a universally consistent classifier, since $W_{ni}(X)$ is nonzero only on the $mk$ nearest neighbours in the $\theta_n$ distance (or $\norm{\cdot}$ distance) and the maximum of the weights is $1/k$ which goes to zero as $n \to \infty$ (since $k \to \infty$ as $n \to \infty$).

To generate random variables independent of the sample and the query that are useful for us, we can independently randomly split the original sample into two smaller samples (with a fixed proportion of the points going into each sample). One of the subsamples becomes the sample used directly for classification (the weights can be nonzero for these points), and the other subsample becomes a set of points independent from the sample and the query which we can use for selecting the distance for $k$-NN. We then have points with their label which we can use in an optimization procedure for the distance while preserving universal consistency (since they are independent of the sample used for classification and the query, they become part of the independent random variable in the weight function). For instance, if we have an original sample of $(X_1, Y_1), (X_2, Y_2), \dots, (X_{2n}, Y_{2n})$ containing $2n$ points, we can split it into a sample of size $n$ for the $k$-NN classifier and another disjoint set of points of size $n$ which we use to find the distance for $k$-NN. For the global $\theta_n$ distance to find the $mk$ nearest points, we use just those $n$ points, for the local distance for the $k$ nearest of those points, we can also use the query and sample points (but not labels) as well as those independent points.

\chapter{Datasets and Experimental Results}

In this chapter we discuss classifying datasets based on the techniques described above. We consider both using a fixed distance for $k$-NN for the entire dataset and \emph{locally chosen} distances, where we select the distance for $k$-NN based on the query and the sample. We first discuss some optimization methods used for selecting good distances for $k$-NN and a method for assessing the accuracy of our classification (where we run many trials for an accuracy and stable estimate). We then give examples showing classification accuracy improvements (compared to $k$-NN with the Euclidean norm) for a variety of datasets.

\section{Methodology}

We would like to find parameters (such as the $p$ for the $\ell^p$ norm, the entries of a matrix, etc.) to achieve as high of a classification accuracy as possible. One approach is to try to maximize the classification accuracy on the training set. This has the disadvantage that the classification accuracy (for an empirical sample) is a step function, that is constant with sudden jumps. Optimization of such functions (especially when there are many parameters) can be very difficult. Optimizing the classification accuracy for locally chosen distances on the training set produced poor results on the testing set for our datasets.

Another approach is to look at the correlation between distance and whether or not the points have the same label. We recall that the \emph{Pearson's correlation coefficient} (often denoted $\rho_{X, Y}$) between two random variables $X$, $Y$ is defined by
\begin{equation}
\Corr(X, Y) = \rho_{X, Y} = \frac{\Cov(X, Y)}{\sigma_X \sigma_Y} .
\end{equation}
It can be shown that $-1 \leq \Corr(X, Y) \leq 1$ always, with a correlation of 1 denoting a perfect positive linear relationship and a correlation of $-1$ denoting a perfect negative linear relationship between $X$ and $Y$.\cite{CasellaBerger} Suppose we select a point in the training set, which we call a training query. If we let $X$ be the distance between the training query and other points in the dataset and $Y$ be whether or not the label of the training query agrees with the label of the other point, a negative correlation between $X$ and $Y$ indicates that as we move away from the training query, we are more likely to find points with a different label, and if we move towards the training query we are more likely to find points with the same label. We then attempt to find a distance such that this correlation is as close to $-1$ as possible. To optimize over a family of distances to find a good locally chosen distance (for a query in the testing set), we can proceed as follows: we first take the $k_1$ and $k_2$ nearest neighbours of the query (in a fixed norm), with $k_2 \geq k_1$; we then take each of the $k_1$ nearest neighbours as a training query and find the correlation with the $k_2$ nearest neighbours described above, we minimize the mean correlation for each of the points in $k_1$ as a training query, the parameters we found are the ones used for the locally chosen distance for this testing set query. In this approach, we only consider the $k_2$ nearest points to the query in the correlation, and so only points in that neighbourhood affect the locally chosen distance. The correlation has the advantage that for our family of distances it is continuous (and in some cases differentiable), so optimization is far easier.

For optimizing our parameters, we use the R function \texttt{optim} to do the optimization, which implements a variety of optimization methods. We describe some of these below. The problems are traditionally formulated as minimization problems, we can easily convert them into to a maximization problems by multiplying the function by $-1$ (in R, this can be done with the \texttt{fnscale=-1} option). The following brief descriptions are based on the book \cite{NR} and the R documentation \cite{R}.

The default method is the \emph{Nelder-Mead method} of optimization (which is also known as the \emph{Downhill Simplex method} or the \emph{amoeba method}). It is well suited to optimization problems in multiple dimensions and uses only function values, without assuming the existence of derivatives. A \emph{simplex} in $\R^d$ is a $d$-dimensional generalization of a tetrahedron (a triangle in $\R^2$, a tetrahedron in $\R^3$, and so on). In the Nelder-Mead method, we move the simplex using expansion, reflection, reflection and expansion, contraction, or multiple contraction at each step as necessary towards the minimum. The simplex thereby moves downhill towards a minimum. This is a robust method, but is relatively slow.

The \emph{Broyden-Fletcher-Goldfarb-Shanno} (BFGS) method uses function values and gradients to build a picture of the surface to optimize. BFGS is a quasi-Newton method that computes an approximation to the Hessian matrix, which is updated at each step.

The \emph{Conjugate Gradient} (CG) method (originally developed by Fletcher and Reeves (1964), there are options available to use modified versions) also uses gradients. The CG method does not compute a Hessian matrix approximation and so is better suited for large optimization problems, however, it tends to be more fragile than BFGS.

In the \emph{Simulated Annealing} method, we try to find the global minimum of the function. Starting with some point, we generate new points around the current point randomly, based on a temperature parameter (using a Gaussian Markov kernel by default, in the R \texttt{optim} implementation). We sometimes accept points worse than the existing point, to avoid getting stuck in a local minimum (how often we do so depends on the temperature). We then slowly decrease the temperature with time. This method uses only function values and works well for functions with noisy surfaces, but is relatively slow and is sensitive to the control parameters (like temperature) passed to it.

For each dataset described below, we find the accuracy by splitting the dataset into training and testing sets and finding the accuracy on the testing set. We repeat this cross-validation procedure multiple times, the exact details are described separately for each dataset. We have found (using Q-Q plots and the Shapiro-Wilk normality tests) that the empirical classification accuracies appear to have an approximate normal distribution. This means that we can find the $1 - \alpha$ confidence interval for the mean accuracy using the formula (based on the Student's $t$ distribution)
\begin{equation}
\overline{x} \pm t_{\alpha / 2}(n - 1) \frac{s}{\sqrt{n}}
\end{equation}
where $\overline{x}$ is the sample accuracy, $n$ is the sample size, and $s$ is the standard deviation of the observed data (using Bessel's correction of using the denominator $n-1$ instead of $n$ to yield an unbiased estimator). \cite{HoggTanis}

When we apply $k$-NN, we determine the value of $k$ by an empirical optimization procedure, in which we test all values of $k$ up to a certain threshold (for instance, all $k$ between 1 and 20) and use the value of $k$ that produces the best accuracy. We apply this empirical optimization procedure for $k$ on the training set (that is, we split the training set, and find the best value of $k$ on this set), and we use the optimal $k$ we found as the value of $k$ for $k$-NN in the actual classification on the testing set. For the locally chosen distances, we select points near the point we found (near in the original distance) and find the optimal value of $k$ for these points, which we use as the value of $k$ for classifying the query with the locally chosen distance.

\section{Experimental Results}

\subsection{Computer Generated Polynomials Dataset}

We first look at the classification accuracy for a computer generated dataset we created. Each data point is a random vector in $\R^7$ that is uniformly distributed on $[0, 1]$ in each coordinate. For a data point $\vecx = (x_1, x_2, \dots, x_7)$, we define $t = 2 x_1$, and we evaluate the polynomials $p_1(t) = x_2 t + x_3 t^2 + x_4 t^3$ and $p_2(t) = x_5 t + x_6 t^2 + x_7 t^3$. If $p_1(t) > p_2(t)$, we assign label one, otherwise we assign label zero. We generate 50 such datasets, each generated independently and containing 500 points in the training set and another 500 points in the testing set. We then find the classification accuracy with $k$-NN in each case, trying various $\ell^p$ norms/quasinorms and locally chosen distances (in which we first multiply the data by a matrix and then apply either an $\ell^p$ norm or a polynomial as our distance, with these parameters being determined locally based on the labelled sample and the query with the above procedure). A table of the results we obtained is in Table \ref{table:TableComputerPolynomial}, a box-and-whiskers plot is in Figure \ref{fig:BoxWhiskersComputerPolynomial}, and a plot of the 95 \% confidence intervals is in \ref{fig:LinePlotComputerPolynomial}. From this we see that the locally chosen $\ell^p$ distance with matrix produces the best results, followed by the fixed $\ell^p$ norms (with the accuracy being better for $p$ significantly larger than one, and being roughly constant for such $p$). The $\ell^p$ quasinorms (with $0 < p < 1$) perform poorly on this dataset.

\begin{figure}
\centering
\includegraphics[scale=0.75]{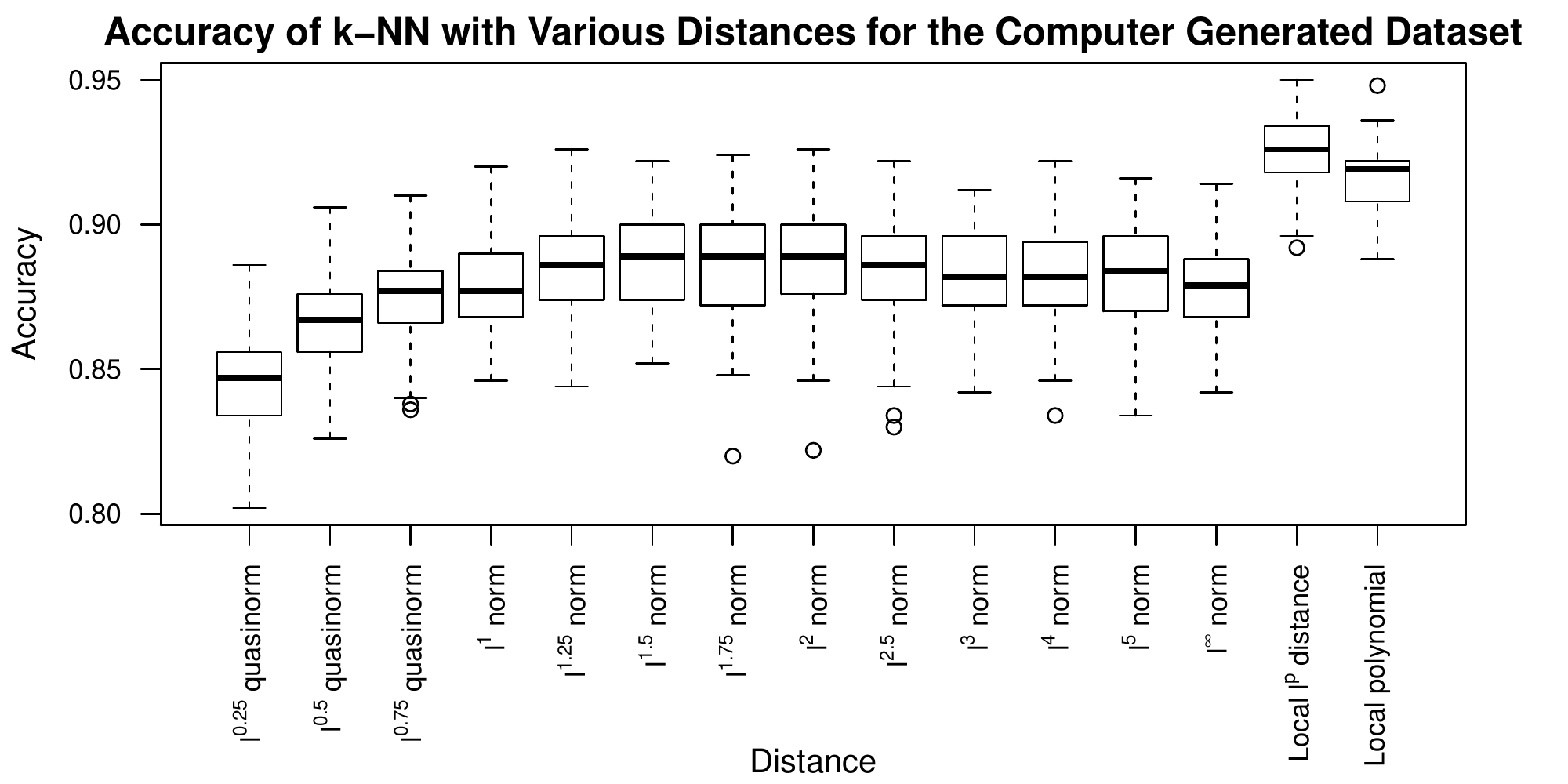}
\caption{Box-and-whiskers plot of the accuracy of $k$-NN with various $\ell^p$ norms and locally chosen distances for the computer generated dataset.}
\label{fig:BoxWhiskersComputerPolynomial}
\end{figure}

\begin{figure}
\centering
\includegraphics[scale=0.5]{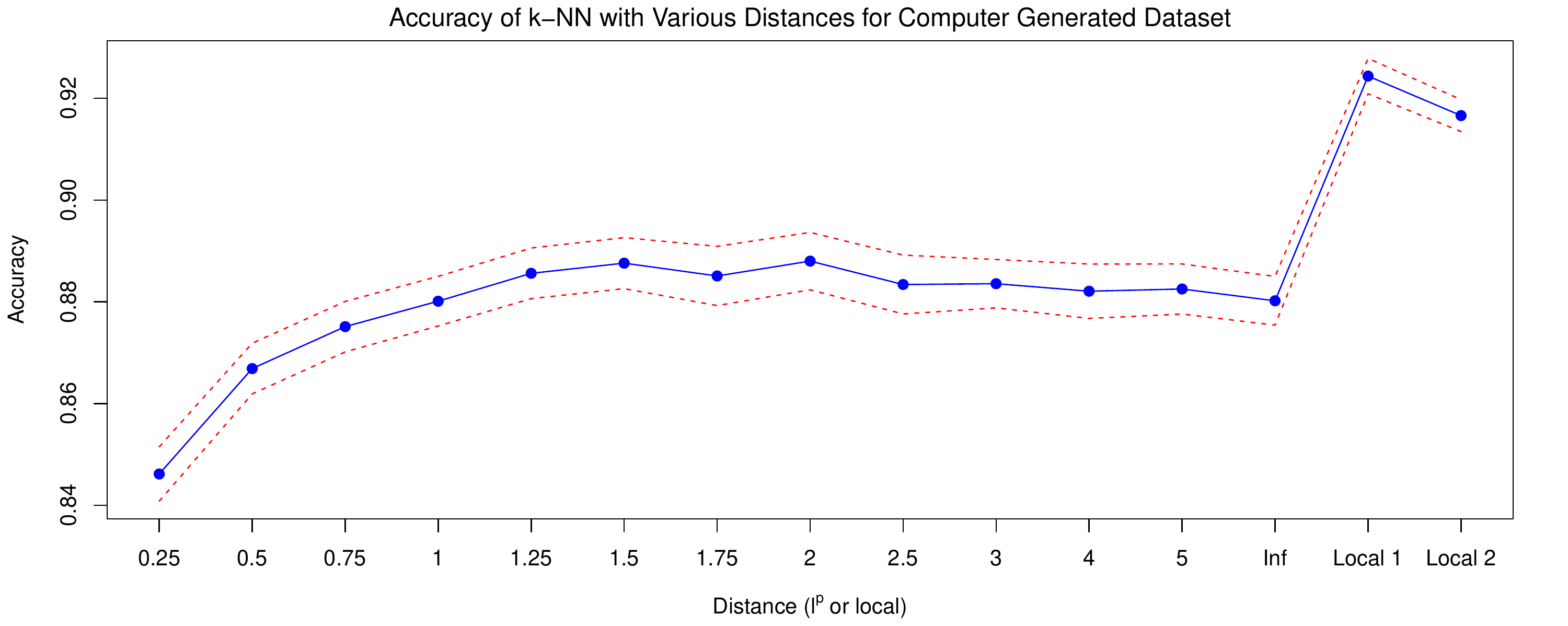}
\caption{Plot of the accuracy of $k$-NN with various $\ell^p$ norms, quasinorms, and locally chosen distances for computer generated dataset (showing 95 \% confidence intervals around the mean result, and data points). Here local 1 is locally chosen $\ell^p$ distance with matrix, local 2 is locally chosen polynomial with matrix.}
\label{fig:LinePlotComputerPolynomial}
\end{figure}

\begin{table}
\centering
\begin{tabular}{|c|c|c|}
\hline 
Distance & Mean Accuracy & 95\% Confidence Interval \\ 
\hline 
$\ell^{0.25}$ quasinorm & $0.84616$ & $[0.8408282, 0.8514918]$ \\ 
\hline 
$\ell^{0.5}$ quasinorm & $0.86688$ & $[0.8619008, 0.8718592]$ \\ 
\hline 
$\ell^{0.75}$ quasinorm & $0.87512$ & $[0.870166, 0.880074]$ \\ 
\hline 
$\ell^1$ norm & $0.88012$ & $[0.8752407, 0.8849993]$ \\ 
\hline 
$\ell^{1.25}$ norm & $0.8856$ & $[0.8806249, 0.8905751]$ \\ 
\hline 
$\ell^{1.5}$ norm & $0.8876$ & $[0.8825867, 0.8926133]$ \\ 
\hline 
$\ell^{1.75}$ norm & $0.88508$ & $[0.8792594, 0.8909006]$ \\ 
\hline 
$\ell^2$ norm & $0.888$ & $[0.882358, 0.893642]$ \\ 
\hline 
$\ell^{2.5}$ norm & $0.8834$ & $[0.877618, 0.889182]$ \\ 
\hline 
$\ell^3$ norm & $0.88356$ & $[0.8788055, 0.8883145]$ \\ 
\hline 
$\ell^4$ norm & $0.88208$ & $[0.8767357, 0.8874243]$ \\ 
\hline 
$\ell^5$ norm & $0.88252$ & $[0.8775985, 0.8874415]$ \\ 
\hline 
$\ell^\infty$ norm & $0.8802$ & $[0.8754124, 0.8849876]$ \\ 
\hline 
{\scriptsize Local Distance with $\ell^p$ norm and matrix} & $0.92436$ & $[0.9208851, 0.9278349]$ \\ 
\hline 
{\scriptsize Local Distance with degree 5 polynomial and matrix} & $0.9166$ & $[0.9134567, 0.9197433]$ \\ 
\hline 
\end{tabular}
\caption{The mean accuracy and confidence intervals for $k$-NN applied to the computer generated dataset with various $\ell^p$ norms and local distances.}
\label{table:TableComputerPolynomial}
\end{table}

\subsection{Face Recognition Dataset}

We compare various norms on the CDMC2013 face recognition task (from \cite{CDMC2013Faces}). We have a dataset consisting of 864 image vectors with 2576 dimensions, with 216 classes (with each class repeated exactly four times in the dataset). We use stratified sampling where for each class we select one image vector to be in our testing set and the other three to be in our training set. We first apply Principal Component Analysis (PCA)\cite{FoundationsOfMachineLearning}, reduce the dimension to 864, and calculate median centroids (that is, we calculate the median of each feature for all image vectors of the same class in the training set). Following this, we apply $k$-NN with $k=1$ to the median centroids. We find that the $\ell^{1.25}$ norm is optimal on this dataset, with norms near the $\ell^{1.25}$ norm having similar performance and norms further away having worse performance.

When we apply $k$-NN with various $\ell^p$ norms to this dataset, we obtain the results in the box-and-whiskers plot \ref{fig:BoxWhiskersLpFace}. We have also tested various other norms which where found to perform far worse, in particular the quasinorms with $p < 1$ have been found to perform very poorly. In Table \ref{table:TableLpFace}, we give some mean accuracies for various $\ell^p$ norms for this dataset (the confidence intervals are from sampling training and testing sets from this particular dataset, with 4 possibilities for each class to be chosen for the training set, this should not be interpreted as 99 \% confidence intervals for the actual misclassification error for the underlying distribution for the data). Figure \ref{fig:BoxWhiskersLpFace} gives a box-and-whiskers plot of the accuracies for many $\ell^p$ norms with $1 \leq p \leq 2$. The results for $p$ outside this range were poor, with performance dropping off for $p < 1$ or $p > 2$. We can see from this that the $\ell^{1.25}$ norm (and $\ell^p$ norms with $p$ near $1.25$) gives the best result here.

\begin{figure}
\centering
\includegraphics[scale=.5]{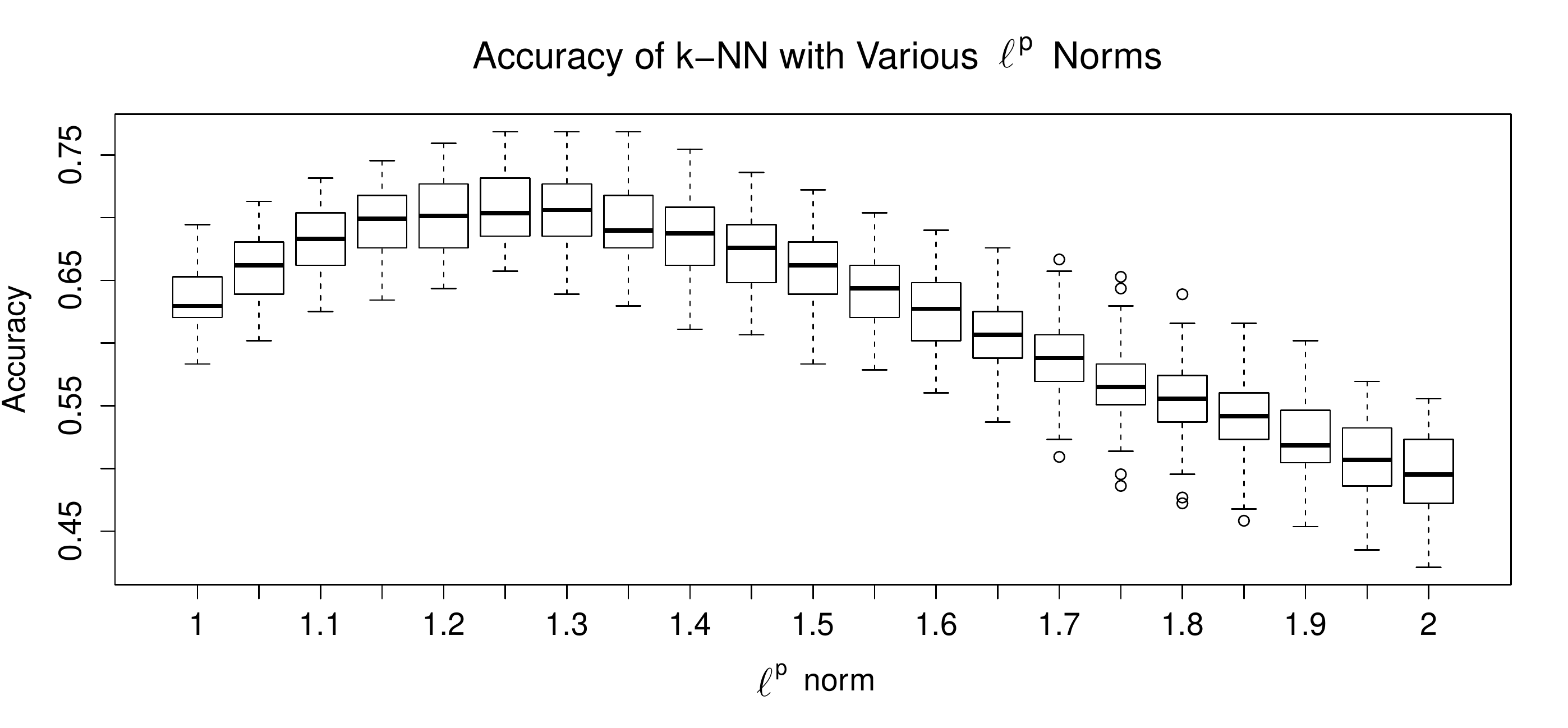}
\caption{Box-and-whiskers plot of the accuracy of $k$-NN with various $\ell^p$ norms for the Face Recognition dataset.}
\label{fig:BoxWhiskersLpFace}
\end{figure}

\begin{table}
\begin{tabular}{|c|c|c|}
\hline 
$\ell^p$ norm & Mean Accuracy & 99\% Confidence Interval for this Sample Dataset \\ 
\hline 
1 & $0.6341667$ & $[0.6246978, 0.6436356]$ \\ 
\hline 
1.25 & $0.7075926$ & $[0.6967336, 0.7184516]$ \\ 
\hline 
1.5 & $0.6585185$ & $[0.6469469, 0.6700901]$ \\ 
\hline 
1.75 & $0.5660185$ & $[0.5534964, 0.5785407]$ \\ 
\hline 
2 & $0.4932407$ & $[0.4808274, 0.5056541]$ \\ 
\hline 
\end{tabular}
\caption{The mean accuracy and confidence intervals for $k$-NN applied to the Face Recognition dataset with various $\ell^p$ norms.}
\label{table:TableLpFace}
\end{table}

\subsection{Forest Cover Dataset}

We now look at a subset of a forest cover dataset from \cite{ForestCover}, which was the subject of a Kaggle competition. We do not include the categorical data from the original dataset, only the numerical part (hence our results below are not comparable to the Kaggle competition, in practice we would use this in conjunction with a classifier for the categorical part as the categorical part is very important to achieve a high classification accuracy, classification accuracies much higher than ours are possible using the categorical part alone). There are 10 numerical columns, we fit all of them to the interval $[0, 1]$.

The original dataset contains 15120 rows. We take 50 random samples, each of them containing 2000 rows for the training set and 1000 rows for the testing set (with the training and testing sets disjoint in each case). We then find the classification accuracy with $k$-NN in each case, trying various $\ell^p$ norms/quasinorms and locally chosen distances (for the locally chosen distances, we first multiply the data by a matrix and then apply either an $\ell^p$ norm or a polynomial as our distance, with these parameters being determined locally based on the labelled sample and the query with the above procedure). In particular, we have tried the $\ell^p$ norms and quasinorms with $p$ being $0.25, 0.5, 0.75, 1, 1.25, 1.5, 1.75, 2, 2.5, 3, 4, 5, \infty$, and a couple of locally chosen distances, one consisting of optimizing over matrices and $\ell^p$ norms, the other consisting of optimizing over matrices and polynomials of degree 5. The mean accuracies obtained (with 95 \% confidence intervals) are given in Table \ref{table:TableForestCover}, a box-and-whiskers plot is shown in Figure \ref{fig:BoxWhiskersForestCover}, and a plot of the accuracies with 95 \% confidence intervals is in Figure \ref{fig:LinePlotForestCover}. We see that both locally chosen distances deliver superior performance to any fixed norm that was tested, with local $\ell^p$ norms with matrices being better than local polynomials with matrices. For fixed $\ell^p$ norms/quasinorms, the accuracy seems to be highest around 0.5 and 0.75, with smaller $p$ performing much worse and the accuracy dropping off as $p$ increases beyond 0.75.

\begin{figure}
\centering
\includegraphics[scale=0.8]{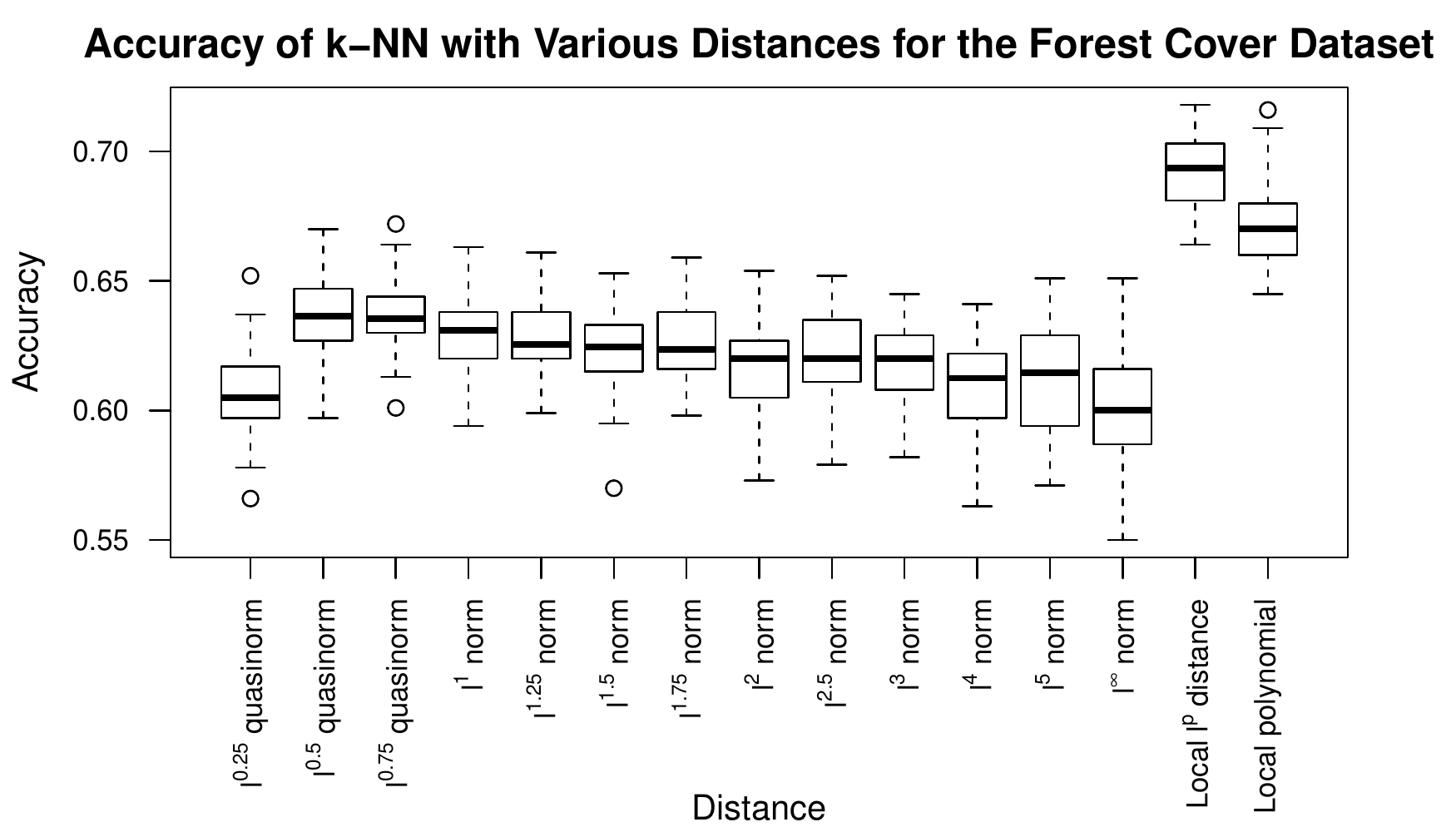}
\caption{Box-and-whiskers plot of the accuracy of $k$-NN with various $\ell^p$ norms and locally chosen distances for the Forest Cover dataset.}
\label{fig:BoxWhiskersForestCover}
\end{figure}

\begin{figure}
\centering
\includegraphics[scale=.5]{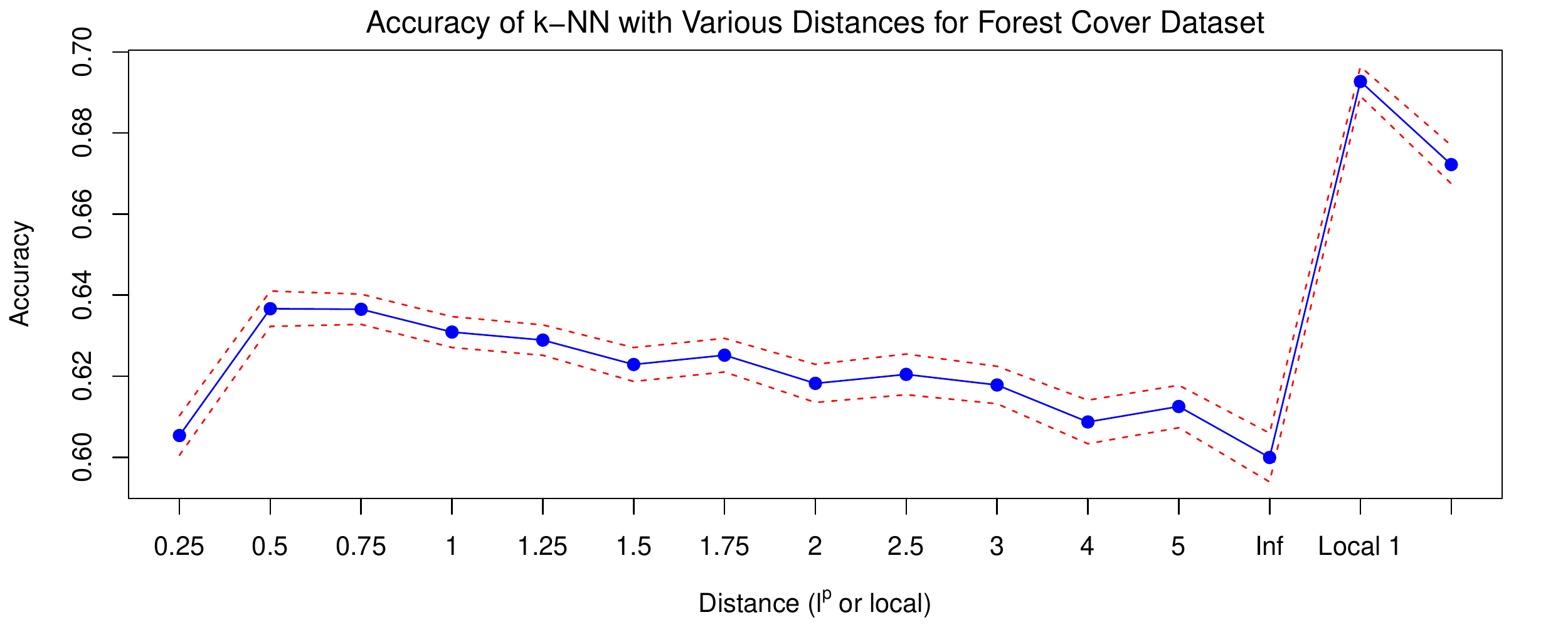}
\caption{Plot of the accuracy of $k$-NN with various $\ell^p$ norms, quasinorms, and locally chosen distances for Forest Cover dataset (showing 95 \% confidence intervals around the mean result, and data points). Here local 1 is locally chosen $\ell^p$ distance with matrix, local 2 is locally chosen polynomial with matrix.}
\label{fig:LinePlotForestCover}
\end{figure}

\begin{table}
\centering
\begin{tabular}{|c|c|c|}
\hline 
Distance & Mean Accuracy & 95\% Confidence Interval \\ 
\hline 
$\ell^{0.25}$ quasinorm & $0.60538$ & $[0.6005078, 0.6102522]$ \\ 
\hline 
$\ell^{0.5}$ quasinorm & $0.63666$ & $[0.6323007, 0.6410193]$ \\ 
\hline 
$\ell^{0.75}$ quasinorm & $0.63652$ & $[0.6327927, 0.6402473]$ \\ 
\hline 
$\ell^1$ norm & $0.6309$ & $[0.6270883, 0.6347117]$ \\ 
\hline 
$\ell^{1.25}$ norm & $0.62892$ & $[0.6251664, 0.6326736]$ \\ 
\hline 
$\ell^{1.5}$ norm & $0.6229$ & $[0.6187195, 0.6270805]$ \\ 
\hline 
$\ell^{1.75}$ norm & $0.6252$ & $[0.6210497, 0.6293503]$ \\ 
\hline 
$\ell^{2}$ norm & $0.61824$ & $[0.6135345, 0.6229455]$ \\ 
\hline 
$\ell^{2.5}$ norm & $0.62046$ & $[0.615441, 0.625479]$ \\ 
\hline 
$\ell^3$ norm & $0.61784$ & $[0.6072968, 0.6177832]$ \\ 
\hline 
$\ell^4$ norm & $0.60872$ & $[0.6033483, 0.6140917]$ \\ 
\hline 
$\ell^5$ norm & $0.61254$ & $[0.6072968, 0.6177832]$ \\ 
\hline 
$\ell^\infty$ norm & $0.59996$ & $[0.5939458, 0.6059742]$ \\ 
\hline 
{\scriptsize Local distance with $\ell^p$ norm and matrix} & $0.69272$ & $[0.6890909, 0.6963491]$ \\ 
\hline 
{\scriptsize Local distance with degree 5 polynomial and matrix} & $0.67222$ & $[0.6675009, 0.6769391]$ \\ 
\hline 
\end{tabular}
\caption{The mean accuracy and confidence intervals for $k$-NN applied to the Forest Cover dataset with various $\ell^p$ norms and locally chosen distances.}
\label{table:TableForestCover}
\end{table}

\subsection{Higgs Boson Dataset}

The ATLAS Higgs Boson dataset (\cite{HiggsBoson}) contains 29 numeric data columns and a response column with two possible states (namely, whether an event is a Higgs Boson event or is background noise). There are other columns, but they are not relevant for us (they include an importance weight for an alternative classification performance measure). We remove data columns with missing values to obtain 18 data columns. The original dataset contains 818238 rows. We then randomly generate 500 independent random subsets of the original dataset, each containing 5000 training rows and 5000 testing rows (with the training and testing sets being disjoint). Before applying $k$-NN, we fit the columns fitted to the interval $[0, 1]$.

We then find the classification accuracy with $k$-NN in each case, trying various $\ell^p$ norms/quasinorms. In particular, we have tried the $\ell^p$ norms and quasinorms with $p$ being $0.25, 0.5, 0.75, 1, 1.25, 1.5, 1.75, 2, 2.25, 2.5, 2.75, 3, 4, 5, \infty$. The mean accuracies obtained (with 95 \% confidence intervals) are given in Table \ref{table:TableHiggsBoson}, a box-and-whiskers plot is shown in Figure \ref{fig:BoxWhiskersHiggsBoson}. We see that the $\ell^p$ quasinorms with $0 < p < 1$ perform better than the $\ell^p$ norms with $p \geq 1$, with $p = 1/4$ (the $\ell^{1/4}$ quasinorm) resulting in the highest classification accuracy. We see that after $p = 1/4$, the classification accuracy decreases as $p$ increases. The improvement in accuracy in using the $\ell^{1/4}$ quasinorm is very significant, the mean accuracy with the $\ell^{1/4}$ quasinorm is approximately $77.2$ \% while for the Euclidean norm it is approximately $72.1$ \% (so the quasinorm performs approximately 5 \% better than the Euclidean norm).

\begin{table}
\centering
\begin{tabular}{|c|c|c|}
\hline 
Distance & Mean Accuracy & 95\% Confidence Interval \\ 
\hline
$\ell^{0.1}$ & $0.7678608$ & $[0.7672895, 0.7684321]$ \\
\hline
$\ell^{0.25}$ & $0.7718852$ & $[0.7713035, 0.7724669]$ \\
\hline
$\ell^{0.5}$ & $0.7646396$ & $[0.7640515, 0.7652277]$ \\
\hline
$\ell^{0.75}$ & $0.7529876$ & $[0.7523944, 0.7535808]$ \\
\hline
$\ell^{1}$ & $0.7423464$ & $[0.7417668, 0.742926]$ \\
\hline
$\ell^{1.25}$ & $0.7345884$ & $[0.7340186, 0.7351582]$ \\
\hline
$\ell^{1.5}$ & $0.7289776$ & $[0.7283997, 0.7295555]$ \\
\hline
$\ell^{1.75}$ & $0.724774$ & $[0.7242062, 0.7253418]$ \\
\hline
$\ell^{2}$ & $0.7214408$ & $[0.720877, 0.7220046]$ \\
\hline
$\ell^{2.25}$ & $0.718962$ & $[0.718397, 0.719527]$ \\
\hline
$\ell^{2.5}$ & $0.716842$ & $[0.7162883, 0.7173957]$ \\
\hline
$\ell^{2.75}$ & $0.7152068$ & $[0.7146595, 0.7157541]$ \\
\hline
$\ell^{3}$ & $0.7138832$ & $[0.713334, 0.7144324]$ \\
\hline
$\ell^{4}$ & $0.7104716$ & $[0.7099282, 0.711015]$ \\
\hline
$\ell^{5}$ & $0.708712$ & $[0.7081656, 0.7092584]$ \\
\hline
$\ell^{\infty}$ & $0.7042$ & $[0.7036437, 0.7047563]$ \\
\hline 
\end{tabular}
\caption{The mean accuracy and confidence intervals for $k$-NN applied to the Higgs Boson dataset with various $\ell^p$ norms.}
\label{table:TableHiggsBoson}
\end{table}

\begin{figure}
\centering
\includegraphics[scale=.5]{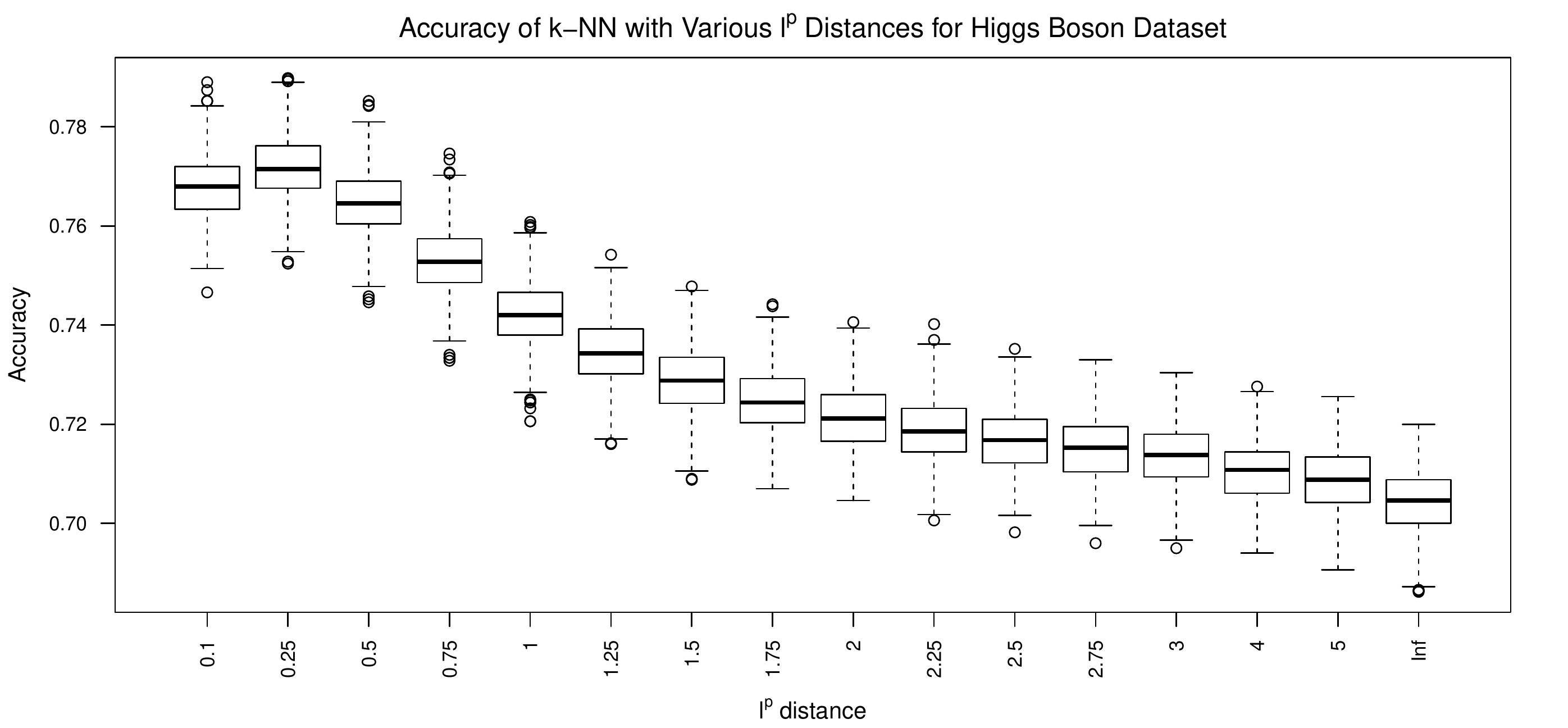}
\caption{Box-and-whiskers plot of the accuracy of $k$-NN with various $\ell^p$ norms and quasinorms for the Higgs Boson dataset.}
\label{fig:BoxWhiskersHiggsBoson}
\end{figure}

\begin{figure}
\centering
\includegraphics[scale=.5]{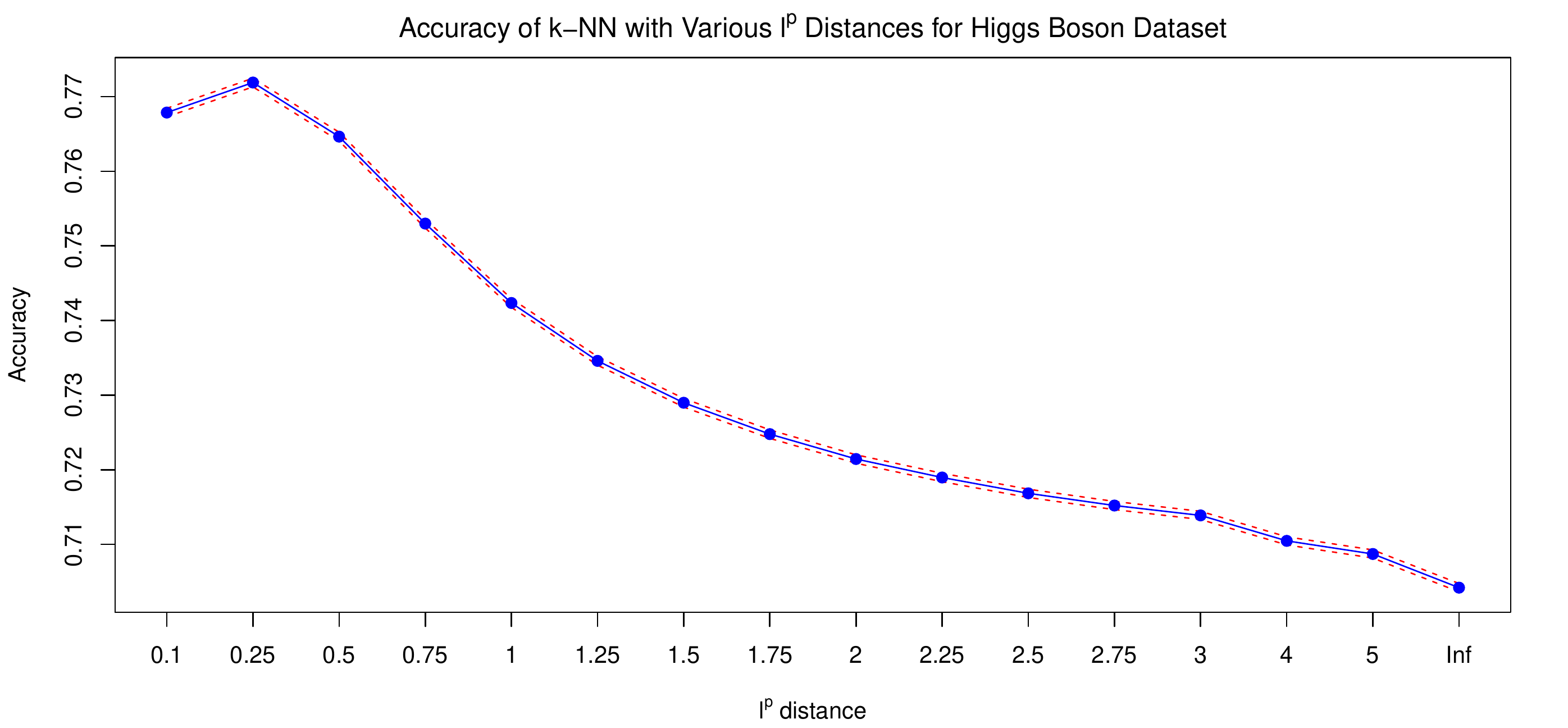}
\caption{Plot of the accuracy of $k$-NN with various $\ell^p$ norms and quasinorms for the Higgs Boson dataset (showing 95 \% confidence intervals around the mean result, and data points).}
\label{fig:LinePlotHiggsBoson}
\end{figure}

\chapter{Future Prospects}

There are many possible future projects based on this work. Here is an outline of some possibilities:
\begin{itemize}
\item We would like to extend our result in Theorem \ref{theorem:KnnWithSandwichedNormIsUniversallyConsistent} to be able to handle norms chosen at each step based on the sample points $X_1, X_2, \dots, X_n$ and possibly the query $X$ (and similarly for Theorem \ref{theorem:LipschitzFamilyKnnIsUniversallyConsistent} with the uniformly locally Lipschitz family). In \cite{pbook}, a similar result is claimed in Theorem 26.3 (where we multiply the data by a matrix that is a function of the sample points $X_1, X_2, \dots, X_n$ and then apply the Euclidean norm), however the proof provided is incomplete (part of it being incorrect), as we have shown above. We would like to recover that result, which should hold for more general families of norms.

\item There is a classification algorithm (described in \cite{LMNN}) called the \emph{large margin nearest neighbour} (\emph{LMNN}), in which we learn a positive semi-definite matrix $M$ used to construct a pseudometric of the form $\rho(\vecx, \vecy) = {(\vecx - \vecy)}^\intercal M (\vecx - \vecy)$ for $k$-NN (a \emph{pseudometric} is similar to a metric, but can take value zero for distinct points). The procedure for learning the matrix $M$ depends on the query and the labelled sample. LMNN has been found to produce good results for classifying various datasets We would like to determine if LMNN (or a similar algorithm) is universally consistent.

\item Here, we have proven the universal consistency of $k$-NN based on Stone's theorem, and have applied this to various modifications of the classical $k$-NN classifier (such as using a sequence of norms or of Lipschitz functions). There is an alternative proof for the universal consistency of $k$-NN based on the Lebesgue-Besicovitch Differentiation Theorem, originally given by Luc Devroye in \cite{Devroye} and discussed further by Fr\'{e}d\'{e}ric C\'{e}rou and Arnaud Guyader in the paper \cite{KnnInf}. Priess has shown (in \cite{Preiss}) that the conclusion of the differentiation theorem is equivalent to the $\sigma$-finite dimensionality of a metric space. We would like to extend this result to sequences of norms (instead of a single fixed norm, similar to our result) and investigate if it can lead to a universal consistency proof with quasinorms and various other distances.

\item We would like to improve our optimization methods, so we can more accurately and rapidly optimize over a class of distances for $k$-NN. We have found some like the correlation method discussed above, we would like to find others.

\item We have assumed a bounded family of norms (or uniformly locally Lipschitz distances). This was a necessary assumption as we saw for the sequences of norms given by equations \eqref{eq:BadNormSequenceUnboundedAbove} and \eqref{eq:BadNormSequenceUnboundedBelow}. We notice that such a sequence is extremely unlikely to be picked by an optimizer optimizing over the family of norms for that distribution (indeed, the probability approaches zero as $n$ approaches infinity for our distribution). We would like to determine if we can remove this assumption provided we follow an optimization procedure instead of picking an arbitrary (possibly bad) sequence. We may find a solution to this problem if we develop a theory of capacity for norms for $k$-NN, similar to VC dimension (Vapnik-Chervonenkis dimension) and more generally metric entropy.
\end{itemize}

\end{document}